\documentclass[12pt]{article}
\usepackage[margin=1in]{geometry}
\usepackage{times}
\usepackage{hyperref}

\usepackage{url,booktabs}
\usepackage{nicefrac,microtype,xcolor}
\usepackage{graphicx,tikz}
\usetikzlibrary{calc,arrows.meta}
\usepackage{subfigure}
\usepackage{amsmath,amsthm,amssymb,amsfonts}
\usepackage{mathtools,mathrsfs,bm}

\usepackage{amsmath, amssymb, amsthm, amsfonts}
\usepackage{times}
\usepackage{xcolor}
\usepackage{graphicx}
\usepackage{enumitem}
\usepackage{cleveref}
\usepackage{titlesec}
\usepackage{mathrsfs}
\usepackage{physics}
\usepackage{booktabs}
\usepackage{tensor}
\usepackage{nicematrix}
\usepackage{thm-restate}
\usepackage[numbers]{natbib}

\DeclareMathOperator{\softmax}{softmax}
\DeclareMathOperator{\attn}{attn}
\DeclareMathOperator{\hop}{hop}
\DeclareMathOperator{\cyc}{cyc}
\renewcommand{\S}{\mathcal{S}}

\titlespacing*{\section}{0pt}{6pt}{0pt}
\titlespacing*{\subsection}{0pt}{6pt}{0pt}

\DeclareMathOperator{\E}{\mathbb{E}}

\newcommand{\R}{\mathbb{R}}

\newcommand{\1}{\mathbf{1}}

\DeclareMathOperator{\diag}{diag}

\DeclareMathOperator{\poly}{poly}

\newtheorem{lemma}{Lemma}
\newtheorem{corollary}{Corollary}
\newtheorem{theorem}{Theorem}

\newtheorem{definition}{Definition}

\theoremstyle{remark}
\newtheorem{remark}{Remark}

\ifdefined\usebigfont

\usepackage{times}
\usepackage[fontsize=13pt]{scrextend}
\AtBeginDocument{
	\newgeometry{letterpaper,left=1.56in,right=1.56in,top=1.71in,bottom=1.77in}
}
\else
\fi

\usepackage{caption}
\usepackage{algorithm, algpseudocode}
\usepackage{authblk}
\usepackage{etoc}
\usepackage{wrapfig} 

\usepackage{titlesec}
\titlespacing*{\paragraph}{0pt}{3.2pt plus 1pt minus 1pt}{1em}

\hypersetup{
  colorlinks,
  linkcolor={red!60!black},
  citecolor={blue!60!black},
}

\newcommand*\samethanks[1][\value{footnote}]{\footnotemark[#1]}

\ifdefined\usebigfont

\usepackage{times}
\usepackage[fontsize=13pt]{scrextend}
\makeatletter
\@ifpackageloaded{geometry}{\AtBeginDocument{\newgeometry{letterpaper,left=1.56in,right=1.56in,top=1.71in,bottom=1.77in}}}{\usepackage[letterpaper,left=1.56in,right=1.56in,top=1.71in,bottom=1.77in]{geometry}}
\AtBeginDocument{\newgeometry{letterpaper,left=1.56in,right=1.56in,top=1.71in,bottom=1.77in}}
\usepackage{hyperref} 

\else
\fi

\title{Learning Compositional Functions with Transformers from Easy-to-Hard Data}

\usepackage{tikz}

\usepackage{csquotes}
\MakeOuterQuote{"}

\author{
{Zixuan Wang}\thanks{Equal contribution. \vspace{-2.5mm}}$^{~\,,1}$\!,\,~ 
{Eshaan Nichani}\samethanks[1]$^{~\,,1}$\!,\,~ 
{Alberto Bietti}$^{2}$\!,\,~
{Alex Damian}$^{1}$\!,\\
\vspace{-1.5mm}
{Daniel Hsu}$^{3}$\!,\,~
{Jason D. Lee}$^{1}$\!,\,~ 
{Denny Wu}$^{2,4}$
\\
\vspace{4mm}
\normalsize{$^{1}$Princeton University},\,\,\, 
\normalsize{$^{2}$Flatiron Institute},\,\,\,
\normalsize{$^{3}$Columbia University},\,\,\,
\normalsize{$^{4}$New York University}
\vspace{1.5mm}
\\
}

\begin{document}

\maketitle 

\vspace{-3mm}

\begin{abstract}%
Transformer-based language models have demonstrated impressive capabilities across a range of complex reasoning tasks. Prior theoretical work exploring the expressive power of transformers has shown that they can efficiently perform multi-step reasoning tasks involving parallelizable computations. 
However, the learnability of such constructions, particularly the conditions on the data distribution that enable efficient learning via gradient-based optimization, remains an open question.
Towards answering this question, in this work we study the learnability of the \emph{$k$-fold composition} task, which requires computing an interleaved composition of $k$ input permutations and $k$ hidden permutations, and can be expressed by a transformer with $O(\log k)$ layers.
On the negative front, we prove a Statistical Query (SQ) lower bound showing that any SQ learner that makes only polynomially-many queries to an SQ oracle for the $k$-fold composition task distribution must have sample size exponential in $k$, thus establishing a statistical-computational gap.
On the other hand, we show that this function class can be efficiently learned, with runtime and sample complexity polynomial in $k$, by gradient descent on an $O(\log k)$-depth transformer via two different curriculum learning strategies: one in which data consists of $k'$-fold composition functions with $k' \le k$ presented in increasing difficulty, and another in which all such data is presented simultaneously.
Our work sheds light on the necessity and sufficiency of having both easy and hard examples in the data distribution for transformers to learn complex compositional tasks.
\end{abstract}

\allowdisplaybreaks

\section{Introduction}

Large language models based on transformers have demonstrated promising capabilities in complex reasoning tasks that require combining multiple intermediate steps \citep{nye2021show,wei2022chain,lewkowycz2022solving,lanchantin2024learning,yao2024tree}. Recent theoretical works have proven that transformers can \textit{express} various sequential/compositional reasoning algorithms \citep{liu2023transformers,merrill2023expresssive,li2024chain,feng2024towards,sanford2024understanding}. However, representational power does not entail statistical or optimization efficiency. In fact, empirical studies have shown that elaborate training procedures -- such as curriculum or process supervision \citep{uesato2022solving,lightman2023let,dziri2023faith,bachmann2024pitfalls,deng2024explicit} -- are often required for models to acquire strong reasoning capabilities. 
This highlights the need for a deeper theoretical understanding of the optimization and sample efficiency of transformers on compositional reasoning tasks.

Our starting point is the recent work of \citet{sanford2024transformers}, which examined the expressivity of transformers on a specific synthetic reasoning task called the ``$k$-hop induction head'', which involves composing 
$k$ steps of pointer-following given in the context to predict the correct answer.
This function is a generalization of the induction head mechanism identified by \citet{olsson2022context} ($k=1$), and intuitively, the difficulty of computing the function increases with~$k$.
\citet{sanford2024transformers} showed that a transformer with 
$\Theta(\log k)$ layers can efficiently represent this task and that this depth is necessary, conditional on a well-known conjecture from the \textit{Massively Parallel Computation} literature~\citep{im2023massively}.
Furthermore, the authors empirically observed that gradient-based learning is challenging unless some form of curriculum (i.e., including lower-hop data in the training process) is introduced.

The $k$-hop task (for $k\geq2$) is closely related to the function composition tasks studied by \citet{peng2024on} and \citet{chen2024theoretical}, which were motivated by certain natural language understanding tasks
(e.g., finding an ancestor of a person in a genealogy), as well 
as ``multi-hop'' reasoning problems studied extensively in the natural language literature, even before transformers were introduced~\citep[e.g.,][]{weston2014memory,weston2015towards,sukhbaatar2015end}.
An example of a 2-hop reasoning problem due to~\citet{weston2014memory} is as follows: \emph{John plays football. The football game is on Sunday. On what day does John play?} 
Here, a model must learn to compose the relationships $(\text{John} \rightarrow \text{Football}, \text{Football} \rightarrow \text{Sunday})$ present in the context. 

More challenging compositional reasoning tasks can also require composing information present in the context (i.e "contextual knowledge") with global "parametric knowledge" not provided in the context~\citep{cheng2024understanding, yang2024large}. Consider instead the prompt: \emph{John plays quarterback. The football game is on Sunday. On what day does John play?} Now, the model must compose the contextual knowledge $(\text{John} \rightarrow \text{Quarterback}, \text{Football} \rightarrow \text{Sunday})$ given in the prompt, with the parametric knowledge $(\text{Quarterback} \rightarrow \text{Football})$ which cannot be inferred from context alone.

\paragraph{Our contributions.} We analyze the complexity of training a (deep) transformer model using SGD to solve a task involving  $k$-hop compositional reasoning, which we refer to as the $k$-\textit{fold composition} task. The task requires outputting an element of~$[N]$ after applying an interleaved product of $k$ \textit{in-context} permutations on $N$ elements and $k$ \textit{hidden parametric} permutations, and can be viewed as an instance of $k$-hop prediction combining contextual and hidden parametric knowledge. We aim to establish sample complexity upper bounds for gradient-based learning as well as computational lower bounds for this function class. More specifically, our contributions are as follows:



\begin{enumerate}[leftmargin=*,topsep=0mm,itemsep=0.5mm]
    \item \textbf{$k$-fold composition task.} In \Cref{sec:task} we introduce the $k$-fold composition task, which requires computing an interleaved composition of $k$ input permutations and $k$ hidden permutations on $N$ elements. In \Cref{thm:construction} we prove that this task is expressible via an $O(\log k)$ depth transformer with embedding dimension $d = \tilde O(Nk)$.
    \item \textbf{Lower bound.} In \Cref{sec:lb}, we establish a statistical query (SQ) lower bound showing that either $N^{\Omega(k)}$ queries or a tolerance of $\tau \le N^{-\Omega(k)}$ is required to learn the $k$-fold composition task when trained on samples only from the $k$-fold functions. Using the standard $\tau \approx n^{-1/2}$ heuristic, this implies that any SQ learner (which can model gradient descent on neural networks) must either have sample size or runtime exponential in $k$.
    \item \textbf{Gradient-based learning.} On the other hand, in Section~\ref{sec:ub} we show that a transformer \emph{can} efficiently learn the task when training on easy-to-hard data consisting of $k'$-fold functions for $k' \le k$. In \Cref{thm:main}, we show that if the transformer is presented with a curriculum of $2^\ell$-fold data for increasing values of $\ell$, then gradient descent can learn the $k$-fold task with $\poly(N,k)$ samples, which removes the exponential dependence on $k$ in the SQ lower bound. In \Cref{thm:mix-data-main}, we show that this efficient learning guarantee also applies to simultaneously training on a mixture of the $2^\ell$-fold data.
\end{enumerate}


\subsection{Related Work: Transformer Theory}
\paragraph{Expressivity of transformers.} As already alluded to previously, many prior works have connected the computational power of transformers to models of parallel computation, including circuit models~\citep{liu2023transformers,merrill2023parallelism} and massively parallel computation~\citep{sanford2024transformers,sanford2024understanding}.
This stands in contrast to sequential neural architectures such as recurrent neural networks, which are unable to efficiently represent certain parallel computations that transformers can~\citep{sanford2023representational, bhattamishra2024separations,jelassi2024repeat}. As for negative results on the expressivity of compositional tasks, \citet{peng2024on}~showed the composition of two functions cannot be efficiently represented by one-layer transformers (even in an average-case sense), and \citet{chen2024theoretical}~showed that, for any constant $L$, compositions of $L$ functions cannot be efficiently represented by $L$-layer decoder-only transformers. Since we consider cross-attention, and not decoder-only, transformers, the lower bound of~\citet{chen2024theoretical} does not apply to our setting.

\paragraph{Optimization guarantees for transformers.} A number of prior works have studied the gradient descent dynamics of transformers for various synthetic tasks~\citep{jelassi2022vision,li2023transformers,bietti2023birth,tian2023scan,zhang2023trained, huang2023context,nichani2024transformers, nichani2024understanding,renlearning,wang2024transformers,chen2024unveiling,chen2024training, huang2025transformers}. These works, however, only focus on one or two-layer transformers. While existing gradient flow or landscape results do exist for deeper transformers \citep{ahn2024transformers,gao2024global}, we are the first to provide an end-to-end optimization and statistical guarantee for a deep transformer.

Among existing optimization guarantees for transformers, to our knowledge, the only setting that considers a compositional structure is the parity problem \citep{kim2024transformers,wen2024sparse}. That said, while the task decomposition (i.e., curriculum) for parity exhibits a hierarchical structure, the target function itself is not inherently compositional and can be represented by a shallow network; \citet{abbe2023provable,panigrahi2024progressive} have theoretically studied the benefit of curriculum learning for learning parities with shallow neural networks. Our goal is to analyze the gradient-based learning of a more challenging compositional task that fundamentally requires a deeper transformer architecture.

\subsection{Related Work: Compositional Tasks}

Our $k$-fold composition task is most similar to the $k$-hop induction head task in \citet{sanford2024transformers}, which takes as input a sequence of $T$ tokens $X \in \Sigma^T$, and requires outputting a $k$-fold composition of a certain "hop" function defined as an in-context map from $[T] \rightarrow [T]$ by a similar mechanism to the induction head~\citep{olsson2022context}. The permutation composition task also computes such a composition, but where each hop is specified explicitly (from $(i, j) \in [k] \times [N]$ to $(i-1, \sigma_i(j))$) rather than needing to be learned in context. 
The $k$-hop induction head task is closely related to the well-studied pointer-chasing problem \citep{PAPADIMITRIOU1984260,nisancommunication}, for which communication complexity lower bounds have been established in various settings~\citep[e.g.,][]{yehudayoff2020pointer,assadi2021graph}. Furthermore, our $k$-fold composition task is an instance of computing an interleaved group product, a problem also studied in communication complexity~\citep{gowers2019interleaved}. However, in neither setting has the question of learnability been investigated.

Our task is also related to the semiautomaton simulation task of \citet{liu2023transformers}, which takes as input a sequence of elements $(g_1, \dots, g_T)$ from some semigroup, and outputs their product $g_Tg_{T-1}\cdots g_2g_1$. Indeed, \citet{liu2023transformers} showed that there exists an $O(\log T)$ depth transformer which can solve the task. Our $k$-fold composition task can be viewed as a special case of this task for the symmetric group $G = S_N$. However, our approach differs in that we encode each permutation $\sigma_i$ as $N$ tokens $(\sigma_i(1), \dots, \sigma_i(N))$ rather than a single token, and consider the learning of a function class formed by interleaving the hidden permutations $(\pi_1, \dots, \pi_k)$ into the product.

Another relevant compositional problem is ``planning in path-star graphs'' introduced by~\citet{bachmann2024pitfalls} to show the limitation of learning with next-token prediction. Their task involves finding a path to a final node in a star graph, when the edges of the graph are given in the context. The first node in the path is then essentially a $k$-hop prediction starting from the final node, where~$k$ is the length of the paths. The authors discussed the difficulty of this $k$-hop prediction, and the benefits of first predicting intermediate steps in the path, which is related to our curriculum and mixture strategies, but no explicit expressivity or optimization guarantees were given. 
Finally, the hardness of learning the $k$-fold function class is related to the locality barrier studied in \citet{abbe2024far}, where the authors established a computational lower bound for a particular cycle task that exhibits compositional structure, using the permutation invariance of the learning algorithm. 


\section{Preliminaries}

\textbf{Notation.} Let $S_N$ denote the set of permutations on $N$ elements, and for two permutations $\pi, \sigma \in S_N$, let $\pi \circ \sigma$ denote their composition. For integer $d$ and index $i \in [d]$, let $e_{d, i} \in \mathbb{R}^d$ denote the $d$-dimensional one-hot vector with a 1 in the $i$th coordinate. We use $\Tilde{O},\Tilde{\Omega}$ to hide $\poly\log (kN)$ factors,
and we use $f\lesssim g$ (or $f=O(g)$, $g=\Omega(f)$) when $f \leq Cg$ for an absolute constant $C>0$.

\subsection{The $k$-fold Composition Task}
\label{sec:task}

Let $\pi := (\pi_1, \dots, \pi_k) \in \qty(S_N)^k$ be a tuple of $k$ hidden permutations. The \emph{$k$-fold composition} task takes as input $(\sigma, x) \in \mathcal{X} := (S_N)^k \times [N]$, where $\sigma = (\sigma_1, \dots, \sigma_k)$ is a tuple of $k$ permutations and $x$ is an index in $[N]$. The task $f_\pi : \mathcal{X} \rightarrow [N]$ is defined as
\begin{align*}
    f_\pi(\sigma, x) := (\sigma_k \circ \pi_k \circ \sigma_{k-1} \circ \pi_{k-1} \circ \cdots \circ \sigma_1 \circ \pi_1)(x).
\end{align*}
Our target function class is $\mathcal{F} = \{f_\pi : \pi \in (S_N)^k\}$, and the goal is to learn $\mathcal{F}$ with respect to the uniform distribution over $\mathcal{X}$. See \Cref{fig:diagram} for illustration.


We will also consider an extension of the task called the \emph{cyclic} $k$-fold composition task. Here, the input space is $\mathcal{X}^{\cyc} := (S_N)^k \times [k] \times [N]$, and the target $f_\pi^{\cyc} : \mathcal{X}^{\cyc} \rightarrow [N]$ is defined by
\begin{align*}
    f_\pi^{\cyc}(\sigma, i, x) := (\sigma_{i + k - 1} \circ \pi_{i + k - 1} \circ \sigma_{i + k-2} \circ \pi_{i + k-2} \circ \cdots \circ \sigma_{i+1} \circ \pi_{i+1} \circ \sigma_i \circ \pi_i)(x),
\end{align*}
where the indices of permutation are taken modulo $k$. We define $\mathcal{F}^{\cyc} = \{f^{\cyc}_\pi : \pi \in (S_N)^k\}$.

\begin{remark}
As mentioned in the Introduction, the target function is defined as an interleaved composition of $k$ contextual permutations $\sigma$ (e.g., $\text{John} \rightarrow \text{Quarterback}$) and $k$ hidden permutations $\pi$ not given in the input (e.g., $\text{Quarterback} \rightarrow \text{Football}$) which may represent ``parametric'' or ``in-weights'' knowledge \cite{chan2022data,yang2024large,cheng2024understanding}. The contextual permutation is a standard feature in the $k$-hop reasoning task \cite{sanford2024transformers}, whereas the parametric permutation enables us to define a function class to study the statistical hardness \cite{10.1145/293347.293351} of compositional tasks. 
\end{remark}

\begin{figure}[t]
    \centering
    \includegraphics[width=0.9\textwidth]{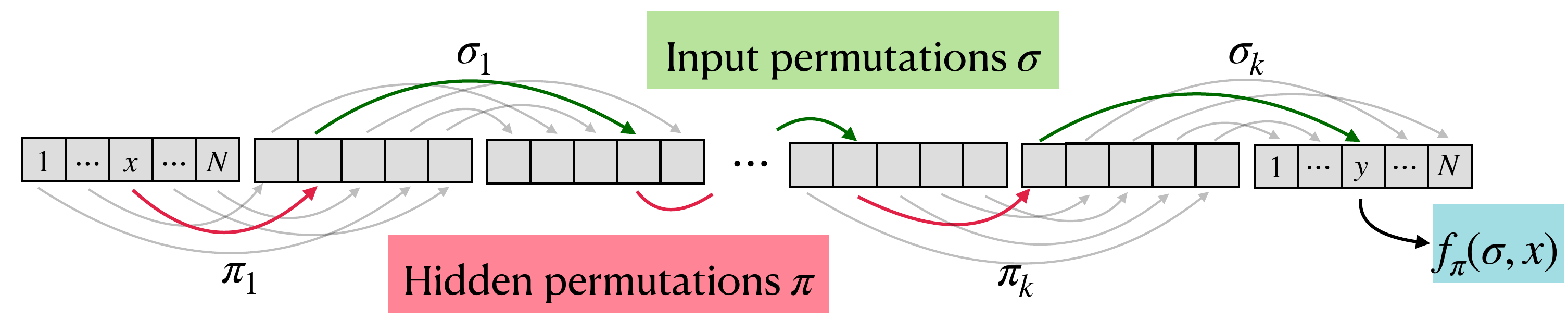}
    \caption{\small \textbf{$k$-fold composition task} -- red arrows represent the hidden permutation $\pi_i$ and green arrows denote input permutations $\sigma_i$. Given an input $(\sigma,x)$, $f_\pi(\cdot,\cdot)$ composes $2k$ permutations to output $f_\pi(\sigma,x)$.}
    \label{fig:diagram}
\end{figure}

\subsection{Transformer Architecture}
Our learner is an $L$-layer transformer. Transformers take as input a length $T$ sequence of $d$ dimensional embedding vectors $X = \begin{bmatrix} x_1, \dots, x_T \end{bmatrix} \in \mathbb{R}^{d \times T}$. We will restrict ourselves to attention-only transformers, where each layer is a \emph{self-attention head}, defined as follows:
\begin{definition}[Self-attention head]
    For a vector $v \in \mathbb{R}^k$, define the element-wise softmax operator $\mathcal{S}:\mathbb{R}^k \rightarrow \mathbb{R}^k$ by $\mathcal{S}(v)_i = \exp(v_i)/\sum_{j=1}^k \exp(v_j)$. The self-attention head is a mapping $\attn(~\cdot~; W_{KQ}, W_{OV}) : \mathbb{R}^{d \times T} \rightarrow \mathbb{R}^{d \times T}$ parameterized by the key-query matrix $W_{KQ} \in \mathbb{R}^{d \times d}$ and output-value matrix $W_{OV} \in \mathbb{R}^{d \times d}$, which operates on a sequence of embeddings $X \in \mathbb{R}^{d \times T}$ as
    \begin{align*}
        \attn(X; W_{KQ}, W_{OV}) = X + W_{OV}X\S(X^\top W_{KQ} X),
    \end{align*}
    where the softmax function $\S$ is applied column-wise.
\end{definition}

A multi-layer transformer composes multiple self-attention heads in series. For simplicity, we consider transformers with a single head per layer.

\begin{definition}[Attention-only transformer]\label{def:transformer}
Let $L$ be the depth and $d$ be the embedding dimension. For $\ell \in [L]$, let $W_{KQ}^{(\ell)}, W_{OV}^{(\ell)}$ be the key-query and output-value matrices in the $\ell$-th layer, respectively. Let $\theta := \{(W_{KQ}^{(\ell)}, W_{OV}^{(\ell)})\}_{\ell \in [L]}$ denote the aggregate parameter vector. A transformer $\mathrm{TF}_\theta : \mathbb{R}^{d \times T} \rightarrow \mathbb{R}^{d \times T}$ operates on an input sequence of embeddings $X \in \mathbb{R}^{d \times T}$ as follows:
\begin{align*}
    X^{(0)} &= X,\\
    X^{(\ell)} &= \attn(X^{(\ell - 1)}; W_{KQ}^{(\ell)}, W_{OV}^{(\ell)}), \quad i \in [L]\\
    \mathrm{TF}_\theta(X) &= X^{(L)}.
\end{align*}
\end{definition}

\paragraph{Embedding and Decoding.}  We embed the input $\sigma$ as follows. Let $\phi : [k] \times [N] \times [N] \rightarrow \mathbb{R}^d$ be some embedding function; then, the input to the transformer is the length $T = kN$ sequence
\begin{align*}
    X(\sigma) = \{ \phi(i, j, \sigma_i(j))\}_{i \in [k], j \in [N]},
\end{align*}
where the columns of $X(\sigma)$ are indexed by the tuples $(i, j) \in [k] \times [N]$. 

The output of $\mathrm{TF}_\theta(X(\sigma))$ must also be decoded to a prediction in $[N]$ as follows. Let $\Psi \in \mathbb{R}^{d \times N}$, be the readout layer. The predictions of the learner for the permutation composition and the cyclic permutation tasks are given by
\begin{align*}
    \hat f(\sigma, x) = (\Psi^\top \mathrm{TF}_\theta(X(\sigma))_{(1, x)} \in \mathbb{R}^N, \qand \hat f(\sigma, i, x) = (\Psi^\top \mathrm{TF}_\theta(X(\sigma))_{(i, x)} \in \mathbb{R}^N
\end{align*}
respectively.

\section{Transformer Construction}
While the $k$-fold composition task requires composing the $2k$ permutations, we show that there exists a transformer with $O(\log k)$ layers that solves the task, given by the following theorem.

\begin{theorem}\label{thm:construction}
    Assume that $k$ is a power of two. There exists an embedding function $\phi$ with $d = kN(3 + \log_2k)$ such that, for any $\pi \in (S_N)^k$, there exists an $L = \log_2 k + 1$ layer transformer which can exactly express the $k$-fold composition task, i.e
    \begin{align*}
        (\Psi^\top\mathrm{TF}_\theta(X(\sigma)))_{(1,j)} = e_{N, f_\pi(\sigma, x)} \quad \text{for all}~(\sigma, x) \in \mathcal{X}.
    \end{align*}
\end{theorem}

\begin{figure}[t]
    \centering
    \includegraphics[width=0.92\textwidth]{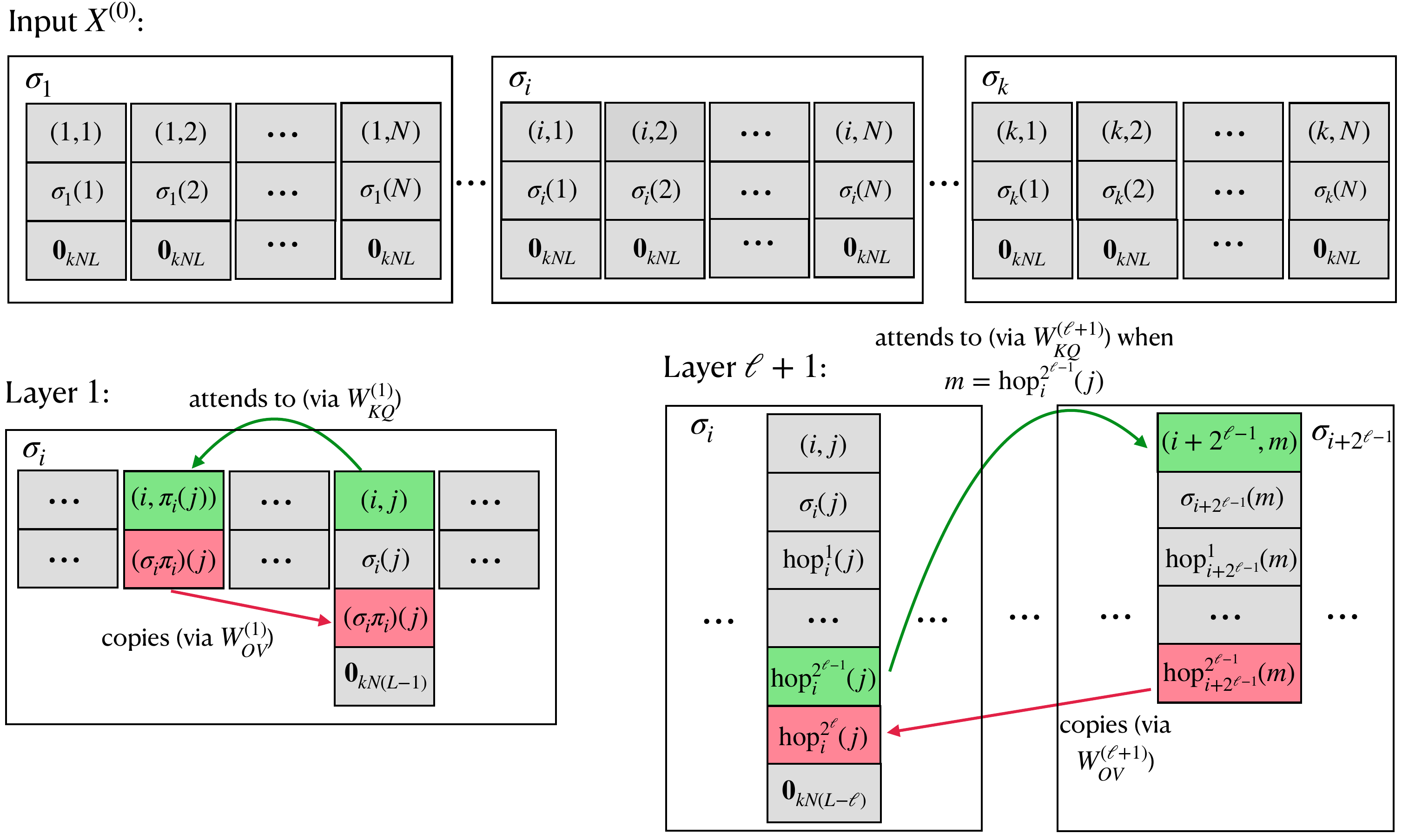}
    \caption{\small Illustration the format of input $X^{(0)}$ and the attention pattern in \Cref{thm:construction}.}
    \label{fig:construction}
\end{figure}

\paragraph{Proof Sketch.} The proof proceeds similarly to the constructions in \citet{sanford2024transformers, liu2023transformers}. For notational convenience, define the permutation $\hop_i^r(\sigma, \cdot) \in S_N$ by 
\begin{align*}
    \hop_i^r(\sigma,\cdot) := \sigma_{i + r - 1} \circ \pi_{i + r-1} \circ \cdots \circ \sigma_{i + 1} \circ \pi_{i + 1} \circ \sigma_{i} \circ \pi_{i}.
\end{align*}
We will consider the embedding
\begin{align*}
    \phi(i, j, \sigma_i(j)) = \begin{bmatrix}
        E(i, j) \\
        E(i, \sigma_i(j))\\
        0_{kNL}
    \end{bmatrix} \in \mathbb{R}^{kN(L+2)},~~\text{where}~~ E(i, j) := e_{k, i} \otimes e_{N, j} \in \mathbb{R}^{kN}.
\end{align*}
The first layer of the transformer encodes the hidden permutation $\pi$. The key-query matrix $W_{KQ}^{(1)}$ is set so that the $(i,j)$ position attends to the $(i, \pi_i(j))$ position. The value matrix $W_{OV}^{(1)}$ then copies the second block of $X^{(0)}_{(i, \pi_i(j))}$, which encodes $\hop_i^1(\sigma, j)$, to the residual stream of $(i, j)$ (see Figure \ref{fig:construction}, left). As such, $X^{(1)}_{(i,j)}$ now contains $\hop_i^1(\sigma, j)$.

The remainder of the construction proceeds recursively. Let us assume that the output of the $\ell$-th layer has computed $\hop_i^{2^{\ell - 1}}(\sigma,\cdot)$; in particular that $X^{(\ell)}$ is of the form
    \begin{align*}
        (X^{(\ell)})_{(i, j)} &= \Bigg[E(i,j)^\top, ~~ E(i,\sigma_i(j))^\top, ~~ E(i, \hop_i^1(j))^\top, ~~
        E(i+1, \hop_i^2(j))^\top,~~
        E(i+3, \hop_i^4(j))^\top,\\
        & \qquad
        \cdots~~
        E(i + 2^{\ell-1} - 1, \hop_i^{2^{\ell - 1}}(j))^\top,~~
        0_{kN(L - \ell)}^\top \Bigg]^\top.
    \end{align*}
The $(\ell + 1)$th layer composes the quantities $\hop_i^{2^{\ell - 1}}(\sigma,\cdot)$ and $\hop_{i + 2^{\ell - 1}}^{2^{\ell - 1}}(\sigma,\cdot)$ to obtain $\hop_i^{2^{\ell}}(\sigma,\cdot)$. To do so, $W_{KQ}^{(\ell+1)}$ is first set so that the $(i,j)$ position attends to the $(i + 2^{\ell-1}, \hop_i^{2^{\ell-1}}(j))$ position (by comparing the green blocks in Figure \ref{fig:construction}, right). Then, the value matrix $W_{OV}^{(\ell+1)}$ copies the last nonzero block (the red block in Figure \ref{fig:construction}, right) from $X^{(\ell)}_{(i, \hop_i^{2^{\ell-1}}(j))}$ to the residual stream of $(i, j).$ The $(\ell + 1)$th layer thus computes $\hop_{i + 2^{\ell-1}}^{2^{\ell-1}}(\sigma, \hop_i^{2^{\ell-1}}(\sigma, j)) = \hop_i^{2^\ell}(\sigma, j)$ as desired.

Altogether $\log_2k + 1$ layers suffice to compute $f_\pi(\sigma, x) = \hop_1^{k}(\sigma, x).$ The complete proof of \Cref{thm:construction} is contained in \Cref{sec:construction-proofs}. In \Cref{sec:m ll k task}, we consider a modification of the $k$-fold composition task with $m \ll k$ hidden permutations, and show that an embedding dimension of $d = \tilde\Theta(mN)$ suffices.

\section{Statistical Query Lower Bound}
\label{sec:lb}

We have shown that for any $f_\pi \in \mathcal{F}$, there exists a transformer with $O(\log k)$ layers and embedding dimension $\poly(Nk)$ which can exactly express $f_\pi$. On the contrary, we will now show that in order to learn $f_\pi$, a learner must use either compute or sample size exponential in $k$.
Formally, we prove a statistical query (SQ) lower bound for learning $\mathcal{F}$~\citep{10.1145/293347.293351}.
Many learning algorithms, including gradient descent, can be understood in the SQ model, and thus SQ complexity is
a useful proxy for
the complexity of learning via gradient descent; we discuss this further in~\Cref{sec:lb-proofs}.

Under the SQ framework, the learner can interact with the target function $f_\pi$ by specifying a query $g: \mathcal{X} \times [N] \rightarrow \mathbb{R}$ and tolerance level $\tau$; the SQ oracle then returns any response $\hat q$
satisfying\footnote{For ease of exposition, we focus here on the regular $k$-fold composition task and function class $\mathcal{F}$, and defer the cyclic variant to the appendix. \vspace{-2mm}} $\abs{\hat q - \mathbb{E}\qty[g(\sigma, x, f_\pi(\sigma, x))]} \le \tau$. We further assume without loss of generality that $g$ satisfies the normalization $\sum_{y \in [N]}\mathbb{E}_{\sigma, x}[g(\sigma, x, y)^2] = 1$ and $\mathbb{E}_{\sigma}\qty[g(\sigma, x, y)] = 0$ (see \Cref{sec:lb-proofs} for discussion on this choice of normalization).
After making some number of queries, the learner outputs a predictor $\hat f : \mathcal{X} \rightarrow [N]$, which incurs the 0-1 loss $L(\hat f) := \mathbb{P}_{\sigma, x}\qty(\hat f(\sigma, x) \neq f_\pi(\sigma, x))$. Our lower bound against SQ learners is given by the following theorem:
\begin{theorem}\label{thm:SQ}
    Any SQ learner for the function class $\mathcal{F}$ or $\mathcal{F}^{\cyc}$ must either make $q \ge N^{\Omega(k)}$ queries or use a tolerance $\tau \le N^{-\Omega(k)}$ to output a predictor $\hat f$ with loss $L(\hat f) = O(1)$.
\end{theorem}
\begin{remark}
Using the standard $\tau \approx n^{-1/2}$ concentration heuristic, where $n$ is the sample size, \Cref{thm:SQ} implies that either runtime (at least the number of queries) or sample size $n$ must be $\geq N^{\Omega(k)}$; this exponential dependence on $k$ implies a large statistical-computational gap under the SQ class, since information theoretically $\tilde{\Theta}(Nk)$ samples are sufficient to learn this target function.
\end{remark}


The proof of \Cref{thm:SQ} requires constructing a subset of $\mathcal{F}$ of nearly orthogonal functions. It turns out that the SQ model induces the following inner product between functions $f_\pi, f_\rho \in \mathcal{F}$.
\begin{equation}
    \langle f_\pi, f_\rho \rangle := \mathbb{P}_{\sigma, x}\qty[f_\pi(\sigma, x) = f_\rho(\sigma, x)] - 1/N. \label{eq:inner-product}
\end{equation}


One can interpret this inner product as a covariance between the random variables $f_\pi(\sigma, x)$ and $f_\rho(\sigma, x)$. An intermediate step towards the proof is 
\Cref{lem:near-orth-set},
where we construct a large set of functions with a small pairwise correlation under the inner product \eqref{eq:inner-product}. The complete proof of \Cref{thm:SQ} is presented in \Cref{sec:lb-proofs}.

\section{Upper Bound for Gradient-based Learning}
\label{sec:ub}
Theorem \ref{thm:SQ} implies that any SQ learner with polynomial compute and access only to the target labels must use $N^{\Omega(k)}$ samples in order to learn a $k$-fold composition function $f_\pi$. 
We now show that with the aid of a specifically chosen \emph{curriculum}, it is possible to train a $O(\log k)$ depth transformer to learn $f_\pi$ with only $\poly(Nk)$ samples. Curriculum learning \citep{elman1993learning,bengiocurriculum} is the process of training a model on data with increasing difficulty, the benefit of which has been demonstrated in various empirical and theoretical works; see Section \ref{sec:discussion} for further discussion.

Our curriculum learning strategy is summarized as follows. Recall that for $i, r \in [k]$, we have defined the permutation $\hop_i^r(\sigma,\cdot) \in S_N$ by $\hop_i^r(\sigma,\cdot) = \sigma_{i + r - 1} \circ \pi_{i + r-1} \circ \cdots \circ \sigma_{i + 1} \circ \pi_{i + 1} \circ \sigma_{i} \circ \pi_{i}$. Intuitively, the difficulty of the task scales with $k$, and thus we expect it to be much easier for a transformer to learn from samples of $\hop_i^r(\sigma,\cdot)$, for much smaller $r \le k$. Our construction in Theorem \ref{thm:construction} also motivates a natural curriculum, as the model computes $f_\pi$ recursively with the output of the $\ell$-th layer being $\hop_i^{2^\ell - 1}(\sigma,\cdot)$. We thus consider a curriculum where in the $\ell$-th stage the model receives samples of $\hop_i^{2^{\ell - 1}}(\sigma,\cdot)$. We begin by describing this strategy and the gradient-based training algorithm in detail and show that it can learn $f_\pi$ in $\poly(N,k)$ samples.

\subsection{Training Procedure}
Our training algorithm is $L$-stage gradient descent with $L = \log k + 1$ on the cross entropy loss using the transformer architecture defined in \Cref{def:transformer}. We apply an easy-to-hard curriculum for training: for each stage $\ell$, we sample $M$ input sequences $\{\sigma^{(m)}\}_{m \in [M]}$, and for each sequence sample a query position $(i_m,j_m)$ and receive as label its $2^{\ell-1}$th hop $\hop_{i_m}^{2^{\ell-1}}(\sigma^{(m)}, j_m)$. The empirical training loss of stage $\ell$ is thus
\begin{equation}
\label{eqn: cross entropy loss for the stage l}
        \mathcal{L}^{(\ell)}(\theta) = -\frac{1}{M}\sum_{m=1}^M\qty[\sum_{s'\in[N]}\mathbf{1}\{s'=\hop^{2^{\ell-1}}_{i_m}(\sigma^{(m)}, j_m)\}\log(\S(\Psi_\ell^\top \mathrm{TF}_\theta(X_m)_{(i_m,j_m)})_{s'})],
\end{equation}
where we denote $X_m = X(\sigma^{(m)})$. The key-query matrices are initialized at $W_{KQ}^{(\ell)}(0) = 0_{d \times d}$, and the readout/unembedding layers are initialized at $\Psi_{\ell}(0)=\beta_0 e_{L+2,\ell+2}\otimes \mathbf{1}_{k}\otimes I_{N\times N}$ for initialization scale $\beta_0>0$. We fix the value matrix for each layer as $W_{OV}^{(\ell)}=e_{L+2,\ell+2}e_{L+2,\ell+1}^\top\otimes I_{kN\times kN}$ to match the sparsity pattern in the construction. See Section~\ref{sec:discussion} for a discussion on learning value matrices.

We define the parameter $\theta := (\theta_{KQ},\theta_\Psi)$, where $\theta_{KQ}=(W_{KQ}^{(1)},\cdots,W_{KQ}^{(L)})$ and $\theta_\Psi=(\Psi_1,\cdots,\Psi_{L})$.
Within each stage, our algorithm first takes one step of gradient descent on all key-query matrices $\theta_{KQ}$. We then take one gradient step on the readout layer $\theta_\Psi$.
Pseudocode is given in \Cref{alg:training_alg}.
\begin{algorithm}[t]
    \caption{Training Algorithm (Curriculum)}\label{alg:training_alg}
    \begin{algorithmic}
        \State \textbf{Input:} {initialization size $\beta_0$; learning rate $\eta$}
        
        \State {Initialize $W^{(\ell)}_{KQ}(0) = 0_{d \times d}, \Psi^{(\ell)}(0)=\beta_0 e_{L+2,\ell+2}\otimes \mathbf{1}_{k}\otimes I_{N\times N}$, $\ell\in[L]$}
        
        \For{$t=1,\cdots,L$}
        \State Use loss $\mathcal{L}^{(t)}(\theta(t-1))$.\Comment{Stage $t$: train on $\hop^{2^{t-1}}$}
        
        \State $\theta_{KQ}(t) \leftarrow \theta_{KQ}(t-1) - \eta\nabla_{\theta_{KQ}}\mathcal{L}^{(t)}(\theta{(t-1)})$ \Comment{Train the key-query matrices $\theta_{KQ}$}
        
        \State $\theta' \leftarrow (\theta_{KQ}(t),\theta_\Psi(t-1))$
        
        \State $\theta_\Psi(t) \leftarrow \theta_\Psi(t-1) - \eta\nabla_{\theta_\Psi}\mathcal{L}^{(t)}(\theta')$ \Comment{Train the readout layer $\theta_\Psi$}
        
        \State $\theta(t) \leftarrow (\theta_{KQ}(t),\theta_\Psi(t))$
        
        \EndFor
    \State \textbf{Output:} $\hat{\theta}={\theta(L)}$.
    \end{algorithmic}
\end{algorithm}

\subsection{Main Theorem}
Our main theorem is that $\hat{\theta}$, the output of the $L$ stage curriculum training in \Cref{alg:training_alg}, successfully learns all the $2^\ell$-hop functions for $\ell\le L$. 

\begin{theorem}[Guarantee for \Cref{alg:training_alg}] 
\label{thm:main}
Assume $k=2^{L-1}$, $M\ge \Tilde{\Omega}(k^4N^6)$, $0<\epsilon\le \tilde{O}(\frac{1}{k^2N^3})$, and $\eta \ge \tilde\Omega(\frac{k^2N^3}{\beta_0}\log\frac{1}{\epsilon})$.Then, with high probability the final output $\hat\theta$ of \Cref{alg:training_alg} satisfies that, over any draw of the input permutation $\sigma$ and the query index $(i,j)$, for any $\ell\in[L]$,
$$\sup_{\sigma,(i,j)}\norm{\S(\Psi_\ell^\top \mathrm{TF}_{\hat{\theta}}(X(\sigma))_{(i,j)})-e_{\hop^{2^{\ell-1}}_i(j)}}_\infty\le \epsilon.$$
In particular, for $\ell = L$, the output of the transformer $\hat f(\sigma, i, j) = \Psi_L^\top \mathrm{TF}_{\hat{\theta}}(X(\sigma))_{(i,j)}$ approximates the $k$-fold composition task $f^{\cyc}_\pi(\sigma, i, j)$.
\end{theorem}
Note that the $\mathrm{poly}(N,k)$ sample complexity with curriculum learning represents an significant improvement over the exponential dependence on $k$ in the SQ lower bound (Theorem~\ref{thm:SQ}). 
We provide a proof sketch in the next section with the complete proof deferred to \Cref{appendix:upper_bound}. For clarity, we denote $X:=X(\sigma),\hop_i^r(\sigma,j):=\hop_i^r(j)$ when $\sigma$ is clear from context.

\subsection{Proof Sketch}
\subsubsection{Stage 1: Learning the Hidden Permutation ($1$-hop)}\label{subsec:stage_1}
The first step of the proof is to show that during the first stage of training, the first attention layer $W_{KQ}^{(1)}$ learns the hidden permutations $\pi_i$ for all $i\in[k]$. The proof strategy is to first analyze the population dynamics, and then upper bound the sample noise by concentration.

We begin by decomposing the transformer output in the following summation:
{\small\begin{align*}\mathrm{TF}_\theta(X)=X^{(0)}+\sum_{\ell=1}^{L}W^{(\ell)}_{OV}X^{(\ell-1)} \S({X^{(\ell-1)}}^\top W^{(\ell)}_{KQ}X^{(\ell-1)}).
\end{align*}}

By our choice of initialization of the $\ell$-th readout layer $\Psi_\ell$ and value matrix $W^{(\ell)}_{OV}$, we observe that ${\Psi_{\ell'}^\top W_{OV}^{(\ell)}}$ is non-zero if and only if $\ell=\ell'$. This means the final relevant output for stage $\ell$ is the $\ell$-th layer $W^{(\ell)}_{OV}X^{(\ell-1)} \S({X^{(\ell-1)}}^\top W^{(\ell)}_{KQ}X^{(\ell-1)}),$ the gradient for the $\ell$-th layer $W^{(\ell)}_{KQ}$ is non-zero only in stage $\ge\ell$, and the gradient for a layer $W_{KQ}^{(\ell)}$ is also close to zero after being trained. We thus only consider the gradient of $W_{KQ}^{(\ell)}$ with respect to the loss $\mathcal{L}^{(\ell)}$ and upper bound the error term after each stage. In the first stage, the only updated key-query matrix is $W_{KQ}^{(1)}$. 

Now we show that after a large step of gradient descent,  the key-query matrix $W_{KQ}^{(1)}$ encodes the hidden permutations $\pi$, and the token in the $(i,j)$ position attends to the token in the $(i, \pi_i(j))$ position. Define $\mathcal{L}^{(\ell)}_{\mathcal{D}}:=\E[\mathcal{L}^{(\ell)}]$ as the population loss. The population gradient for $W_{KQ}^{(1)}$ can be computed as follows: 
{\small\begin{align*}
    \nabla_{W_{KQ}^{(1)}} \mathcal{L}_{\mathcal{D}}^{(1)} =-\E_{\sigma,(i,j)}\qty[\sum_{s'\in[N]}\qty(\mathbf{1}\{s'=\hop_i^{1}(j)\}-\S(\Psi_1^\top \mathrm{TF}_{(i,j)})_{s'}) XJ^{(1)} X^\top (\Psi_1^\top W_{OV}^{(1)})_{s'}^\top {X}_{(i,j)}^\top],
\end{align*}}

where we know that since $W_{KQ}^{(1)}$ is zero-initialized, $\S(\Psi_1^\top \mathrm{TF}_{(i,j)})_{s'}\!=\!\frac{1}{N}$ and the Jacobian term is $J^{(1)}\!=\!\frac{1}{kN}\qty(I_{kN}\!-\! \frac{1}{kN}\1_{kN}\1_{kN}^\top)$. Recall that the input embedding is $X_{(i,j)}  = {\scriptsize\begin{bmatrix}
        e_{k,i}\otimes e_{N,j} \\
        e_{k,i}\otimes e_{N,\sigma_i(j)}\\
        0_{kNL}
\end{bmatrix}}$. Since the readout layer is $(\Psi_{1}^\top W_{OV}^{(1)})_{s'}^\top=\beta_0 e_{L+2,2}\otimes \mathbf{1}_k \otimes e_{N,s'}$, one observes that $XX^\top(\Psi_{1}^\top W_{OV}^{(1)})_{s'}^\top = \beta_0 \sum_{p=1}^k X_{(p, s')}$. Finally, substituting $s' = \hop_i^1(j) = \sigma_i\pi_i(j)$, and dealing with cancellation in the mean centering term, one can simplify and expand the gradient into block matrices with dimension $\R^{kN\times kN}$ as follows:
{\small\begin{align*}
    \nabla_{W_{KQ}^{(1)}} \mathcal{L}_{\mathcal{D}}^{(1)}=-\frac{\beta_0}{kN}\E\qty[\qty(\begin{bmatrix}
        \sum_{p=1}^k e_{k,p}\otimes e_{N,\sigma_p^{-1}\sigma_i\pi_i(j)}\\
        \sum_{p=1}^k e_{k,p}\otimes e_{N,\sigma_i\pi_i(j)}\\
        0_{kNL}
    \end{bmatrix}-\frac1N\begin{bmatrix}
        \mathbf{1}_{kN}\\
        \mathbf{1}_{kN}\\
        0_{kNL}
    \end{bmatrix}) \begin{bmatrix}
        e_{k,i}\otimes e_{N,j} \\
        e_{k,i}\otimes e_{N,\sigma_i(j)}\\
        0_{kNL}
    \end{bmatrix}^\top]
\end{align*}}

Let $A_1$ denote the top left $kN \times kN$ block of $W_{KQ}^{(1)}$. To attend to $(i,\pi_i(j))$, our construction in Theorem \ref{thm:construction} uses $A_1$ to map $X_{(i,j)}$ to $X_{(i,\pi_i(j))}$; in particular, $A_1$ takes the first block in $X_{(i,j)}$, i.e. $e_{k,i}\otimes e_{N,j}$, and maps it to $e_{k,i}\otimes e_{N,\pi_i(j)}$. Computing the gradient of the population loss with respect to $A_1$, we observe that
{\small\begin{align*}
        \nabla_{A_1} \mathcal{L}^{(1)}_{\mathcal{D}} = -\frac{\beta_0}{kN}\E_{\sigma,(i,j)}\qty[\qty(
        \sum_{p=1}^k e_{k,p}\otimes e_{N,\sigma_p^{-1}\sigma_i\pi_i(j)}
        -\frac1N\mathbf{1}_{kN}) 
        (e_{k,i}\otimes e_{N,j})^\top].
\end{align*}}\\
Conditioning on the query position $(i,j)$, we observe that $\sigma_p$ is independent of the query embedding $X_{(i,j)}$ when $p\not=i$.
The terms with independent permutation $\sigma_p^{-1}$ will have expectation $\frac{1}{N}\1_{kN}$ and cancel with the mean-centering term. Therefore, the only remaining term is 
$e_{k,i}\otimes e_{N,\pi_i(j)}$, 
Taking expectation over all $(i,j)$, the gradient with respect to $A_1$ is thus {\small
$$\nabla_{A_1} \mathcal{L}^{(1)}=-\frac{\beta_0}{k^2N^2}\sum_{i=1}^k\sum_{j=1}^N \underbrace{(e_{k,i}\otimes e_{N,\pi_i(j)}) (e_{k,i}\otimes e_{N,j})^\top}_{\text{Maps }(i,j)\to(i,\pi_i(j))}+\frac{\beta_0}{k^2N^3}\sum_{i=1}^k\underbrace{(e_{k,i}\otimes \mathbf{1}_{N})(e_{k,i}\otimes \mathbf
{1}_{N})^\top}_{\text{Mean-centering term}}.$$}We thus see that $A_1$ encodes all the hidden permutations $\pi_i$ for all $i \in [k]$,


Moreover, we show that the gradient of the $A_1$ block dominates the gradient of the other blocks of $W_{KQ}^{(1)}$. Altogether, after concentrating the finite-sample gradient, we see that with high probability $\S(X^\top W_{KQ}^{(1)}X_{(i,j)})_{(i,\pi_i(j))}\approx 1$ and thus the first layer approximately outputs the 1-hop embedding
$$W^{(1)}_{OV}X^{(0)} \S({X^{(0)}}^\top W^{(1)}_{KQ}X^{(0)})_{(i,j)}\approx e_{L+2,3}\otimes e_{k,i}\otimes e_{N,\sigma_i\pi_i(j)}.$$
Next, after the second gradient step on the readout layer $\Psi_1$, the readout layer grows large in the direction $I_N-\frac{1}{N}\1_N\1_N^\top$ and approximately outputs the one-hot vector of the correct hop $e_{N,\hop_i^1(j)}=e_{N,\sigma_i\pi_i(j)}$ after the output softmax layer. 



\subsubsection{Stage $\ell$ ($2\le \ell\le 1+\log_2k$): Learning the $2^{\ell-1}$-hop}
\label{subsec:stage_ell}
Next, we inductively show that in the $\ell$-th stage, the $\ell$-th layer learns to compute the $2^{\ell-1}$-hop. In particular, we show that for a query $(i,j)\in[k]\times [N]$, the key-query matrix learns to compose $\hop_i^{2^{\ell-2}}$ and $\hop^{2^{\ell-2}}_{i+2^{\ell-2}}$ computed in the previous layer. 

In \Cref{lemma: empirical gradient of the second stage} ($\ell=2$) and \Cref{lemma: empirical gradient of stage ell} ($\ell\ge 3$), we show that a one-step gradient update on $W_{KQ}^{(\ell)}$ approximately matches the constructed $\ell$-th layer key-query matrix in \Cref{thm:construction}.
After the gradient descent step, in the $\ell$-th layer the $(i, j)$ position correctly attends to the $(i+2^{\ell-2},\hop_{i}^{2^{\ell-2}}(j))$ position and extracts the $2^{\ell-2}$-hop, thus computing the desired output $e_{N,\hop_i^{2^{\ell-1}}(j)}$.

Finally, we upper bound the total error from all $\ell$ stages altogether. We show inductively that the accumulation of errors in the actual intermediate sequences $\hat{X}^{(\ell-1)}$ (resulting from the finite-sample gradient and finite learning rate) grows linearly with depth, i.e. $\|\hat{X}^{(\ell)}-X^{(\ell)}\|_\infty\le \ell\epsilon'$, where $\epsilon'$ is the single stage error. For $\epsilon'$ sufficiently small, the output of the transformer is indeed close to the desired solution. To conclude the proof of \Cref{thm:main}, we finally train $\Psi_L$. As in Sec. \ref{subsec:stage_1}, $\Psi_L$ grows large and the transformer outputs $e_{N,\hop_i^{2^{\ell-1}}(j)}=e_{N,\hop^k_i(j)}$, as desired.

\subsection{Implicit Curriculum via Data Mixture}
In practice, designing a curriculum involving multiple stages with tasks of increasing difficulty can be quite challenging. As such, practitioners often rely on training data consisting of mixtures of various tasks with different skills and difficulties~\citep{xie2024doremi, liu2024regmix, dubey2024llama}. We thus consider an arguably simpler training scheme, where instead of stage-wise learning with a curriculum, we directly train on a \textit{mixture of easy-to-hard data} simultaneously. Consider the following objective for the mixture problem, which is the empirical loss summed over tasks with varying hops:
\begin{equation}\small
\mathcal{L}^{\mathrm{M}}(\theta) := \sum_{\ell=1}^L \mathcal{L}^{(\ell)}(\theta)= -\frac{1}{M}\sum_{\ell=1}^L\sum_{m=1}^M\Bigg[\sum_{s'\in[N]}\mathbf{1}\{s'\!=\!\hop^{2^{\ell-1}}_{i_m}(\sigma^{(m)}, j_m)\}\log(\S(\Psi_\ell^\top \mathrm{TF}_\theta(X_m)_{(i_m,j_m)})_{s'})\Bigg]
\end{equation}

We remark that this objective
is equivalent, at the population level, to predicting multiple hops for each input sequence at the same time, which relates to \textit{multi-token prediction}, a technique that has recently been found to be effective in practice~\citep{bachmann2024pitfalls,gloeckle2024better,liu2024deepseek}. Pseudocode for the mixed training algorithm is presented in \Cref{alg:training_alg_mix}.

\begin{algorithm}[t]
    \caption{Training Algorithm (Data Mixture)}\label{alg:training_alg_mix}
    \begin{algorithmic}
        \State{\textbf{Input:} initialization size $\beta_0$; learning rate $\eta$}
        \State Initialize $W^{(\ell)}_{KQ}(0) = 0_{d \times d}, \Psi^{(\ell)}(0)=\beta_0 e_{L+2,\ell+2}\otimes \mathbf{1}_{k}\otimes I_{N\times N}$, $\ell\in[L]$
        
        \For{$t=1,\dots,L$} \Comment{Train on mixture of all $2^{\ell-1}$ hops}        
        
        \State $\theta_{KQ}(t) \leftarrow \theta_{KQ}(t-1) - \eta\nabla_{\theta_{KQ}}\mathcal{L}^{\mathrm{M}}(\theta{(t-1)})$ \Comment{Train the key-query matrices $\theta_{KQ}$}
        \State $\theta(t) \leftarrow (\theta_{KQ}(t),\theta_\Psi(0))$
        \EndFor        
        \State $\theta_\Psi(L) \leftarrow \theta_\Psi(0) - \eta\nabla_{\theta_\Psi}\mathcal{L}^{\mathrm{M}}(\theta(L))$ \Comment{Train the readout layer $\theta_\Psi$}
        \State $\theta(L) \leftarrow (\theta_{KQ}(L),\theta_\Psi(L))$
    \State \textbf{Output:} $\hat{\theta}={\theta(L)}$.
    \end{algorithmic}
\end{algorithm}


The following theorem shows that training on a mixture of easy-to-hard data also enables a transformer to learn the $k$-fold composition task by inducing an equivalent curriculum as \Cref{alg:training_alg}.

\begin{restatable}[Guarantee for mixed data training]{theorem}{mixed}\label{thm:mix-data-main}
    Assume $k=2^{L-1}$, $M\ge \Tilde{\Omega}(k^4N^6)$ and $\eta \ge \tilde\Omega(\frac{k^2N^3}{\beta_0}\log\frac{1}{\epsilon}).$ For sufficiently small $\epsilon>0$, with high probability the final output $\hat\theta$ of \Cref{alg:training_alg_mix} satisfies that over any draw of input permutations $\sigma$ and query index $(i,j)$, for any $\ell\in[L]$,
$$\sup_{\sigma,(i,j)}\norm{\S(\Psi_\ell^\top \mathrm{TF}_{\hat{\theta}}(X(\sigma))_{(i,j)})-e_{\hop^{2^{\ell-1}}_i(j)}}_\infty\le \epsilon.$$
In particular, the transformer approximates the $k$-fold composition task when $\ell=L$.
\end{restatable}

The proof of \Cref{thm:mix-data-main} is deferred to \Cref{appendix:mixture}. The high level proof idea is as follows.
Consider an ``idealized'' population gradient dynamics with all the attention outputs either one-hot or uniform. The key observation is that if all layers starting from the $t$-th layer are zero, i.e. $W_{KQ}^{(\ell)}=0$ for $\ell\ge t$, the gradients of all later layers $\nabla_{W_{KQ}^{(\ell)}}\mathcal{L}^{\mathrm{M}}=0$ for $\ell \ge t+1$. This is because if the layer $W_{KQ}^{(\ell-1)}$ is zero, the output of the softmax will be the uniform distribution $\frac{1}{kN}\1_{kN}$, which contains no signal. This will cancel out with the mean-centering term in the Jacobian, leading to zero gradient. We can therefore show that in the idealized population dynamics, the $\ell$-th layer is learned in the $\ell$-th gradient step, thus mimicking \Cref{alg:training_alg}.

\section{Experiments}\label{sec:simulations}

In this section, we provide empirical support for the conclusions of \Cref{thm:main} and \Cref{thm:mix-data-main}. In the leftmost plot of \Cref{fig:experiments}, we consider training a 5 layer transformer with architecture and initialization exactly matching that of \Cref{thm:main} with $k=16, N=5$. Our curriculum training procedure exactly follows that of the \Cref{alg:training_alg}: for $\ell \in [5]$, during the $\ell$-th stage we train on $2^{\ell-1}$ hop data, and take a single gradient step on $W_{KQ}^{(\ell)}$ followed by a single gradient step on $\Psi^{(\ell)}$. We observe that the model's loss on $2^{\ell-1}$-hop examples decreases to zero after the $\ell$-th stage, and thus after the final stage the model has perfectly learned the 16-fold composition. In the middle pane, we train a 4 layer transformer with architecture, initialization, and training procedure exactly matching that of the mixed training algorithm (\Cref{alg:training_alg_mix}), with $k = 8, N = 5$. We observe that for $\ell \in [4]$, the model's loss on the $2^{\ell-1}$-hop examples decreases, and at the end of training,\footnote{In \Cref{alg:training_alg_mix}, the first 4 steps are on the $W^{(\ell)}_{KQ}$ matrices, and the remainder of the steps are on the readout matrices $\Psi_{\ell}$, and thus the loss can only be decreased to 0 once the readout matrices are trained.} the model has perfectly learned the 8-fold composition.

In the rightmost plot, we consider training a standard transformer on the 4-fold composition task. To more closely match standard language modeling tasks, each token in the sequence is simply the element $\sigma_i(j) \in [N]$, and the desired hop number is prepended to the beginning of the sequence. We use Adam \cite{kingma2014adam} as the optimizer and train a standard encoder transformer with learned embeddings, MLPs and layer normalization. For the training distribution, we first train the transformer with uniformly mixed $M=10^5$ examples from $\{1,2,4\}$-fold composition data, and then train another transformer with only 4-fold compositions. We observe a similar phenomenology as predicted by \Cref{thm:main} -- if we only train on the 4-fold composition data, then the transformer is unable to learn, yet training with a curriculum learning strategy helps the model to correctly learn the 4-fold task.

\begin{figure}
\centering
    \includegraphics[width=0.33\textwidth]{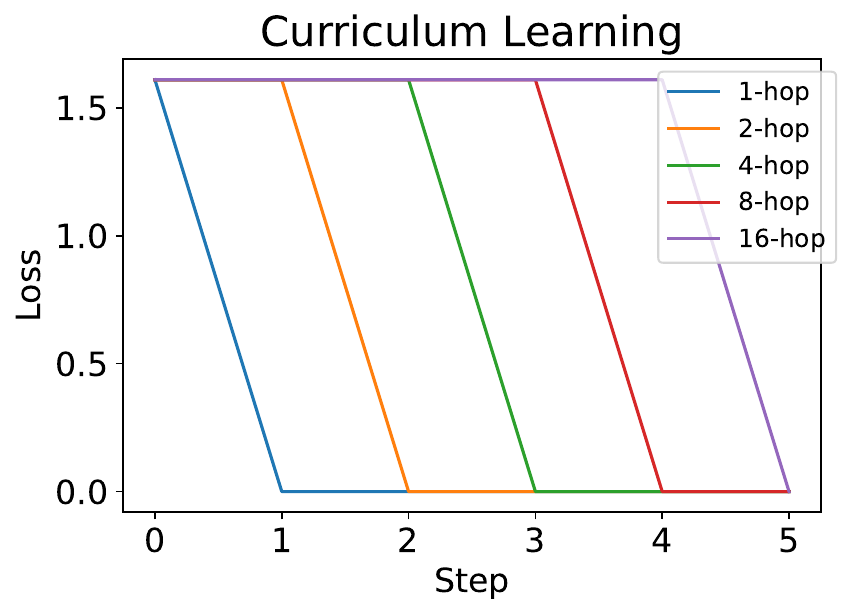}
    \includegraphics[width=0.33\textwidth]{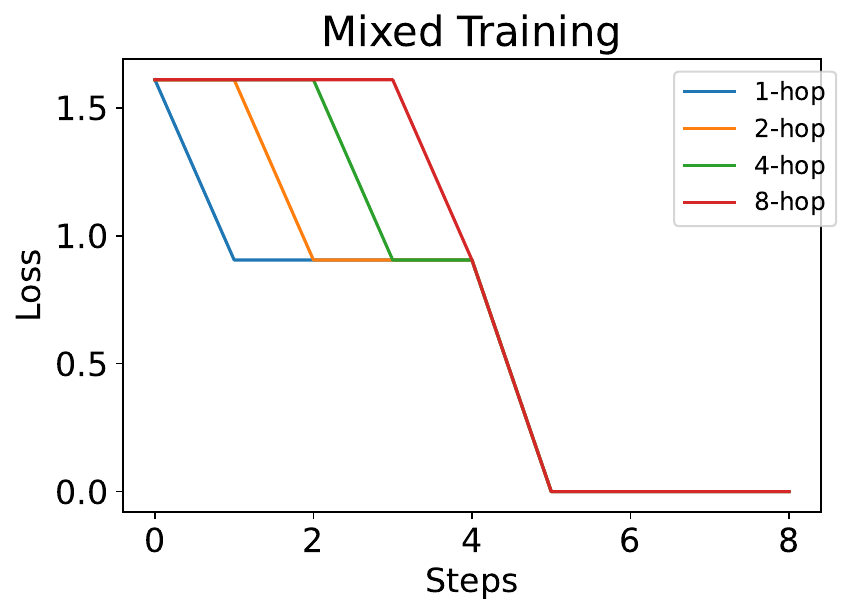}
    \includegraphics[width=0.32\textwidth]{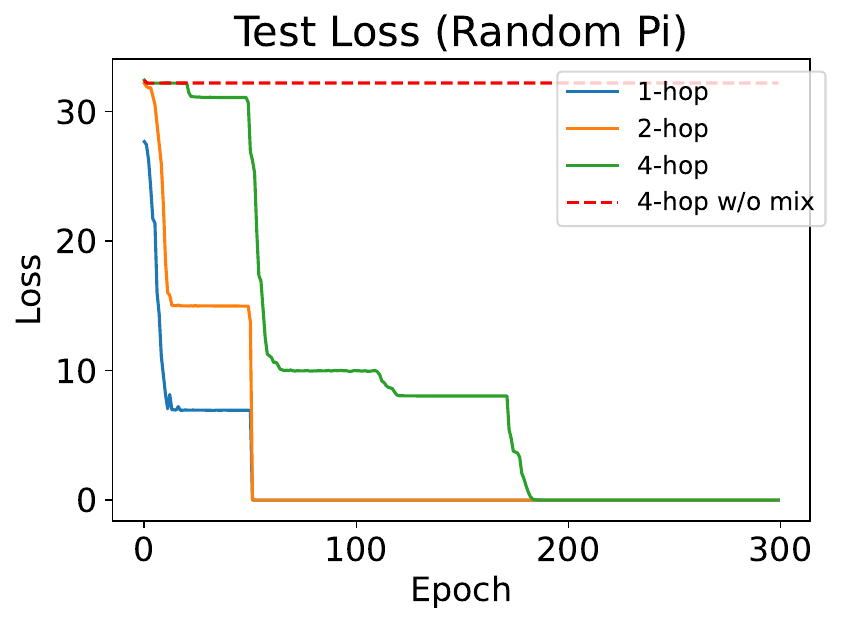}
    \caption{\small \textbf{Left:} Curriculum learning (Algorithm \ref{alg:training_alg}). \textbf{Middle:} Learning with data mixture (Algorithm \ref{alg:training_alg_mix}).
    \textbf{Right:} Comparison between training with and without mixed data on standard encoder transformer.}
\label{fig:experiments}
\end{figure}

\section{Discussion and Future Directions}\label{sec:discussion}

\paragraph{Connection to $k$-sparse parity.} 
Theorem \ref{thm:SQ} implies a statistical-to-computational gap under the SQ framework: a covering number argument on the function class $\mathcal{F}$ yields an information-theoretic sample complexity of $kN\log N$, whereas an SQ learner with polynomial compute requires $N^{\Omega(k)}$ samples. 
This mirrors the case for $k$-sparse parities over $d$ variables, where the SQ complexity $q/\tau^{2}\sim d^{\Omega(k)}$ is also statistically suboptimal. 
Heuristically speaking, both problems exhibit a ``global'' structure in which intermediate steps (partial solutions) are uncorrelated with the true function, and hence a learner using only correlational information pays a complexity exponential in $k$. For more discussion on similar computational lower bounds see \citet[Appendix A.4]{abbe2024far}. 

For the parity problem, recent works have shown that the statistical and optimization complexity can be improved by introducing a {curriculum} \citep{abbe2023provable,panigrahi2024progressive}. More relevant to our results, \cite{kim2024transformers,wen2024sparse} showed that by decomposing the problem into subtasks of intermediate parities and employing process supervision (i.e., forcing the model to predict the intermediate steps), a single-layer transformer can learn $k$-parity with $\text{poly}(d,k)$ samples, thereby avoiding exponential dependence in $k$. At a high level, such task decomposition resembles our curriculum learning objective, which constructs a ``staircase'' \citep{abbe2021staircase,abbe2022merged} to guide gradient-based learning. 
Note that while parity can be represented by a single-layer transformer, it is believed that for tasks similar to our $k$-fold composition such as the $k$-hop induction head, $\log k$ layers are necessary for parameter-efficient representation \citep{sanford2024transformers}.

\paragraph{On the embedding dimension and model size.} \citet{sanford2024transformers} show an advantage of transformers over recurrent neural networks (RNNs) for expressing the $k$-hop induction head task. Via a reduction to a communication complexity lower bound for the pointer-chasing problem, an RNN requires either depth at least $k$ or width at least $N/\poly(k)$ to solve the $k$-hop task. In contrast, the $O(\log k)$ depth transformer construction succeeds with embedding dimension $d = O(1)$. Similarly, the $O(\log k)$-depth construction for the automata simulation task in \citet{liu2023transformers} requires embedding dimension and MLP width of $\poly(N)$, independent of $k$. In comparison, our construction in \Cref{thm:construction} appears suboptimal in that the embedding dimension is $\poly(Nk)$. However, since we require the embedding $\phi$ of the transformer to be fixed (that is, independent, of the target function $f_\pi$),
the size of the representation is $Ld^2p$, where $p$ is the bit precision. The representation size must be at least the log packing number, which is $\Theta(kN\log N)$ (\Cref{cor:packing}).
As such, an embedding dimension of $d = \poly(Nk)$ is necessary when $p = \tilde O(1)$. On the other hand, if we allow for $\phi$ to be trainable (i.e., depend on the target $f_\pi$), then we can encode $\hop_i^1(\sigma, j)$ using $\phi$ to yield a valid construction with embedding dimension $d = \tilde O(1)$, for example, by leveraging rotary embeddings similar to~\cite{sanford2023representational}.
We leave understanding whether the learned model can be distilled into a smaller dimensional model, or whether a smaller model with a trainable embedding map can still learn the task, to future work.

\paragraph{Learning the value matrix.} 
Our analysis of learning dynamics in Section~\ref{sec:ub} assumes that the value matrix~$W_{OV}^{(\ell)}$ at each layer is fixed to the desired block-sparse matrix with an identity block that matches our construction in Theorem~\ref{thm:construction}.
We note that while this simplifies the problem, it does not encode any information about the hidden permutations~$\pi$ in the target~$f_\pi$.
In practice, these matrices are not fixed at the identity, and we show in Appendix~\ref{appx:value_matrix} that their population gradient at zero initialization vanishes under our data model. This suggests that other ingredients are needed for successfully learning the value matrices, such as exploiting a non-zero initialization or changing the data distribution. We leave a precise study of such learning dynamics to future work.

\paragraph{Generalization to $k\not=2^\ell$.} Our analysis assumes for simplicity that $k$ is a power of two. When $k$ is not a power of two, the construction in \Cref{thm:construction} does not work, as keeping only the $2^{\ell}$-hops is insufficient.
For general $k$, we believe that a similar analysis applies provided
a larger embedding dimension $\Theta(k^2N)$ to encode all $\ell$-hops with $\ell\in[k]$.

\bigskip

\subsection*{Acknowledgments}

This collaboration began during the ``Modern Paradigms in Generalization'' and ``Special Year on Large Language Models and Transformers, Part 1'' programs at the Simons Institute for the Theory of Computing, Berkeley in 2024.
DH acknowledges support from the ONR under grant N00014-24-1-2700. JDL acknowledges support of the NSF CCF 2002272, NSF IIS 2107304, NSF CIF 2212262, ONR Young Investigator Award, and NSF CAREER Award 2144994.

\bigskip

\bibliography{main}

\begin{thebibliography}{68}
\providecommand{\natexlab}[1]{#1}
\providecommand{\url}[1]{\texttt{#1}}
\expandafter\ifx\csname urlstyle\endcsname\relax
  \providecommand{\doi}[1]{doi: #1}\else
  \providecommand{\doi}{doi: \begingroup \urlstyle{rm}\Url}\fi

\bibitem[Abbe et~al.(2021)Abbe, Boix-Adsera, Brennan, Bresler, and Nagaraj]{abbe2021staircase}
Emmanuel Abbe, Enric Boix-Adsera, Matthew~S Brennan, Guy Bresler, and Dheeraj Nagaraj.
\newblock The staircase property: How hierarchical structure can guide deep learning.
\newblock \emph{Advances in Neural Information Processing Systems}, 34:\penalty0 26989--27002, 2021.

\bibitem[Abbe et~al.(2022)Abbe, Adsera, and Misiakiewicz]{abbe2022merged}
Emmanuel Abbe, Enric~Boix Adsera, and Theodor Misiakiewicz.
\newblock The merged-staircase property: a necessary and nearly sufficient condition for sgd learning of sparse functions on two-layer neural networks.
\newblock In \emph{Conference on Learning Theory}, pages 4782--4887. PMLR, 2022.

\bibitem[Abbe et~al.(2023)Abbe, Cornacchia, and Lotfi]{abbe2023provable}
Emmanuel Abbe, Elisabetta Cornacchia, and Aryo Lotfi.
\newblock Provable advantage of curriculum learning on parity targets with mixed inputs.
\newblock \emph{Advances in Neural Information Processing Systems}, 36:\penalty0 24291--24321, 2023.

\bibitem[Abbe et~al.(2024)Abbe, Bengio, Lotfi, Sandon, and Saremi]{abbe2024far}
Emmanuel Abbe, Samy Bengio, Aryo Lotfi, Colin Sandon, and Omid Saremi.
\newblock How far can transformers reason? {T}he locality barrier and inductive scratchpad.
\newblock In \emph{Advances in Neural Information Processing Systems}, 2024.

\bibitem[Ahn et~al.(2024)Ahn, Cheng, Daneshmand, and Sra]{ahn2024transformers}
Kwangjun Ahn, Xiang Cheng, Hadi Daneshmand, and Suvrit Sra.
\newblock Transformers learn to implement preconditioned gradient descent for in-context learning.
\newblock In \emph{Advances in Neural Information Processing Systems}, volume~36, 2024.

\bibitem[Assadi and N(2021)]{assadi2021graph}
Sepehr Assadi and Vishvajeet N.
\newblock Graph streaming lower bounds for parameter estimation and property testing via a streaming xor lemma.
\newblock In \emph{Proceedings of the 53rd Annual ACM SIGACT Symposium on Theory of Computing}, pages 612--625, 2021.

\bibitem[Bachmann and Nagarajan(2024)]{bachmann2024pitfalls}
Gregor Bachmann and Vaishnavh Nagarajan.
\newblock The pitfalls of next-token prediction.
\newblock In \emph{International Conference on Machine Learning (ICML)}, 2024.

\bibitem[Bengio et~al.(2009)Bengio, Louradour, Collobert, and Weston]{bengiocurriculum}
Yoshua Bengio, J\'{e}r\^{o}me Louradour, Ronan Collobert, and Jason Weston.
\newblock Curriculum learning.
\newblock In \emph{Proceedings of the 26th Annual International Conference on Machine Learning}, ICML '09, page 41–48, New York, NY, USA, 2009. Association for Computing Machinery.
\newblock ISBN 9781605585161.
\newblock \doi{10.1145/1553374.1553380}.
\newblock URL \url{https://doi.org/10.1145/1553374.1553380}.

\bibitem[Bhattamishra et~al.(2024)Bhattamishra, Hahn, Blunsom, and Kanade]{bhattamishra2024separations}
Satwik Bhattamishra, Michael Hahn, Phil Blunsom, and Varun Kanade.
\newblock Separations in the representational capabilities of transformers and recurrent architectures.
\newblock \emph{arXiv preprint arXiv:2406.09347}, 2024.

\bibitem[Bietti et~al.(2023)Bietti, Cabannes, Bouchacourt, Jegou, and Bottou]{bietti2023birth}
Alberto Bietti, Vivien Cabannes, Diane Bouchacourt, Herve Jegou, and Leon Bottou.
\newblock Birth of a transformer: A memory viewpoint.
\newblock In \emph{Advances in Neural Information Processing Systems (NeurIPS)}, 2023.

\bibitem[Chan et~al.(2022)Chan, Santoro, Lampinen, Wang, Singh, Richemond, McClelland, and Hill]{chan2022data}
Stephanie Chan, Adam Santoro, Andrew Lampinen, Jane Wang, Aaditya Singh, Pierre Richemond, James McClelland, and Felix Hill.
\newblock Data distributional properties drive emergent in-context learning in transformers.
\newblock In \emph{Advances in neural information processing systems}, 2022.

\bibitem[Chen et~al.(2024{\natexlab{a}})Chen, Peng, and Wu]{chen2024theoretical}
Lijie Chen, Binghui Peng, and Hongxun Wu.
\newblock Theoretical limitations of multi-layer transformer.
\newblock \emph{arXiv preprint arXiv:2412.02975}, 2024{\natexlab{a}}.

\bibitem[Chen et~al.(2024{\natexlab{b}})Chen, Sheen, Wang, and Yang]{chen2024training}
Siyu Chen, Heejune Sheen, Tianhao Wang, and Zhuoran Yang.
\newblock Training dynamics of multi-head softmax attention for in-context learning: Emergence, convergence, and optimality.
\newblock \emph{arXiv preprint arXiv:2402.19442}, 2024{\natexlab{b}}.

\bibitem[Chen et~al.(2024{\natexlab{c}})Chen, Sheen, Wang, and Yang]{chen2024unveiling}
Siyu Chen, Heejune Sheen, Tianhao Wang, and Zhuoran Yang.
\newblock Unveiling induction heads: Provable training dynamics and feature learning in transformers.
\newblock \emph{arXiv preprint arXiv:2409.10559}, 2024{\natexlab{c}}.

\bibitem[Cheng et~al.(2024)Cheng, Pan, Yin, Wang, and Wang]{cheng2024understanding}
Sitao Cheng, Liangming Pan, Xunjian Yin, Xinyi Wang, and William~Yang Wang.
\newblock Understanding the interplay between parametric and contextual knowledge for large language models.
\newblock \emph{arXiv preprint arXiv:2410.08414}, 2024.

\bibitem[Damian et~al.(2022)Damian, Lee, and Soltanolkotabi]{damian2022neural}
Alexandru Damian, Jason Lee, and Mahdi Soltanolkotabi.
\newblock Neural networks can learn representations with gradient descent.
\newblock In \emph{Conference on Learning Theory}, pages 5413--5452. PMLR, 2022.

\bibitem[Deng et~al.(2024)Deng, Choi, and Shieber]{deng2024explicit}
Yuntian Deng, Yejin Choi, and Stuart Shieber.
\newblock From explicit cot to implicit cot: Learning to internalize cot step by step.
\newblock \emph{arXiv preprint arXiv:2405.14838}, 2024.

\bibitem[Dubey et~al.(2024)Dubey, Jauhri, Pandey, Kadian, Al-Dahle, Letman, Mathur, Schelten, Yang, Fan, et~al.]{dubey2024llama}
Abhimanyu Dubey, Abhinav Jauhri, Abhinav Pandey, Abhishek Kadian, Ahmad Al-Dahle, Aiesha Letman, Akhil Mathur, Alan Schelten, Amy Yang, Angela Fan, et~al.
\newblock The llama 3 herd of models.
\newblock \emph{arXiv preprint arXiv:2407.21783}, 2024.

\bibitem[Dziri et~al.(2023)Dziri, Lu, Sclar, Li, Jiang, Lin, West, Bhagavatula, Le~Bras, Hwang, Sanyal, Welleck, Ren, Ettinger, Harchaoui, and Choi]{dziri2023faith}
Nouha Dziri, Ximing Lu, Melanie Sclar, Xiang~Lorraine Li, Liwei Jiang, Bill~Yuchen Lin, Peter West, Chandra Bhagavatula, Ronan Le~Bras, Jena~D. Hwang, Soumya Sanyal, Sean Welleck, Xiang Ren, Allyson Ettinger, Zaid Harchaoui, and Yejin Choi.
\newblock Faith and fate: Limits of transformers on compositionality (2023).
\newblock \emph{arXiv preprint arXiv:2305.18654}, 2023.

\bibitem[Elman(1993)]{elman1993learning}
Jeffrey~L Elman.
\newblock Learning and development in neural networks: The importance of starting small.
\newblock \emph{Cognition}, 48\penalty0 (1):\penalty0 71--99, 1993.

\bibitem[Feng et~al.(2024)Feng, Zhang, Gu, Ye, He, and Wang]{feng2024towards}
Guhao Feng, Bohang Zhang, Yuntian Gu, Haotian Ye, Di~He, and Liwei Wang.
\newblock Towards revealing the mystery behind chain of thought: a theoretical perspective.
\newblock \emph{Advances in Neural Information Processing Systems}, 36, 2024.

\bibitem[Gao et~al.(2024)Gao, Cao, Li, He, Wang, Liu, Klusowski, and Fan]{gao2024global}
Cheng Gao, Yuan Cao, Zihao Li, Yihan He, Mengdi Wang, Han Liu, Jason~Matthew Klusowski, and Jianqing Fan.
\newblock Global convergence in training large-scale transformers.
\newblock \emph{arXiv preprint arXiv:2410.23610}, 2024.

\bibitem[Gloeckle et~al.(2024)Gloeckle, Idrissi, Rozi{\`e}re, Lopez-Paz, and Synnaeve]{gloeckle2024better}
Fabian Gloeckle, Badr~Youbi Idrissi, Baptiste Rozi{\`e}re, David Lopez-Paz, and Gabriel Synnaeve.
\newblock Better \& faster large language models via multi-token prediction.
\newblock \emph{arXiv preprint arXiv:2404.19737}, 2024.

\bibitem[Gowers and Viola(2019)]{gowers2019interleaved}
William~Timothy Gowers and Emanuele Viola.
\newblock Interleaved group products.
\newblock \emph{SIAM Journal on Computing}, 48\penalty0 (2):\penalty0 554--580, 2019.

\bibitem[Huang et~al.(2025)Huang, Wang, and Lee]{huang2025transformers}
Jianhao Huang, Zixuan Wang, and Jason~D Lee.
\newblock Transformers learn to implement multi-step gradient descent with chain of thought.
\newblock \emph{arXiv preprint arXiv:2502.21212}, 2025.

\bibitem[Huang et~al.(2023)Huang, Cheng, and Liang]{huang2023context}
Yu~Huang, Yuan Cheng, and Yingbin Liang.
\newblock In-context convergence of transformers.
\newblock \emph{arXiv preprint arXiv:2310.05249}, 2023.

\bibitem[Im et~al.(2023)Im, Kumar, Lattanzi, Moseley, and Vassilvitskii]{im2023massively}
Sungjin Im, Ravi Kumar, Silvio Lattanzi, Benjamin Moseley, and Sergei Vassilvitskii.
\newblock Massively parallel computation: Algorithms and applications.
\newblock \emph{Foundations and Trends{\textregistered} in Optimization}, 5\penalty0 (4):\penalty0 340--417, 2023.

\bibitem[Jelassi et~al.(2022)Jelassi, Sander, and Li]{jelassi2022vision}
Samy Jelassi, Michael Sander, and Yuanzhi Li.
\newblock Vision transformers provably learn spatial structure.
\newblock In \emph{Advances in Neural Information Processing Systems (NeurIPS)}, 2022.

\bibitem[Jelassi et~al.(2024)Jelassi, Brandfonbrener, Kakade, and Malach]{jelassi2024repeat}
Samy Jelassi, David Brandfonbrener, Sham~M Kakade, and Eran Malach.
\newblock Repeat after me: Transformers are better than state space models at copying.
\newblock \emph{arXiv preprint arXiv:2402.01032}, 2024.

\bibitem[Kearns(1998)]{10.1145/293347.293351}
Michael Kearns.
\newblock Efficient noise-tolerant learning from statistical queries.
\newblock \emph{J. ACM}, 45\penalty0 (6):\penalty0 983–1006, 1998.

\bibitem[Kim and Suzuki(2024)]{kim2024transformers}
Juno Kim and Taiji Suzuki.
\newblock Transformers provably solve parity efficiently with chain of thought.
\newblock \emph{arXiv preprint arXiv:2410.08633}, 2024.

\bibitem[Kingma(2014)]{kingma2014adam}
Diederik~P Kingma.
\newblock Adam: A method for stochastic optimization.
\newblock \emph{arXiv preprint arXiv:1412.6980}, 2014.

\bibitem[Lanchantin et~al.(2024)Lanchantin, Toshniwal, Weston, and Sukhbaatar]{lanchantin2024learning}
Jack Lanchantin, Shubham Toshniwal, Jason Weston, and Sainbayar Sukhbaatar.
\newblock Learning to reason and memorize with self-notes.
\newblock In \emph{Advances in Neural Information Processing Systems}, volume~36, 2024.

\bibitem[Lewkowycz et~al.(2022)Lewkowycz, Andreassen, Dohan, Dyer, Michalewski, Ramasesh, Slone, Anil, Schlag, Gutman-Solo, Wu, Neyshabur, Gur-Ari, and Misra]{lewkowycz2022solving}
Aitor Lewkowycz, Anders Andreassen, David Dohan, Ethan Dyer, Henryk Michalewski, Vinay Ramasesh, Ambrose Slone, Cem Anil, Imanol Schlag, Theo Gutman-Solo, Yuhuai Wu, Behnam Neyshabur, Guy Gur-Ari, and Vedant Misra.
\newblock Solving quantitative reasoning problems with language models.
\newblock In \emph{Advances in Neural Information Processing Systems}, volume~35, 2022.

\bibitem[Li et~al.(2023)Li, Li, and Risteski]{li2023transformers}
Yuchen Li, Yuanzhi Li, and Andrej Risteski.
\newblock How do transformers learn topic structure: Towards a mechanistic understanding.
\newblock In \emph{ICML}, 2023.

\bibitem[Li et~al.(2024)Li, Liu, Zhou, and Ma]{li2024chain}
Zhiyuan Li, Hong Liu, Denny Zhou, and Tengyu Ma.
\newblock Chain of thought empowers transformers to solve inherently serial problems.
\newblock \emph{arXiv preprint arXiv:2402.12875}, 2024.

\bibitem[Lightman et~al.(2023)Lightman, Kosaraju, Burda, Edwards, Baker, Lee, Leike, Schulman, Sutskever, and Cobbe]{lightman2023let}
Hunter Lightman, Vineet Kosaraju, Yura Burda, Harri Edwards, Bowen Baker, Teddy Lee, Jan Leike, John Schulman, Ilya Sutskever, and Karl Cobbe.
\newblock Let's verify step by step.
\newblock \emph{arXiv preprint arXiv:2305.20050}, 2023.

\bibitem[Liu et~al.(2024{\natexlab{a}})Liu, Feng, Xue, Wang, Wu, Lu, Zhao, Deng, Zhang, Ruan, et~al.]{liu2024deepseek}
Aixin Liu, Bei Feng, Bing Xue, Bingxuan Wang, Bochao Wu, Chengda Lu, Chenggang Zhao, Chengqi Deng, Chenyu Zhang, Chong Ruan, et~al.
\newblock Deepseek-v3 technical report.
\newblock \emph{arXiv preprint arXiv:2412.19437}, 2024{\natexlab{a}}.

\bibitem[Liu et~al.(2023)Liu, Ash, Goel, Krishnamurthy, and Zhang]{liu2023transformers}
Bingbin Liu, Jordan~T. Ash, Surbhi Goel, Akshay Krishnamurthy, and Cyril Zhang.
\newblock Transformers learn shortcuts to automata.
\newblock In \emph{The Eleventh International Conference on Learning Representations}, 2023.
\newblock URL \url{https://openreview.net/forum?id=De4FYqjFueZ}.

\bibitem[Liu et~al.(2024{\natexlab{b}})Liu, Zheng, Muennighoff, Zeng, Dou, Pang, Jiang, and Lin]{liu2024regmix}
Qian Liu, Xiaosen Zheng, Niklas Muennighoff, Guangtao Zeng, Longxu Dou, Tianyu Pang, Jing Jiang, and Min Lin.
\newblock Regmix: Data mixture as regression for language model pre-training.
\newblock \emph{arXiv preprint arXiv:2407.01492}, 2024{\natexlab{b}}.

\bibitem[Merrill and Sabharwal(2023{\natexlab{a}})]{merrill2023expresssive}
William Merrill and Ashish Sabharwal.
\newblock The expresssive power of transformers with chain of thought.
\newblock \emph{arXiv preprint arXiv:2310.07923}, 2023{\natexlab{a}}.

\bibitem[Merrill and Sabharwal(2023{\natexlab{b}})]{merrill2023parallelism}
William Merrill and Ashish Sabharwal.
\newblock The parallelism tradeoff: Limitations of log-precision transformers.
\newblock \emph{Transactions of the Association for Computational Linguistics}, 11:\penalty0 531--545, 2023{\natexlab{b}}.

\bibitem[Nichani et~al.(2024{\natexlab{a}})Nichani, Damian, and Lee]{nichani2024transformers}
Eshaan Nichani, Alex Damian, and Jason~D Lee.
\newblock How transformers learn causal structure with gradient descent.
\newblock In \emph{International Conference on Machine Learning (ICML)}, 2024{\natexlab{a}}.

\bibitem[Nichani et~al.(2024{\natexlab{b}})Nichani, Lee, and Bietti]{nichani2024understanding}
Eshaan Nichani, Jason~D Lee, and Alberto Bietti.
\newblock Understanding factual recall in transformers via associative memories.
\newblock \emph{arXiv preprint arXiv:2412.06538}, 2024{\natexlab{b}}.

\bibitem[Nisan and Wigderson(1993)]{nisancommunication}
Noam Nisan and Avi Wigderson.
\newblock Rounds in communication complexity revisited.
\newblock \emph{SIAM Journal on Computing}, 22\penalty0 (1):\penalty0 211--219, 1993.
\newblock \doi{10.1137/0222016}.
\newblock URL \url{https://doi.org/10.1137/0222016}.

\bibitem[Nye et~al.(2021)Nye, Andreassen, Gur-Ari, Michalewski, Austin, Bieber, Dohan, Lewkowycz, Bosma, Luan, Sutton, and Odena]{nye2021show}
Maxwell Nye, Anders~Johan Andreassen, Guy Gur-Ari, Henryk Michalewski, Jacob Austin, David Bieber, David Dohan, Aitor Lewkowycz, Maarten Bosma, David Luan, Charles Sutton, and Augustus Odena.
\newblock Show your work: scratchpads for intermediate computation with language models.
\newblock \emph{arXiv preprint arXiv:2112.00114}, 2021.

\bibitem[Olsson et~al.(2022)Olsson, Elhage, Nanda, Joseph, DasSarma, Henighan, Mann, Askell, Bai, Chen, Conerly, Drain, Ganguli, Hatfield-Dodds, Hernandez, Johnston, Jones, Kernion, Lovitt, Ndousse, Amodei, Brown, Clark, Kaplan, McCandlish, and Olah]{olsson2022context}
Catherine Olsson, Nelson Elhage, Neel Nanda, Nicholas Joseph, Nova DasSarma, Tom Henighan, Ben Mann, Amanda Askell, Yuntao Bai, Anna Chen, Tom Conerly, Dawn Drain, Deep Ganguli, Zac Hatfield-Dodds, Danny Hernandez, Scott Johnston, Andy Jones, Jackson Kernion, Liane Lovitt, Kamal Ndousse, Dario Amodei, Tom Brown, Jack Clark, Jared Kaplan, Sam McCandlish, and Chris Olah.
\newblock In-context learning and induction heads.
\newblock \emph{Transformer Circuits Thread}, 2022.
\newblock https://transformer-circuits.pub/2022/in-context-learning-and-induction-heads/index.html.

\bibitem[Panigrahi et~al.(2024)Panigrahi, Liu, Malladi, Risteski, and Goel]{panigrahi2024progressive}
Abhishek Panigrahi, Bingbin Liu, Sadhika Malladi, Andrej Risteski, and Surbhi Goel.
\newblock Progressive distillation induces an implicit curriculum.
\newblock \emph{arXiv preprint arXiv:2410.05464}, 2024.

\bibitem[Papadimitriou and Sipser(1984)]{PAPADIMITRIOU1984260}
Christos~H. Papadimitriou and Michael Sipser.
\newblock Communication complexity.
\newblock \emph{Journal of Computer and System Sciences}, 28\penalty0 (2):\penalty0 260--269, 1984.
\newblock ISSN 0022-0000.
\newblock \doi{https://doi.org/10.1016/0022-0000(84)90069-2}.
\newblock URL \url{https://www.sciencedirect.com/science/article/pii/0022000084900692}.

\bibitem[Peng et~al.(2024)Peng, Narayanan, and Papadimitriou]{peng2024on}
Binghui Peng, Srini Narayanan, and Christos Papadimitriou.
\newblock On limitations of the transformer architecture.
\newblock In \emph{First Conference on Language Modeling}, 2024.
\newblock URL \url{https://openreview.net/forum?id=KidynPuLNW}.

\bibitem[Ren et~al.(2024)Ren, Wang, and Lee]{renlearning}
Yunwei Ren, Zixuan Wang, and Jason~D Lee.
\newblock Learning and transferring sparse contextual bigrams with linear transformers.
\newblock In \emph{The Thirty-eighth Annual Conference on Neural Information Processing Systems}, 2024.

\bibitem[Sanford et~al.(2023)Sanford, Hsu, and Telgarsky]{sanford2023representational}
Clayton Sanford, Daniel Hsu, and Matus Telgarsky.
\newblock Representational strengths and limitations of transformers.
\newblock In \emph{Advances in Neural Information Processing Systems 36}, 2023.

\bibitem[Sanford et~al.(2024{\natexlab{a}})Sanford, Fatemi, Hall, Tsitsulin, Kazemi, Halcrow, Perozzi, and Mirrokni]{sanford2024understanding}
Clayton Sanford, Bahare Fatemi, Ethan Hall, Anton Tsitsulin, Mehran Kazemi, Jonathan Halcrow, Bryan Perozzi, and Vahab Mirrokni.
\newblock Understanding transformer reasoning capabilities via graph algorithms.
\newblock \emph{arXiv preprint arXiv:2405.18512}, 2024{\natexlab{a}}.

\bibitem[Sanford et~al.(2024{\natexlab{b}})Sanford, Hsu, and Telgarsky]{sanford2024transformers}
Clayton Sanford, Daniel Hsu, and Matus Telgarsky.
\newblock Transformers, parallel computation, and logarithmic depth.
\newblock In \emph{Forty-First International Conference on Machine Learning}, 2024{\natexlab{b}}.

\bibitem[Sukhbaatar et~al.(2015)Sukhbaatar, Weston, and Fergus]{sukhbaatar2015end}
Sainbayar Sukhbaatar, Jason Weston, and Rob Fergus.
\newblock End-to-end memory networks.
\newblock In \emph{Advances in Neural Information Processing Systems}, volume~28, 2015.

\bibitem[Sz{\"o}r{\'e}nyi(2009)]{szorenyi2009characterizing}
Bal{\'a}zs Sz{\"o}r{\'e}nyi.
\newblock Characterizing statistical query learning: simplified notions and proofs.
\newblock In \emph{International Conference on Algorithmic Learning Theory}, pages 186--200. Springer, 2009.

\bibitem[Tian et~al.(2023)Tian, Wang, Chen, and Du]{tian2023scan}
Yuandong Tian, Yiping Wang, Beidi Chen, and Simon Du.
\newblock Scan and snap: Understanding training dynamics and token composition in 1-layer transformer.
\newblock In \emph{Advances in Neural Information Processing Systems (NeurIPS)}, 2023.

\bibitem[Uesato et~al.(2022)Uesato, Kushman, Kumar, Song, Siegel, Wang, Creswell, Irving, and Higgins]{uesato2022solving}
Jonathan Uesato, Nate Kushman, Ramana Kumar, Francis Song, Noah Siegel, Lisa Wang, Antonia Creswell, Geoffrey Irving, and Irina Higgins.
\newblock Solving math word problems with process-and outcome-based feedback.
\newblock \emph{arXiv preprint arXiv:2211.14275}, 2022.

\bibitem[Wang et~al.(2024)Wang, Wei, Hsu, and Lee]{wang2024transformers}
Zixuan Wang, Stanley Wei, Daniel Hsu, and Jason~D. Lee.
\newblock Transformers provably learn sparse token selection while fully-connected nets cannot.
\newblock In \emph{ICML}, 2024.

\bibitem[Wei et~al.(2022)Wei, Wang, Schuurmans, Bosma, Xia, Chi, Le, and Zhou]{wei2022chain}
Jason Wei, Xuezhi Wang, Dale Schuurmans, Maarten Bosma, Fei Xia, Ed~Chi, Quoc~V Le, and Denny Zhou.
\newblock Chain-of-thought prompting elicits reasoning in large language models.
\newblock In \emph{Advances in Neural Information Processing Systems}, volume~35, 2022.

\bibitem[Wen et~al.(2024)Wen, Zhang, Lin, and Zhang]{wen2024sparse}
Kaiyue Wen, Huaqing Zhang, Hongzhou Lin, and Jingzhao Zhang.
\newblock From sparse dependence to sparse attention: unveiling how chain-of-thought enhances transformer sample efficiency.
\newblock \emph{arXiv preprint arXiv:2410.05459}, 2024.

\bibitem[Weston et~al.(2014)Weston, Chopra, and Bordes]{weston2014memory}
Jason Weston, Sumit Chopra, and Antoine Bordes.
\newblock Memory networks.
\newblock \emph{arXiv preprint arXiv:1410.3916}, 2014.

\bibitem[Weston et~al.(2015)Weston, Bordes, Chopra, Rush, Van~Merri{\"e}nboer, Joulin, and Mikolov]{weston2015towards}
Jason Weston, Antoine Bordes, Sumit Chopra, Alexander~M Rush, Bart Van~Merri{\"e}nboer, Armand Joulin, and Tomas Mikolov.
\newblock Towards ai-complete question answering: A set of prerequisite toy tasks.
\newblock \emph{arXiv preprint arXiv:1502.05698}, 2015.

\bibitem[Xie et~al.(2024)Xie, Pham, Dong, Du, Liu, Lu, Liang, Le, Ma, and Yu]{xie2024doremi}
Sang~Michael Xie, Hieu Pham, Xuanyi Dong, Nan Du, Hanxiao Liu, Yifeng Lu, Percy~S Liang, Quoc~V Le, Tengyu Ma, and Adams~Wei Yu.
\newblock Doremi: Optimizing data mixtures speeds up language model pretraining.
\newblock In \emph{Advances in Neural Information Processing Systems}, volume~36, 2024.

\bibitem[Yang et~al.(2024)Yang, Gribovskaya, Kassner, Geva, and Riedel]{yang2024large}
Sohee Yang, Elena Gribovskaya, Nora Kassner, Mor Geva, and Sebastian Riedel.
\newblock Do large language models latently perform multi-hop reasoning?
\newblock \emph{arXiv preprint arXiv:2402.16837}, 2024.

\bibitem[Yao et~al.(2024)Yao, Yu, Zhao, Shafran, Griffiths, Cao, and Narasimhan]{yao2024tree}
Shunyu Yao, Dian Yu, Jeffrey Zhao, Izhak Shafran, Tom Griffiths, Yuan Cao, and Karthik Narasimhan.
\newblock Tree of thoughts: deliberate problem solving with large language models.
\newblock \emph{Advances in Neural Information Processing Systems}, 36, 2024.

\bibitem[Yehudayoff(2020)]{yehudayoff2020pointer}
Amir Yehudayoff.
\newblock Pointer chasing via triangular discrimination.
\newblock \emph{Combinatorics, Probability and Computing}, 29\penalty0 (4):\penalty0 485--494, 2020.

\bibitem[Zhang et~al.(2023)Zhang, Frei, and Bartlett]{zhang2023trained}
Ruiqi Zhang, Spencer Frei, and Peter~L Bartlett.
\newblock Trained transformers learn linear models in-context.
\newblock \emph{arXiv preprint arXiv:2306.09927}, 2023.

\end{thebibliography}

\newpage
\tableofcontents

\appendix
\newpage
\crefalias{section}{appendix} 

\allowdisplaybreaks

\section{Constructions}\label{sec:construction-proofs}
\subsection{Construction for $k$-fold Composition}
\begin{proof}[Proof of \Cref{thm:construction}]
    Choose the embedding function as follows:
    \begin{align*}
        \phi(i, j, \sigma_i(j)) = \begin{bmatrix} e_{k, i} \otimes e_{N, j} \\ e_{k, i} \otimes e_{N, \sigma_i(j)} \\ 0_{kN L } 
        \end{bmatrix}
    \end{align*}
    The first layer encodes the hidden permutation $\pi$ in the key-query matrix $W_{KQ}^{(1)}$. In particular, set 
    \begin{align*}
        W_{KQ}^{(1)} = \beta_0\sum_{i \in [k], j \in [N]} \begin{bmatrix} e_{k, i} \otimes e_{N, \pi_i(j)} \\ 0_{kN (L+1) } 
        \end{bmatrix} \begin{bmatrix} e_{k, i} \otimes e_{N, j} \\ 0_{kN (L+1) } 
        \end{bmatrix}^\top
    \end{align*}
    for large constant $\beta_0$. As such, the pre-attention weights from the position $(i, j)$ are given by
    \begin{align*}
        {X^{(0)}}^\top W_{KQ}^{(1)}X^{(0)}_{(i,j)} = \beta_0 e_{k, i} \otimes e_{N, \pi_i(j)}.
    \end{align*}
    Taking $\beta_0 \rightarrow \infty$, the attention weight from $(i, j)$ position concentrates on the $(i, \pi_i(j))$ position, so
    \begin{align*}
        \qty(X^{(0)}\mathcal{S}\qty({X^{(0)}}^\top W_{KQ}^{(1)}X^{(0)}))_{(i,j)} = \begin{bmatrix} e_{k, i} \otimes e_{N, \pi_i(j)} \\ e_{k, i} \otimes e_{N, \sigma_i(\pi_i(j))} \\ 0_{kN L } 
        \end{bmatrix}
    \end{align*}
    Finally, setting the output value matrix to be $W_{OV}^{(1)} = (e_{L + 2, 3}e_{L + 2, 2}^\top) \otimes I_{kN \times kN}$ 
    yields
    \begin{align*}
        (X^{(1)})_{(i, j)} = \begin{bmatrix} e_{k, i} \otimes e_{N, j} \\ e_{k, i} \otimes e_{N, \sigma_i(j)} \\ e_{k, i} \otimes e_{N, \sigma_i(\pi_i(j))} \\ 0_{kN(L - 1)} 
        \end{bmatrix}
    \end{align*}
    The remainder of the construction proceeds inductively. For simplicity of notation, let us define the permutation $\hop_i^r$ by 
    \begin{align*}
        \hop_i^r := \sigma_{i + r - 1} \circ \pi_{i + r-1} \circ \cdots \circ \sigma_{i + 1} \circ \pi_{i + 1} \circ \sigma_{i} \circ \pi_{i},
    \end{align*}
    where the dependence on $\sigma$ and $\pi$ is implicit.
    
    We will prove by induction that the output of the $\ell$th layer of transformer satisfies, for $i \in [k + 1 - 2^{\ell - 1}]$ and $j \in [N]$,
    \begin{align*}
        (X^{(\ell)})_{(i, j)} = \begin{bmatrix} e_{k, i} \otimes e_{N, j} \\ e_{k, i} \otimes e_{N, \sigma_i(j)} \\ e_{k, i} \otimes e_{N, \hop_i^1(j)} \\
        e_{k, i + 1} \otimes e_{N, \hop_i^2(j)} \\
        e_{k, i+ 3} \otimes e_{N, \hop_i^4(j)} \\
        \vdots\\
        e_{k, i + 2^{\ell-1} - 1} \otimes e_{N, \hop_i^{2^{\ell - 1}}(j)} \\
        0_{kN(L - \ell)} \end{bmatrix}.
    \end{align*}
    We have already shown that this is indeed satisfied for $\ell = 1$. 
    
    Assume that the inductive hypothesis holds for some $\ell$. Set the key-query matrix of layer $\ell + 1$ to be 
    \begin{align*}
        W_{KQ}^{(\ell + 1)} = \beta_{\ell + 1}\sum_{a = 1}^k\sum_{b \in [N]} \begin{bmatrix} e_{k, a + 2^{\ell-1}} \otimes e_{N, b} \\ 0_{kN (L + 1) } 
        \end{bmatrix} \begin{bmatrix} 0_{kN(\ell + 1)} \\ e_{k, a+ 2^{\ell-1}-1} \otimes e_{N, b} \\ 0_{kN (L - \ell) } \end{bmatrix}^\top.
    \end{align*}
    Here when $a+2^{\ell-1}>k$, we denote $e_{k,a+2^{\ell-1}}:=e_{k,a+2^{\ell-1}-k}$, i.e. the indices of the permutation or embedding are taken modulo $k$.
    Consider some position $(i, j)\in [k]\times [N]$. The pre-attention weights from position $(i, j)$ are given by
    \begin{align*}
        {X^{(\ell)}}^\top W_{KQ}^{(\ell + 1)}X^{(\ell)}_{(i,j)} = \beta_{\ell + 1} \cdot e_{k, i + 2^{\ell-1}} \otimes e_{N, \hop_{i}^{2^\ell - 1}(j)}.
    \end{align*}
    Taking $\beta_{\ell + 1} \rightarrow \infty$, the attention weight from the $(i, j)$ position will concentrate on the $(i + 2^{\ell - 1}, \hop_i^{2^{\ell - 1}}(j))$ position. Setting the value matrix to be $W_{OV}^{(\ell + 1)} = e_{L+2, \ell + 3} e_{L+2, \ell + 2}^\top \otimes I_{kN \times kN}$ ensures that
    \begin{align*}
        \qty(W^{(\ell + 1)}_{OV}X^{(\ell)}\S({X^{(\ell)}}^\top W^{(\ell + 1)}_{KQ} X^{(\ell)}))_{i,j} &= \begin{bmatrix}
            0_{kN(\ell + 2)} \\
            e_{k, i + 2^{\ell-1} + 2^{\ell - 1} - 1} \otimes e_{N, \hop_{i + 2^{\ell-1}}^{2^{\ell - 1}}\qty(\hop_i^{2^{\ell - 1}}(j))}\\
            0_{kN(L - \ell - 1)} \\
        \end{bmatrix} \\
        &= \begin{bmatrix}
            0_{kN(\ell + 2)} \\
            e_{k, i + 2^{\ell} - 1} \otimes e_{N, \hop_i^{2^{\ell}}(j)}\\
            0_{kN(L - \ell - 1)} \\
        \end{bmatrix}
    \end{align*}
    Therefore $X^{(\ell + 1)}_{(i, j)}$ satisfies
    \begin{align*}
                (X^{(\ell + 1)})_{(i, j)} = \begin{bmatrix} e_{k, i} \otimes e_{N, j} \\ e_{k, i} \otimes e_{N, \sigma_i(j)} \\ e_{k, i} \otimes e_{N, \hop_i^1(j)} \\
        e_{k, i + 1} \otimes e_{N, \hop_i^2(j)} \\
        e_{k, i+ 3} \otimes e_{N, \hop_i^4(j)} \\
        \vdots\\
        e_{k, i + 2^{\ell-1} - 1} \otimes e_{N, \hop_i^{2^{\ell}}(j)} \\
        0_{kN(L - \ell - 1)} \end{bmatrix},
    \end{align*}
    as desired. Therefore by induction, the claim holds for $\ell = L-1$ and thus

        \begin{align*}
                (X^{(L)})_{(1, j)} = \begin{bmatrix} e_{k, 1} \otimes e_{N, j} \\ e_{k, 1} \otimes e_{N, \sigma_1(j)} \\ e_{k, 1} \otimes e_{N, \hop_1^1(j)} \\
        e_{k, 2} \otimes e_{N, \hop_1^2(j)} \\
        e_{k, 4} \otimes e_{N, \hop_1^4(j)} \\
        \vdots\\
        e_{k, k} \otimes e_{N, \hop_1^{k}(j)} \end{bmatrix}
    \end{align*}
    To conclude, set the readout layer $\Psi$ to be
    \begin{align*}
        \Psi^\top = \begin{bmatrix}
            0_{N \times kN(L+1)}, & e_{k, k} \otimes I_{N \times N}.
        \end{bmatrix}
    \end{align*}
    The output of the transformer then satisfies
    \begin{align*}
        \qty(\Psi^\top \mathrm{TF}_\theta(X))_{(1, j)} = e_{N, \hop_1^{k}(j)} = e_{N, f_\pi(\sigma, j)},
    \end{align*}
    as desired.
\end{proof}

\subsection{The $m$-sparse $k$-fold Composition}\label{sec:m ll k task}
The $k$-fold composition task requires composing $k$ hidden permutations with $k$ input permutations. In practice, multi-step reasoning tasks may contain more contextual knowledge than parametric knowledge. We thus introduce a variant of the $k$-fold composition task, called the \emph{$m$-sparse $k$-fold composition}, which requires composing $m$ hidden permutations with $k$ input permutations for $m \le k$.

Recall that the $k$-fold composition task is defined as $f_\pi : \mathcal{X} \rightarrow [N]$:
\begin{align*}
    f_\pi(\sigma, x) := (\sigma_k \circ \pi_k \circ \sigma_{k-1} \circ \pi_{k-1} \circ \cdots \circ \sigma_1 \circ \pi_1)(x).
\end{align*}
Our $m$-sparse $k$-fold composition target function class is defined as $$\mathcal{F}_m := \{f_\pi : \pi \in (S_N)^k, \sum_{i=1}^k\mathbf{1}\{\pi_i=\text{Id}\}\ge k-m\},$$ i.e those $f_\pi$ where at most $m$ of the $\pi$ are not the identity. The goal is to learn $\mathcal{F}_m$ with respect to the uniform distribution over $\mathcal{X}$. 
We similarly define the cyclic task $\mathcal{F}_m^{\cyc} = \{f_\pi^{\cyc} : \pi \in (S_N)^k, \sum_{i=1}^k\mathbf{1}\{\pi_i=\text{Id}\}\ge k-m\}$.

The following construction shows that the $m$-sparse $k$-fold composition is expressible by a $\Theta(\log k)$-depth transformer with embedding dimension $d = \tilde \Theta(mN)$.

\begin{theorem}\label{thm:construction_for_m_sparse}
    Assume that $k$ is a power of two. There exists an embedding function $\phi$ with $d = (m+2)N(3 + \log_2k)$ such that, for any $\pi \in (S_N)^k$ containing at most $m$ non-identity hidden permutations, there exists an $L = \log_2 k + 1$ layer transformer which can exactly express the permutation composition function, i.e
    \begin{align*}
        (\Psi^\top\mathrm{TF}_\theta(X(\sigma)))_{(1,j)} = e_{N, f_\pi(\sigma, x)} \quad \text{for all}~(\sigma, x) \in \mathcal{X}.
    \end{align*}
\end{theorem}
\begin{proof}[Proof of \Cref{thm:construction_for_m_sparse}]
    Suppose the index set of non-identity hidden permutation $\pi_i$ is $\mathcal{M}=\{i_1,...,i_m\}$. Choose the embedding function as follows:
    \begin{align*}
        \phi(i, j, \sigma_i(j)) = \begin{bmatrix} p_{i} \otimes e_{N, j} \\ p_{i} \otimes e_{N, \sigma_i(j)} \\ 0_{kN L } 
        \end{bmatrix}
    \end{align*}
    where the positional encoding for the block position combines sinusoidal and one-hot positional encoding. We further denote $\hat{p}_i = (\cos\qty(\frac{2\pi i}{k}),
        \sin \qty(\frac{2\pi i}{k}))^\top\in \R^2$ as the sinusoidal embedding. $$p_{i} := \begin{bmatrix}
        0_{m}\\
        \hat{p}_i
    \end{bmatrix}=\begin{bmatrix}
        0_{m}\\
        \cos\qty(\frac{2\pi i}{k})\\
        \sin \qty(\frac{2\pi i}{k})
    \end{bmatrix}\in \R^{m+2}\text{ if }i\not\in \mathcal{M},\quad  p_{i_j} :=\begin{bmatrix}
        0_{m}\\
        \hat{p}_{i_y}
    \end{bmatrix}= \begin{bmatrix}
        e_{m,j}\\
        \cos\qty(\frac{2\pi i_j}{k})\\
        \sin \qty(\frac{2\pi i_j}{k})
    \end{bmatrix}\in \R^{m+2}\text{ if }i_j\in \mathcal{M}$$
    
    The first layer encodes the hidden permutation $\pi$ in the key-query matrix $W_{KQ}^{(1)}$. In particular, set
    \begin{align*}
        W_{KQ}^{(1)} = \begin{bmatrix}
            A_1&0\\
            0&0
        \end{bmatrix}, A_1\in \R^{(m+2)N\times (m+2)N}
    \end{align*}
    where the top-left block is
    \begin{align*}
        A_1 = \beta_1 \begin{bmatrix}
            0_{m\times m}&0\\
            0&I_{2}
        \end{bmatrix}\otimes I_{N}+\beta_2 \sum_{i_y\in\mathcal{M},j\in[N]}\begin{bmatrix}
            e_{m,y}\\
            0
        \end{bmatrix} \begin{bmatrix}
            e_{m,y}\\
            0
        \end{bmatrix}^\top\otimes e_{N,\pi_{i_y}(j)}e_{N,j}^\top
    \end{align*}
    for large constant $\beta_2\gg \beta_1\gg 1$. As such, the pre-attention weights from the position $(i, j)$ are given by
    \begin{align*}
        {X^{(0)}}^\top W_{KQ}^{(1)}X^{(0)}_{(i,j)} = \beta_1 \sum_{i'=1}^k \cos\qty(\frac{2\pi(i'-i)}{k})e_{k,i} \otimes e_{N, j}\text{ if $i\notin \mathcal{M}$ }.
    \end{align*}
    \begin{align*}
        {X^{(0)}}^\top W_{KQ}^{(1)}X^{(0)}_{(i_y,j)} = \beta_1 \sum_{i'=1}^k \cos\qty(\frac{2\pi(i'-i)}{k})e_{k,i} \otimes e_{N, j} + \beta_2 e_{k,i_y}\otimes e_{N,\pi_{i_y}(j)}\text{ if $i_y\in \mathcal{M}$ }.
    \end{align*}
    Taking $\beta_2/\beta_1, \beta_1 \rightarrow \infty$, the attention weight from $(i, j)$ position concentrates on the $(i, \pi_i(j))$ position (either $i\in \mathcal{M}$ or not), so
    \begin{align*}
        \qty(X^{(0)}\mathcal{S}\qty({X^{(0)}}^\top W_{KQ}^{(1)}X^{(0)}))_{(i,j)} = \begin{bmatrix} p_{i} \otimes e_{N, \pi_i(j)} \\ p_{i} \otimes e_{N, \sigma_i(\pi_i(j))} \\ 0_{(m+2)N L } 
        \end{bmatrix}
    \end{align*}
    Finally, setting the output value matrix to be $W_{OV}^{(1)} = (e_{L + 2, 3}e_{L + 2, 2}^\top) \otimes I_{(m+2)N \times (m+2)N}$ 
    yields
    \begin{align*}
        (X^{(1)})_{(i, j)} = \begin{bmatrix} p_{i} \otimes e_{N, j} \\ p_{i} \otimes e_{N, \sigma_i(j)} \\ p_{i} \otimes e_{N, \sigma_i(\pi_i(j))} \\ 0_{(m+2)N(L - 1)} 
        \end{bmatrix}
    \end{align*}
    The remainder of the construction proceeds inductively. For simplicity of notation, let us define the permutation $\hop_i^r$ by 
    \begin{align*}
        \hop_i^r := \sigma_{i + r - 1} \circ \pi_{i + r-1} \circ \cdots \circ \sigma_{i + 1} \circ \pi_{i + 1} \circ \sigma_{i} \circ \pi_{i},
    \end{align*}
    where the dependence on $\sigma$ and $\pi$ is implicit.
    
    We will prove by induction that the output of the $\ell$th layer of transformer satisfies, for $i \in [k + 1 - 2^{\ell - 1}]$ and $j \in [N]$,
    \begin{align*}
        (X^{(\ell)})_{(i, j)} = \begin{bmatrix} p_{i} \otimes e_{N, j} \\ p_{i} \otimes e_{N, \sigma_i(j)} \\ p_{i} \otimes e_{N, \hop_i^1(j)} \\
        p_{i + 1} \otimes e_{N, \hop_i^2(j)} \\
        p_{i+ 3} \otimes e_{N, \hop_i^4(j)} \\
        \vdots\\
        p_{i + 2^{\ell-1} - 1} \otimes e_{N, \hop_i^{2^{\ell - 1}}(j)} \\
        0_{(m+2)N(L - \ell)} \end{bmatrix}.
    \end{align*}
    We have already shown that this is indeed satisfied for $\ell = 1$. 
    
    Assume that the inductive hypothesis holds for some $\ell$. Set the key-query matrix of layer $\ell + 1$ to
    \begin{align*}
        W_{KQ}^{(\ell + 1)} = \beta_{\ell + 1}\begin{bmatrix}
            0_{(m+2)N \times (m+2)N(\ell+1)}&A^{(\ell+1)}&0_{(m+2)N \times (m+2)N(L-\ell)}\\
            0_{(m+2)N(L+1) \times (m+2)N(\ell+1)}&0_{(m+2)N \times (m+2)N}&0_{(m+2)N(L+1)\times (m+2)N(L-\ell)}
        \end{bmatrix}
    \end{align*}
    where the block $A^{(\ell+1)}$ should be 
    $$A^{(\ell+1)} = \begin{bmatrix}
        0_{m\times m}&0_m&0_m\\
        0_m^\top&\cos\frac{2\pi}{k}&-\sin\frac{2\pi}{k}\\
        0_m^\top&\sin\frac{2\pi}{k}&\cos\frac{2\pi}{k}
    \end{bmatrix}\otimes I_{N}$$
    Consider some position $(i, j)\in [k]\times [N]$. The pre-attention weights from position $(i, j)$ are given by
    \begin{align*}
        {X^{(\ell)}}^\top W_{KQ}^{(\ell + 1)}X^{(\ell)}_{(i,j)} = \beta_{\ell + 1} \cdot\sum_{i'=1}^k \cos\qty(\frac{2\pi(i'-i-2^{\ell-1})}{k}) e_{k,i' + 2^{\ell-1}} \otimes e_{N, \hop_{i}^{2^\ell - 1}(j)}.
    \end{align*}
    Taking $\beta_{\ell + 1} \rightarrow \infty$, the attention weight from the $(i, j)$ position will concentrate on the $(i + 2^{\ell - 1}, \hop_i^{2^{\ell - 1}}(j))$ position. Setting the value matrix to be $W_{OV}^{(\ell + 1)} = e_{L+2, \ell + 3} e_{L+2, \ell + 2}^\top \otimes I_{(m+2)N \times (m+2)N}$ ensures that
    \begin{align*}
        \qty(W^{(\ell + 1)}_{OV}X^{(\ell)}\S({X^{(\ell)}}^\top W^{(\ell + 1)}_{KQ} X^{(\ell)}))_{i,j} &= \begin{bmatrix}
            0_{(m+2)N(\ell + 2)} \\
            p_{i + 2^{\ell-1} + 2^{\ell - 1} - 1} \otimes e_{N, \hop_{i + 2^{\ell-1}}^{2^{\ell - 1}}\qty(\hop_i^{2^{\ell - 1}}(j))}\\
            0_{(m+2)N(L - \ell - 1)} \\
        \end{bmatrix} \\
        &= \begin{bmatrix}
            0_{(m+2)N(\ell + 2)} \\
            p_{i + 2^{\ell} - 1} \otimes e_{N, \hop_i^{2^{\ell}}(j)}\\
            0_{(m+2)N(L - \ell - 1)} \\
        \end{bmatrix}
    \end{align*}
    Therefore $X^{(\ell + 1)}_{(i, j)}$ satisfies
    \begin{align*}
                (X^{(\ell + 1)})_{(i, j)} = \begin{bmatrix} p_{i} \otimes e_{N, j} \\ p_{i} \otimes e_{N, \sigma_i(j)} \\ p_{i} \otimes e_{N, \hop_i^1(j)} \\
        p_{i + 1} \otimes e_{N, \hop_i^2(j)} \\
        p_{i+ 3} \otimes e_{N, \hop_i^4(j)} \\
        \vdots\\
        p_{i + 2^{\ell-1} - 1} \otimes e_{N, \hop_i^{2^{\ell}}(j)} \\
        0_{kN(L - \ell - 1)} \end{bmatrix},
    \end{align*}
    as desired. Therefore by induction, the claim holds for $\ell = L-1$ and thus

        \begin{align*}
                (X^{(L)})_{(1, j)} = \begin{bmatrix} p_{1} \otimes e_{N, j} \\ p_{1} \otimes e_{N, \sigma_1(j)} \\ p_{1} \otimes e_{N, \hop_1^1(j)} \\
        p_{2} \otimes e_{N, \hop_1^2(j)} \\
        p_{4} \otimes e_{N, \hop_1^4(j)} \\
        \vdots\\
        p_{k} \otimes e_{N, \hop_1^{k}(j)} \end{bmatrix}
    \end{align*}
    To conclude, set the readout layer $\Psi$ to be
    \begin{align*}
        \Psi^\top = \begin{bmatrix}
            0_{N \times (m+2)N(L+1)}, & p_{k} \otimes I_{N \times N}.
        \end{bmatrix}
    \end{align*}
    The output of the transformer then satisfies
    \begin{align*}
        \qty(\Psi^\top \mathrm{TF}_\theta(X))_{(1, j)} = e_{N, \hop_1^{k}(j)} = e_{N, f_\pi(\sigma, j)},
    \end{align*}
    as desired.
\end{proof}

\section{SQ Lower Bound}\label{sec:lb-proofs}

\paragraph{Connection to Gradient Descent.} Consider a neural network $F_\theta : \mathcal{X} \rightarrow \mathbb{R}^{N}$ which outputs the logits for predicting $f_\pi(\sigma, x)$. The standard cross entropy loss is given by
\begin{align*}
    L(F_\theta) = \mathbb{E}_{\sigma, x} \qty[ \sum_{m \in [N]} \mathbf{1}(f_\pi(\sigma, x) = m) \log \S(F_\theta(\sigma, x))_m],
\end{align*}
and thus the gradient descent update on the population loss is
\begin{align*}
 \theta' \leftarrow \theta - \eta \mathbb{E}_{\sigma, x}\qty[\sum_{m \in [N]}\mathbf{1}(f_\pi(\sigma, x) = m)\nabla_\theta F_\theta(\sigma, x)_m] + \eta\mathbb{E}_{\sigma, x}\qty[\frac{\sum_{m \in [N]}\exp\qty(F_\theta(\sigma, x)_{m})\nabla_\theta F_\theta(\sigma, x)_{m}}{\sum_{m \in [N]}\exp\qty(F_\theta(\sigma, x)_{m})}].
\end{align*}
The first term is exactly a (vector-valued) SQ query with $g(\sigma, x, y) = \nabla_\theta F_\theta(\sigma, x)_{y}$, while the second normalization term is independent of the unknown target function $f_\pi$. We remark that this connection is still heuristic, due to the mismatch between i.i.d noise for gradient descent and adversarial noise in the SQ framework.

\paragraph{On the normalization.} The zero mean assumption is indeed without loss of generality, since for any query $g$, the mean-centered query $\bar g(\sigma, x, y) = g(\sigma, x, y) - \mathbb{E}_{\sigma'}[g(\sigma', x, y)]$ satisfies
\begin{align*}
    \mathbb{E}_{\sigma, x}\qty[\bar g(\sigma, x, f_\pi(\sigma, x))] &= \mathbb{E}_{\sigma, x}\qty[g(\sigma, x, f_\pi(\sigma, x))] - \mathbb{E}_{\sigma, x, \sigma'}\qty[g(\sigma', x, f_\pi(\sigma, x))]\\
    &= \mathbb{E}_{\sigma, x}\qty[g(\sigma, x, f_\pi(\sigma, x))] - \frac{1}{N^2}\sum_{x, y \in [N]}\mathbb{E}_{\sigma}\qty[g(\sigma, x, y)],
\end{align*}
The last term is independent of the hidden permutation $\pi$, and thus the queries $g$ and $\bar g$ reveal the same information about $\pi$.

\paragraph{Notation.} For $\pi \in (S_N)^k$, we will let $f_\pi(\sigma)$ denote the permutation $f_\pi(\sigma) = \sigma_k \circ \pi_k \circ \cdots \circ \sigma_1 \circ \pi_1$. We will additionally overload notation, to let $f_\pi(\sigma)$ refer to the corresponding $N \times N$ permutation matrix. Finally, for two matrices $A, B \in \mathbb{R}^{N \times N}$, we let $\langle A, B \rangle = \Tr(A^\top B)$ denote the standard matrix inner product.

Our first goal is to construct a large subset of $\mathcal{F}$, which are nearly orthogonal under the inner product
\begin{align*}
    \langle f_\pi, f_\rho \rangle := \mathbb{P}_{\sigma, x}\qty(f_\pi(\sigma, x) = f_\rho(\sigma, x)) - \frac{1}{N} = \frac{1}{N}\qty(\mathbb{E}_\sigma\qty[\langle f_\pi (\sigma), f_\rho(\sigma) \rangle] - 1).
\end{align*}
We first derive an explicit formula for the inner product between two permutation composition functions.

\begin{lemma}[Inner product between functions]\label{lem:near-orthogonal-permutation}
For two permutation composition functions $f_\pi, f_\rho \in \mathcal{F}$,
\begin{align*}
    \langle f_\pi, f_\rho \rangle =  \frac{1}{N(N-1)^{k-1}}\prod_{i=1}^k\qty(\langle \pi_k, \rho_k \rangle - 1),
\end{align*}
where $\langle \pi_k, \rho_k \rangle = \sum_{i \in [N]} \mathbf{1}(\pi_k(i) = \rho_k(i))$.
\end{lemma}
\begin{proof}
    Let us overload notation, so that for a permutation $\pi_i$, we also let $\pi_i$ be the $N \times N$ permutation matrix. We see that
    \begin{align*}
        \langle f_\pi, f_\rho \rangle = \frac{1}{N}\qty(\mathbb{E}_\sigma\qty[\Tr(\sigma_k\pi_k\cdots\sigma_1\pi_1\rho_1^\top \sigma_1^\top \cdots \rho_k^\top \sigma_k^\top)] - 1)
    \end{align*}

    Let us first compute this expectation with respect to $\sigma_1$. 
    
    Define $A:= \rho_2^\top\sigma_2^\top \cdots \rho_k^\top \sigma_k^\top \sigma_k \pi_k \cdots \sigma_2\pi_2, B = \pi_1 \rho_1^\top$. We have that
    \begin{align*}
        \Tr(A\sigma_1 B \sigma_1^\top) = \sum_{a, b, c, d \in [N]} A_{ab}(\sigma_1)_{bc}B_{cd}(\sigma_1)_{ad}.
    \end{align*}
    Since $\sigma_1$ sampled uniformly at random from $S_N$, we have that
    \begin{align*}
        \mathbb{E}_{\sigma_1}\qty[(\sigma_1)_{bc}(\sigma_1)_{ad}] = \begin{cases}
            \frac{1}{N} & a = b, c = d \\
            0 & a = b, c \neq d \\
            0 & a \neq b, c = d \\
            \frac{1}{N(N-1)} & a \neq b, c \neq d
        \end{cases},
    \end{align*}
    where the last equality is because
    \begin{align*}
        \mathbb{E}_{\sigma_1}\qty[(\sigma_1)_{bc}(\sigma_1)_{ad}] &= \mathbb{P}(\sigma_1(a) = d, \sigma(b) = c)\\ &= \mathbb{P}(\sigma_1(a) = d)\cdot \mathbb{P}(\sigma_1(b) = c \mid \sigma_1(a) = d)\\
        &= \frac{1}{N(N-1)}.
    \end{align*}
    Altogether,
    \begin{align*}&
        \mathbb{E}_{\sigma_1}\qty[\Tr(A\sigma_1 B \sigma_1^\top)] = \frac{1}{N}\sum_{a, c}A_{aa}B_{cc} + \frac{1}{N(N-1)}\sum_{a \neq b, c \neq d} A_{ab}B_{cd}\\
        &= \frac{1}{N}\Tr(A)\Tr(B) + \frac{1}{N(N-1)}\qty(\sum_{a,b}A_{ab} - \Tr(A))(\sum_{c,d}B_{c,d} - \Tr(B))\\
        &= \frac{1}{N}\Tr(A)\Tr(B) + \frac{1}{N(N-1)}\qty(N - \Tr(A))(N - \Tr(B))\\
        &= \frac{(\Tr(A) - 1)(\Tr(B) - 1)}{N-1} + 1.
    \end{align*}
    where the second-to-last equality uses the property that $A, B$ are both permutation matrices, and thus $\sum_{a,b}A_{ab} = \sum_{c,d}B_{c,d} = N$. Plugging everything back in, we see that
    \begin{align*}
        &\mathbb{E}_{\sigma_1, \dots, \sigma_k}\qty[\Tr\qty(\sigma_k\pi_k\cdots\sigma_1\pi_1\rho_1^\top \sigma_1^\top \cdots \rho_k^\top \sigma_k^\top)] - 1\\
        &= \mathbb{E}_{\sigma_1, \dots, \sigma_k}\qty[\Tr(A\sigma_1 B \sigma_1^\top)] - 1\\
        &= \mathbb{E}_{\sigma_2, \dots, \sigma_k}\qty[\frac{(\Tr(A) - 1)(\Tr(B) - 1)}{N-1}]\\
        &= \frac{\langle \pi_1, \rho_1 \rangle - 1}{N-1}\cdot\qty(\mathbb{E}_{\sigma_2, \dots, \sigma_k}\qty[\Tr(A)] - 1)\\
        &= \frac{\langle \pi_1, \rho_1 \rangle - 1}{N-1}\cdot\qty(\mathbb{E}_{\sigma_2, \dots, \sigma_k}\qty[\Tr(\sigma_k\pi_k\cdots\sigma_2\pi_2\rho_2^\top \sigma_2^\top \cdots \rho_k^\top \sigma_k^\top)] - 1).
    \end{align*}
    Computing this quantity recursively, we get
    \begin{align*}
        &\mathbb{E}_{\sigma_1, \dots, \sigma_k}\qty[\Tr\qty(\sigma_k\pi_k\cdots\sigma_1\pi_1\rho_1^\top \sigma_1^\top \cdots \rho_k^\top \sigma_k^\top)] - 1\\
        &= \frac{\langle \pi_1, \rho_1 \rangle - 1}{N-1} \cdots \frac{\langle \pi_k, \rho_k \rangle - 1}{N-1} \cdot \qty(\Tr(I) - 1)\\
        &= \frac{1}{(N-1)^{k-1}}\prod_{i=1}^k\qty(\langle \pi_k, \rho_k \rangle - 1),
    \end{align*}
    and thus
    \begin{align*}
        \langle f_\pi, f_\rho \rangle = \frac{1}{N}\qty(\mathbb{E}_{\sigma_1, \dots, \sigma_k}\qty[\Tr\qty(\sigma_k\pi_k\cdots\sigma_1\pi_1\rho_1^\top \sigma_1^\top \cdots \rho_k^\top \sigma_k^\top)] - 1)\\ = \frac{1}{N(N-1)^{k-1}}\prod_{i=1}^k\qty(\langle \pi_k, \rho_k \rangle - 1).
    \end{align*}
\end{proof}

\begin{remark}For two cyclic permutation functions $f_\pi^{\cyc}, f_\rho^{\cyc}$, we define the inner product by
\begin{align*}
    \langle f_\pi^{\cyc}, f_\rho^{\cyc} \rangle := \mathbb{P}_{\sigma, i, x}\qty(f_\pi(\sigma, i, x) = f_\rho(\sigma, i, x)) - \frac{1}{N} = \frac{1}{k}\sum_{i=1}^k\qty(\frac{1}{N}\E_\sigma[\langle f^{(i)}_\pi(\sigma), f^{(i)}_\rho(\sigma)] - \frac{1}{N}),
\end{align*}
where we have overloaded notation to let $f^{(i)}(\sigma)$ denote the permutation $\sigma_{i + k - 1}\circ \pi_{i + k - 1} \circ \cdots \circ \sigma_i \circ \pi_i$, so that $f^{(i)}(\sigma)(x) = f(\sigma, i, x)$. By Lemma \ref{lem:near-orthogonal-permutation},
\begin{align*}
    \frac{1}{N}\E_\sigma[\langle f^{(i)}_\pi(\sigma), f^{(i)}_\rho(\sigma)\rangle] - \frac{1}{N} = \frac{1}{N(N-1)^{k-1}}\prod_{i=1}^k\qty(\langle \pi_k, \rho_k\rangle - 1)
\end{align*}
independent of the index $i$, and therefore $\langle f_\pi, f_\rho \rangle = \langle f_\pi^{\cyc}, f_\rho^{\cyc}\rangle$.
\end{remark}

We will next construct a subset of $\mathcal{F}$ of nearly orthogonal functions. By the above remark, this also corresponds to a subset of $\mathcal{F}^{\cyc}$ of nearly orthogonal functions.
\begin{lemma}[Nearly orthogonal subset]\label{lem:near-orth-set}
    Pick any $r \in [N]$. There exists a subset $\mathcal{F}_r \subset \mathcal{F}$ such that $\abs{\mathcal{F}_r} \ge (r!/4)^{k/2}$, and for any $f, f' \in \mathcal{F}_r$ with $f \neq f'$, $\abs{\langle f, f' \rangle} \le (2r/N)^{k/2}$.
\end{lemma}
\begin{proof}
    Let $\Pi_r$ be a maximal packing of $S_N$, such that for any $\pi, \pi' \in \Pi_r$, $\langle \pi, \pi' \rangle < r$. 
    
    For a fixed permutation $\pi$, let us first count the number of permutations satisfying $\langle \pi, \pi' \rangle = i$. There are $\binom{N}{i}$ choices for which $i$ elements are the same. Then, there are $D_{S-i}$ ways to assign the rest of the permutation, where $D_n$ counts the number of derangements on $n$ elements. Therefore the number of such $\pi'$ is $\binom{N}{i}D_{N-i} \le  \frac{N!}{i!(N-i)!}\cdot \qty(\frac{(N-i)!}{e} + 1) = \frac{N!}{e\cdot i!} + \binom{N}{i}$. In total, the number of permutations that agree with $\pi$ on more than $r$ elements is at most
    \begin{align*}
        \sum_{i = r + 1}^N\qty(\frac{N!}{e\cdot i!} + \binom{N}{i}) = \frac{N!}{e \cdot r!}\sum_{i=r + 1}\qty(\frac{r!}{i!} + \frac{e r!}{i!(N-i)!}) \le \frac{N!}{e\cdot r!}\sum_{i=r+1}^N\frac{1}{(i-r)!} \le \frac{N!}{r!}.
    \end{align*}
    Therefore if we construct $\Pi_r$ greedily, we get that $M := \abs{\Pi_r} \ge r!$.

    We next construct a maximal packing $\mathcal{I}$ of $[M]^k$, such that for any $\alpha, \alpha' \in \mathcal{I}$, $\alpha$ and $\alpha'$ agree on at most $k/2$ coordinates. For any fixed $\alpha \in \mathcal{I}$, the number of tuples $\alpha' \in \mathcal{I}$ which agree with $\alpha$ on at least $k/2$ coordinates is
    \begin{align*}
        \sum_{j \le k/2} \binom{k}{j}(M-1)^j \le (M-1)^{k/2}2^k.
    \end{align*}
    Therefore if we construct $\mathcal{I}$ greedily, we get $\abs{\mathcal{I}} \ge M^{k/2}2^{-k}$.

    We now construct the subset $\mathcal{F}_r$. Let the elements of $\Pi_r$ be $\pi^{(1)}, \pi^{(2)}, \dots, \pi^{(M)}$ in some order. Define $\mathcal{F}_r := \{f_{(\pi^{(\alpha_1)}, \dots, \pi^{(\alpha_k)})} :  \alpha \in \mathcal{I}\}$. Since any two elements of $\mathcal{I}$ differ in at least $k/2$ coordinates, and any two permutations in $\Pi_r$ have at most $r$ common values, for any $f, f' \in \mathcal{F}_r$, by \Cref{lem:near-orthogonal-permutation} we have 
    \begin{align*}
        \abs{\langle f, f' \rangle} \le \frac{1}{N(N-1)^{k-1}}(N-1)^{k/2}(r-1)^{k/2} \le \qty(\frac{r-1}{N-1})^{k/2} \le \qty(\frac{2r}{N})^{k/2}.
    \end{align*}
    Finally, we see that $\abs{\mathcal{F}_r} = \abs{\mathcal{I}} \ge (M/4)^{k/2} = (r!/4)^{k/2}$.
\end{proof}

An immediate corollary of \Cref{lem:near-orth-set} is a bound on the packing number of $\mathcal{F}$.

\begin{corollary}[Packing number]\label{cor:packing}
    There exists a subset $\mathcal{N} \subset \mathcal{F}$ of size $\log \abs{\mathcal{N}} \gtrsim kN\log N $ such that, for any $f \neq f' \in \mathcal{N},$ $\norm{f - f'}^2 \ge \frac12$.
\end{corollary}
\begin{proof}
    Since $\norm{f}^2 = 1 - 1/N$ for all $f \in \mathcal{F}$, $\abs{\langle f, f' \rangle} \le \varepsilon$ implies that $\norm{f - f'}^2 \ge 2 - 2\varepsilon - 2/N \ge 1 - 2\varepsilon$. Select $\mathcal{N} = \mathcal{F}_r$ for $r = N/8$. Then $\abs{\langle f, f' \rangle} \le (1/4)^{k/2} \le \frac14$, so $\norm{f - f'}^2 \ge 1/2$. Finally, we have that $\log \mathcal{N} \ge k/2\cdot\log((N/8)!/4) \gtrsim kN\log N$, as desired.
\end{proof}

We can now use standard SQ arguments to conclude the proof of \Cref{thm:SQ}. The following follows the proof of Lemma 2 in \citet{damian2022neural} and Theorem 2 in \citet{szorenyi2009characterizing}.

\begin{proof}[Proof of \Cref{thm:SQ}]
We will begin by showing the lower bound for the regular permutation composition task $\mathcal{F}$, and later show how the proof can be adapted to yield the same lower bound for the cyclic task $\mathcal{F}^{\cyc}$.

We will show that there exist two functions $f, f' \in \mathcal{F}_r$ such that $\abs{\mathbb{E}_{\sigma,x}[g_l(\sigma, x, f(\sigma, x))]} \le \tau$, $\abs{\mathbb{E}_{\sigma,x}[g_i(\sigma, x, f'(\sigma, x))]} \le \tau$ for each query $g_l$ made by the learner. As such, an adversary can respond with 0 to each query, and the learner will be unable to distinguish between $f$ and $f'$. The learner must then incur a loss of at least
\begin{align*}
    \max\qty(\mathbb{P}_{\sigma, x}(\hat f(\sigma, x) \neq f(\sigma, x)), \mathbb{P}_{\sigma, x}(\hat f(\sigma, x) \neq f'(\sigma, x))) &\ge \frac12 \mathbb{P}_{\sigma, x}(f'(\sigma, x) \neq f(\sigma, x))\\
    &= \frac12\qty(1 - \frac1N - \langle f, f' \rangle)\\
    &= \Omega(1)
\end{align*}
on either $f$ or $f'$.

For the $i$th query $g_l$, define
\begin{align*}
    \mathcal{A}^+_l &:= \{f \in \mathcal{F}_r :  \mathbb{E}_{\sigma,x}[g_i(\sigma, x, f(\sigma, x))] \ge \tau\}\\
    \mathcal{A}^-_l &:= \{f \in \mathcal{F}_r :  \mathbb{E}_{\sigma,x}[g_i(\sigma, x, f(\sigma, x))] \le -\tau\}.
\end{align*}
Let us overload notation to let $g_l : (S_N)^k \rightarrow \mathbb{R}^{N \times N}$ be the mapping with $g_l(\sigma)_{x, y} := g_l(\sigma, x, y)$. Each query can be written as
\begin{align*}
     \mathbb{E}_{\sigma,x}[g_l(\sigma, x, f(\sigma, x))] = \frac{1}{N}\mathbb{E}_\sigma\qty[\langle g_l(\sigma), f(\sigma) \rangle],
\end{align*}
and thus we have that
\begin{align*}
    \abs{\mathcal{A}^+_l}^2 \tau^2 &\le \qty(\sum_{f \in \mathcal{A}^+_l} \mathbb{E}_{\sigma,x}[g_l(\sigma, x, f(\sigma, x))])\\
    &\le \qty(\frac{1}{N}\sum_{f \in \mathcal{A}^+_l} \mathbb{E}_\sigma\qty[\langle g_l(\sigma), f(\sigma) \rangle])^2\\
    &= \frac{1}{N^2}\mathbb{E}_\sigma\qty[\big\langle g_l(\sigma), \sum_{f \in \mathcal{A}^+_l}f(\sigma) \big\rangle]^2\\
    &= \frac{1}{N^2}\mathbb{E}_\sigma\qty[\big\langle g_l(\sigma), \sum_{f \in \mathcal{A}^+_l}\qty(f(\sigma) - \frac{1}{N}1_N1_N^\top \big\rangle)]^2,
\end{align*}
where the last line uses the fact that $g_l$ are mean zero. By Cauchy-Schwarz, we have
\begin{align*}
    \abs{\mathcal{A}^+_l}^2 \tau^2 &\le \frac{1}{N^2}\mathbb{E}_\sigma\norm{g_l(\sigma)}^2_F \cdot \mathbb{E}_\sigma\norm{\sum_{f \in \mathcal{A}^+_l}\qty(f(\sigma) - \frac{1}{N}1_N1_N^\top \big\rangle)}^2_F\\
    &\le \frac{1}{N}\sum_{f, f' \in \mathcal{A}^+_l} \mathbb{E}_\sigma\qty[\langle f(\sigma) - \frac{1}{N}1_N1_N^\top, f'(\sigma) - \frac{1}{N}1_N1_N^\top\rangle]\\
    &= \frac{1}{N}\sum_{f, f' \in \mathcal{A}^+_l} (\mathbb{E}_\sigma \qty[\langle f(\sigma), f'(\sigma) \rangle] - 1)\\
    &= \sum_{f, f' \in \mathcal{A}^+_l} \langle f, f' \rangle \\
    &\le\abs{\mathcal{A}^+_l} + (2r/N)^{k/2}(\abs{\mathcal{A}^+_l}^2 - \abs{\mathcal{A}^+_l}),
\end{align*}
where the second inequality uses $\E_\sigma\norm{g_l(\sigma)}_F^2 = \sum_{x,y \in [N]}\E_\sigma[g(\sigma, x, y)^2] = N$ from our choice of normalization. Altogether,
\begin{align*}
    \abs{\mathcal{A}^+_l} \le \frac{1}{\tau^2 - (2r/N)^{k/2}}.
\end{align*}
Similarly, $\abs{\mathcal{A}^-_l} \le \frac{1}{\tau^2 - (2r/N)^{k/2}}$. As such, the $i$th query can eliminate at most $\frac{2}{\tau^2 - (2r/N)^{k/2}}$ elements of $\mathcal{F}_r$, and thus the learner must make at least
\begin{align*}
    \abs{\mathcal{F}_r}\cdot\frac{\tau^2 - (2r/N)^{k/2}}{2} \ge (r!/4)^{k/2}\cdot\frac{\tau^2 - (2r/N)^{k/2}}{2}
\end{align*}
queries to recover $f$. Therefore we must either have tolerance $\tau^2 \le 2(2r/N)^{k/2}$, or make at least $\frac12(r!/4)^{k/2}\cdot (2r/N)^{k/2}$ queries. Choosing $r = \log N$, we see that we must have tolerance $$\tau^2 \le 2(2 \log N/N)^{-k/2} \lesssim \frac{\log^{k/2}N}{N^{k/2}},$$ or make at least $$\frac12\qty(\frac{(\log N)! \log N }{2N})^{k/2} \gtrsim N^{k/2 \cdot \log N}$$
queries.

We now consider the function class $\mathcal{F}^{\cyc}$. Let $\mathcal{F}^{\cyc}_r$ be the nearly orthogonal subset of $\mathcal{F}^{\cyc}$ constructed from Lemma \ref{lem:near-orth-set}. The queries are now of the form $g : \mathcal{X}^{\cyc} \times [N] \rightarrow \mathbb{R}$; to a query $g$ the SQ oracle returns a response $\hat q$ with $\abs{\hat q - \E_{\sigma, x, i}[g(\sigma, x, i, f(\sigma, x, i))]} \le \tau$. We assume that $\sum_{y \in N}\E_{\sigma, i, x}[g(\sigma, i, x, y)^2] = 1$, and $\E_\sigma[g(\sigma, i, x, y)] = 0$ for all $(i, x, y) \in [k] \times [N] \times [N]$. 

The proof proceeds identically to the non-cyclic case. On the $l$th query $g_l$, define
\begin{align*}
    \mathcal{A}^+_l &:= \{f \in \mathcal{F}^{\cyc}_r :  \mathbb{E}_{\sigma,x,i}[g_l(\sigma, x, i, f(\sigma, x, i))] \ge \tau\}\\
    \mathcal{A}^-_l &:= \{f \in \mathcal{F}^{\cyc}_r :  \mathbb{E}_{\sigma,x,i}[g_l(\sigma, x, i, f(\sigma, x, i))] \le -\tau\}.
\end{align*}
We will overload notation to let $g_l^{(i)}:(S_N)^k \rightarrow \mathbb{R}^{N \times N}$ be defined as $g_l^{(i)}(\sigma)_{x,y} := g_l(\sigma, x, i, y)$. We then see each query is of the form
\begin{align*}
    \mathbb{E}_{\sigma,x,i}[g_l(\sigma, x, i, f(\sigma, x, i))] = \frac{1}{Nk}\sum_{i=1}^k \mathbb{E}\qty[\langle g_l^{(i)}(\sigma), f^{(i)}(\sigma) \rangle].
\end{align*}
Then we can analogously bound
\begin{align*}
    \abs{\mathcal{A}^+_l}^2 \tau^2 &\le \qty(\sum_{f \in \mathcal{A}^+_l} \mathbb{E}_{\sigma,i, x}[g_l(\sigma, i, x, f(\sigma, i, x))])\\
    &\le \frac{1}{N^2k^2}\mathbb{E}_\sigma\qty[\sum_{i=1}^k\big\langle g_l^{(i)}(\sigma), \sum_{f \in \mathcal{A}^+_l}f^{(i)}(\sigma) \big\rangle]^2\\
    &= \frac{1}{N^2k^2}\mathbb{E}_\sigma\qty[\sum_{i=1}^k\big\langle g_l^{(i)}(\sigma), \sum_{f \in \mathcal{A}^+_l}\qty(f^{(i)}(\sigma) - \frac{1}{N}1_N1_N^\top )\big\rangle]^2,
\end{align*}
and thus by Cauchy-Schwarz
\begin{align*}
    \abs{\mathcal{A}^+_l}^2 \tau^2 \le N^{-2}k^{-2}\qty(\mathbb{E}_\sigma\qty[\sum_{i=1}^k \norm{g_l^{(i)}(\sigma)}_F^2]) \cdot \qty(\mathbb{E}_\sigma\qty[\sum_{i=1}^k \norm{\sum_{f \in \mathcal{A}^+_l}\qty(f^{(i)}(\sigma) - \frac{1}{N}1_N1_N^\top )}_F^2]).
\end{align*}
We see that
\begin{align*}
    \mathbb{E}_\sigma\qty[\sum_{i=1}^k \norm{\sum_{f \in \mathcal{A}^+_l}\qty(f^{(i)}(\sigma) - \frac{1}{N}1_N1_N^\top )}_F^2] &= \sum_{i=1}^k \sum_{f, f' \in \mathcal{A}^+_l}\E_\sigma\qty[\Big\langle f^{(i)}(\sigma) - \frac{1}{N}1_N1_N^\top, {f'}^{(i)}(\sigma) - \frac{1}{N}1_N1_N^\top \Big \rangle]\\
    &= \sum_{f, f' \in \mathcal{A}^+_l}\sum_{i=1}^k\qty(\E_\sigma[\langle f^{(i)}(\sigma), {f'}^{(i)}(\sigma)\rangle - 1])\\
    &= Nk \sum_{f, f' \in \mathcal{A}^+_l}\langle f, f' \rangle.
\end{align*}
Additionally,
\begin{align*}
    \mathbb{E}_\sigma\qty[\sum_{i=1}^k \norm{g_l^{(i)}(\sigma)}_F^2] = Nk\sum_y \E_{\sigma, i, x}[g_l^{(i)}(\sigma, i, x, y)^2] \le Nk.
\end{align*}
Combining these together, we have established $\abs{\mathcal{A}^+_l}^2 \tau^2 \le \sum_{f, f' \in \mathcal{A}^+_l}\langle f, f' \rangle$, an identical result as in the lower bound in the non-cyclic case. The remainder of the proof proceeds analogously.
\end{proof}

\section{Upper Bound: Analyzing the Gradient Dynamics}
\label{appendix:upper_bound}
Throughout this section, we simply denote $X:=X(\sigma),X_m:=X(\sigma^{(m)}),\hop_i^r(\sigma,j):=\hop_i^r(j)$ when $\sigma$ is clear from context for clarity. Note all the $\hop$ functions depend on the input $\sigma$.
\subsection{Architecture}
We use $L$-layer non-causal self-attention layers to learn this task, where $L = 1+\log_2k$. Define $X^{(0)}:=X$. We will let $X^{(\ell)}$ denote the output of the $\ell$th layer of the ground truth transformer which exactly computes the $k$-fold composition, defined in \Cref{thm:construction}. In particular, if $\{({W^{*,(\ell)}_{OV}}, W^{*,(\ell)}_{KQ})\}_{\ell \in [L]}$ are the weights from \Cref{thm:construction}, then
\begin{align*}
    X^{(\ell)}={}&X^{(\ell-1)}+f^{(\ell)}(X^{(\ell-1)}) \\
   f^{(\ell)}(X^{(\ell-1)}):={}& W^{*,(\ell)}_{OV}X^{(\ell-1)} \S({X^{(\ell-1)}}^\top W^{*,(\ell)}_{KQ}X^{(\ell-1)}).
\end{align*}
We refer to these $X^{(\ell)}$ as the \emph{ideal inputs} to each layer. Moreover, given some parameter vector $\theta := \{({W^{(\ell)}_{OV}}, W^{(\ell)}_{KQ})\}_{\ell \in [L]}$, define $\hat X^{(\ell)}$ to be the output of the $\ell$th layer of the transformer with parameters $\theta$ (where $\hat X^{(0)} = X^{(0)} = X$):
\begin{align*}
    {\hat X}^{(\ell)}={}&{\hat X}^{(\ell-1)}+f^{(\ell)}({\hat X}^{(\ell-1)}) \\
   f^{(\ell)}({\hat X}^{(\ell-1)}):={}& W^{(\ell)}_{OV}{\hat X}^{(\ell-1)} \S({{\hat X}^{(\ell-1)\top}} W^{(\ell)}_{KQ}{\hat X}^{(\ell-1)}).
\end{align*}
Throughout this section, we will consider $\theta$ to be an intermediate of \Cref{alg:training_alg}, and note that $\hat X^{(\ell)}$ depends implicitly on the stage of the layerwise training algorithm. We further assume the value matrices $W^{(\ell)}_{OV}$ are fixed to their ground truth values $W^{*,(\ell)}_{OV}$, which only updates the $\ell$th hop's block. In the end, we unembed the final output and predict the hop using a readout layer $\Psi_\ell$ for the $2^{\ell-1}$-th hop. We define the final output as $\mathrm{TF}_\theta(X):=\hat X^{(L)}.$

The population loss function for the $2^{\ell-1}$-th hop is the cross-entropy loss
\begin{equation}
        \mathcal{L}_{\mathcal{D}}^{(\ell)}(\theta) = -\E_{\sigma_1,\sigma_2,\cdots, \sigma_k,(i,j)}\qty[\sum_{s'\in[N]}\mathbf{1}\{s'=\hop^{2^{\ell-1}}_i(j)\}\log(\S(\Psi_\ell^\top \mathrm{TF}_\theta(X)_{(i,j)})_{s'})]
\end{equation}
where $\theta=\qty(W_{KQ}^{(1)},W_{KQ}^{(2)},\cdots, W_{KQ}^{(L)}, \Psi_1,\cdots,\Psi_L)$ are the trainable parameters.
For clarity, we ignore the subscripts of the expectation. If the loss is based on finite samples, we define 
\begin{equation}
        \mathcal{L}^{(\ell)}(\theta) = -\frac{1}{M}\sum_{m=1}^M\qty[\sum_{s'\in[N]}\mathbf{1}\{s'=\hop^{2^{\ell-1}}_i(j)\}\log(\S(\Psi_\ell^\top \mathrm{TF}_\theta(X_m)_{(i,j)})_{s'})]
\end{equation}

\subsection{Gradient Computation}
\label{subsec:grad_computation}
Formally, the model is initialized at $W_{KQ}^{(\ell)}(0) = 0_{d \times d}$ for all key-query matrices, and the readout/unembedding layer $\Psi_\ell(0)=\beta_0 e_{L+2,\ell+2}\otimes \mathbf{1}_{k}\otimes I_{N\times N}$ for small initialization scale $\beta_0$. We fix the value matrix for each layer as $W_{OV}^{(\ell)}=e_{L+2,\ell+2}e_{L+2,\ell+1}^\top\otimes I_{kN\times kN}$.

We expand the transformer into:
$$\mathrm{TF}_\theta(X)=X^{(0)}+\sum_{l=1}^{L}f^{(\ell)}(\hat X^{(\ell-1)}) $$
By the initialization of $\Psi_\ell$ and $W^{(\ell)}_{OV}$ and the definition of the embedding, we have $${(\Psi_{\ell'}^\top W_{OV}^{(\ell)})_{s'}}^\top=\beta_0\delta_{\ell\ell'} \cdot e_{L+2,\ell+1}\otimes \mathbf{1}_k \otimes e_{N,s'}.$$
where $\delta_{ij}=1$ iff $i=j$. That means the projected output $\Psi_\ell^\top f^{(\ell)}$ is non-zero only in stage $\ell$ when training on loss $\mathcal{L}^{(\ell)}$. Thus, we only need to consider the gradient of $W_{KQ}^{(\ell')}$ ($\ell'\le \ell$) in stage $\ell$ with the loss $\mathcal{L}^{(\ell)}$. 

\paragraph{Remark.} Note that $\nabla_{W_{KQ}^{(\ell')}}\mathcal{L}^{(\ell)}$ is \textbf{not exactly zero} for $\ell'<\ell$, because $X^{(\ell-1)}$ depends on $W^{(\ell')}_{KQ}$. However, since the $\ell'$th layer is already trained when $\ell'<\ell$, the softmax of that layer saturates. The Jacobian of the gradient is thus close to zero, making the update also close to zero and preserving the trained parameter $W^{(\ell')}_{KQ}$. It is formally proved in \Cref{appendix: gradient upper bound for trained layers}. For simplicity, we only consider the main parts of the update in each stage $\ell$, i.e. $\nabla_{W_{KQ}^{(\ell)}}\mathcal{L}^{(\ell)}$, in the following argument.

\paragraph{Population dynamics.} The population gradient of the loss of stage $\ell$ becomes:
\begin{align*}
    \nabla \mathcal{L}^{(\ell)}_{\mathcal{D}} = -\E\qty[\sum_{s'\in[N]}\qty(\mathbf{1}\{s'=\hop^{2^{\ell-1}}_i(j)\}-\S(\Psi_\ell^\top f^{(\ell)}(\hat X^{(\ell-1)})_{(i,j)})_{s'}) \nabla (\Psi_\ell^\top f^{(\ell)}(\hat X^{(\ell-1)})_{(i,j)})_{s'}]
\end{align*}
The model differential is 
    \begin{align*}
        \mathrm{d}f^{(\ell)}_{(i,j)}=  W^{(\ell)}_{OV}\hat X^{(\ell-1)}J^{(\ell)}(X, i, j)\hat X^{(\ell-1),\top}\mathrm{d} W_{KQ}^{(\ell)}\hat X^{(\ell-1)}_{(i,j)} + \mathrm{d}W^{(\ell)}_{OV}\hat X^{(\ell-1)} \S({\hat X^{(\ell-1)}}W^{(\ell)}_{KQ}X^{(\ell-1)\top}_{(i,j)}),
    \end{align*}
    where 
    \begin{align*}
    J^{(\ell)}(X, i, j)&:=(\diag(\S^{(\ell)}(X, i, j))-\S^{(\ell)}(X, i, j)(\S^{(\ell)}(X, i, j))^\top)\\
    \S^{(\ell)}(X, i, j)&:=\S({\hat X^{(\ell-1)\top}} W^{(\ell)}_{KQ}X^{(\ell-1)}_{(i,j)}).
    \end{align*}
The gradient of the loss $\mathcal{L}^{(\ell)}$ with respect to $W_{KQ}^{(\ell)}$ is
\begin{align*}
    &\nabla_{W_{KQ}^{(\ell)}} \mathcal{L}_{\mathcal{D}}^{(\ell)}(\theta)\\
    &=-\E\qty[\sum_{s'\in[N]}\qty(\mathbf{1}\{s'=\hop^{2^{\ell-1}}_i(j)\}-\S(\Psi_\ell^\top f^{(\ell)}(\hat X^{(\ell-1)})_{(i,j)})_{s'}) \nabla (\Psi_\ell^\top f^{(\ell)}(\hat X^{(\ell-1)})_{(i,j)})_{s'}]\\
    &=-\E\left[\sum_{s'\in[N]}\qty(\mathbf{1}\{s'=\hop_i^{2^{\ell-1}}(j)\}-\S(\Psi_\ell^\top f^{(\ell)}(\hat X^{(\ell-1)})_{(i,j)})_{s'}) \hat X^{(\ell-1)}J^{(\ell)}(X, i, j)\right. \\
    &\left.\quad \quad \quad \hat X^{(\ell-1)\top} {(\Psi_\ell^\top W_{OV}^{(\ell)})_{s'}}^\top {\hat X}_{(i,j)}^{(\ell-1)\top}\right]
\end{align*}
The gradient with respect to $\Psi_{\ell}$ is
\begin{align*}
    \nabla_{\Psi_{\ell}} \mathcal{L}^{(\ell)}_{\mathcal{D}}(\theta) &=-\E\qty[\sum_{s'\in[N]}\qty(\mathbf{1}\{s'=\hop^{2^{\ell-1}}_i(j)\}-\S(\Psi_\ell^\top f^{(\ell)}(\hat X^{(\ell-1)})_{(i,j)})_{s'}) \nabla_{\Psi_{\ell}} (\Psi_\ell^\top f^{(\ell)}(\hat X^{(\ell-1)})_{(i,j)})_{s'}]\\
    &=-\E\qty[\qty(e_{N,\hop^{2^{\ell-1}}_i(j)}-\S(\Psi_\ell^\top f^{(\ell)}(\hat X^{(\ell-1)})_{(i,j)})) (f^{(\ell)}(\hat X^{(\ell-1)})_{(i,j)})^\top].
\end{align*}

Throughout the proof in the following sections, we will want to compute "exact gradients", where we assume that $\hat X^{(\ell')} = X^{(\ell')}$ for all $\ell' < \ell$. 

Finally, we note that the initial output probabilities (after the $\softmax(\cdot)$) for the ideal input can be computed as
\begin{align*}
\mathcal{P}^{(\ell)}_{s'}(X,i,j):={}&\S(\Psi_\ell^\top f^{(\ell)}(X^{(\ell-1)})_{s'})\\
={}&\S(\Psi_\ell^\top (X^{(\ell-1)}+ W^{(\ell)}_{OV}X^{(\ell-1)} \S({X^{(\ell-1)}}^\top W^{(\ell)}_{KQ}X_{(i,j)}^{(\ell-1)})))_{s'}\\
={}&\S\qty(0+\frac{\beta_0}{kN} \cdot (e_{L+2,\ell+1}\otimes \mathbf{1}_k \otimes I_{N\times N})^\top X^{(\ell-1)}\mathbf{1}_{kN})_{s'}
\end{align*}
Given that $X^{(\ell-1)}$ are ideal inputs for each layer $\ell$, $(e_{L+2,\ell+1}\otimes \mathbf{1}_k \otimes I_{N\times N})^\top X^{(\ell-1)}\mathbf{1}_{kN}=k\1_N.$ That means the output probability $\mathcal{P}_{s'}(X,i,j)$ is $\frac1N$. We note that when the input indices $(i,j)$ are clear from context (e.g. when the single query $(i,j)$ is given), we ignore the $(X,i,j)$ in both probability vectors $\mathcal{P}^{(\ell)}(X,i,j)$ and $\S^{(\ell)}(X,i,j)$ as well as the Jacobian $J^{(\ell)}(X, i, j)$.

\subsection{Proof of Main Theorem}
Here we restate the main theorem.
\begin{theorem}[Guarantee for \Cref{alg:training_alg}] Assume $M\ge \Tilde{\Omega}(k^4N^6)$ and $\eta \ge \tilde\Omega(\frac{k^2N^3}{\beta_0}\log\frac{1}{\epsilon}).$ For any $\epsilon$ satisfying $0<\epsilon \log\frac{1}{\epsilon}\le \tilde{O}(\frac{1}{k^6N^6})$ with probability 0.99, the final output $\hat\theta$ of \Cref{alg:training_alg} satisfies that over any draw of the input permutation $\sigma$ and the query index $(i,j)$,
$$\sup_{\sigma,(i,j)}\norm{\S(\Psi_L^\top \mathrm{TF}_{\hat{\theta}}(X(\sigma))_{(i,j)})-e_{\hop^k_i(j)}}_\infty \le \epsilon$$
\end{theorem}
\begin{proof} We provide an outline of the proof in this section.

\textbf{Stage 1.} We first prove that in \textbf{stage 1}, the first layer $W_{KQ}^{(1)}$ learns all the hidden permutations.

Using \Cref{lemma: empirical gradient of the first stage}, after the first step gradient we have for any $(i,j)$, the softmax probability satisfies 
    $$\S^{(1)}_{(i,\pi_i(j))}:=\S(X^{(0)}W_{KQ}^{(1)}X^{(0)}_{(i,j)}) \ge 1-\frac{1}{2}\epsilon.$$
We can then calculate the intermediate sequence $\hat{X}^{(1)}$:
\begin{align*}
    \hat{X}^{(1)} = X^{(0)} + W_{OV}^{(1)}X^{(0)}\S^{(1)} = X^{(1)} + W_{OV}^{(1)}X^{(0)}{(\S^{(1)}-\S^{(1)}_{\mathrm{ideal}})}.
\end{align*}
where $\S_{\mathrm{ideal}}$ is the ideal one-hot softmax attention score without the non-satuaration error. 
The ideal input sequence for the second layer should be
\begin{align*}
    X^{(1)} = \begin{pNiceArray}{ccc|c|ccc}
 e_{k,1}\otimes e_{N,1}&\cdots&e_{k,1}\otimes e_{N,N}&\cdots&e_{k,k}\otimes e_{N,1}&\cdots&e_{k,1}\otimes e_{N,N}\\
e_{k,1}\otimes e_{N,\sigma_1(1)}&\cdots& e_{k,1}\otimes e_{N,\sigma_1(N)}&\cdots&e_{k,k}\otimes e_{N,\sigma_k(1)}&\cdots&e_{k,k}\otimes e_{N,\sigma_k(N)}\\
e_{k,1}\otimes e_{N,\sigma_1\pi_1(1)}&\cdots& e_{k,1}\otimes e_{N,\sigma_1\pi_1(N)}&\cdots&e_{k,k}\otimes e_{N,\sigma_k\pi_k(1)}&\cdots&e_{k,k}\otimes e_{N,\sigma_k\pi_k(N)}\\
0_{(L-1)kN} &\cdots&0_{(L-1)kN}&\cdots&0_{(L-1)kN}&\cdots&0_{(L-1)kN}
\end{pNiceArray}
\end{align*}
By \Cref{lemma: perturbation for first stage}, $\|X^{(1)}-\hat{X}^{(1)}\|_\infty\le \epsilon$. By \Cref{lemma: readout for first layer}, we prove that after a large step of GD on $\Psi_1$ with $M\ge \tilde\Omega(k^4N^6)$ 
, there is a probability at least $1-0.01/L$ that the transformer learns the 1-hop with $\epsilon$ error, thus completing the proof for stage 1.

\textbf{Stage 2.} By \Cref{lemma: empirical gradient of the second stage} we have after the second stage, for any $(i,j)$ 
$$\S^{(2)}_{(i+1,\hop^1_i(j))}:=\S\qty((\hat{X}^{(1)})^\top W_{KQ}^{(2)}\hat{X}^{(1)}_{(i,j)}) \ge 1-\frac{1}{2}\epsilon.$$
The ideal input sequence for the third layer should be
\begin{align*}
X^{(2)} = \begin{pNiceArray}{ccc|c|ccc}
 e_{k,1}\otimes e_{N,1}&\cdots&e_{k,1}\otimes e_{N,N}&\cdots&e_{k,k}\otimes e_{N,1}&\cdots&e_{k,1}\otimes e_{N,N}\\
e_{k,1}\otimes e_{N,\sigma_1(1)}&\cdots& e_{k,1}\otimes e_{N,\sigma_1(N)}&\cdots&e_{k,k}\otimes e_{N,\sigma_k(1)}&\cdots&e_{k,k}\otimes e_{N,\sigma_k(N)}\\
e_{k,1}\otimes e_{N,\sigma_1\pi_1(1)}&\cdots& e_{k,1}\otimes e_{N,\sigma_1\pi_1(N)}&\cdots&e_{k,k}\otimes e_{N,\sigma_k\pi_k(1)}&\cdots&e_{k,k}\otimes e_{N,\sigma_k\pi_k(N)}\\
e_{k,2}\otimes e_{N,\hop^2_1(1)}&\cdots& e_{k,2}\otimes e_{N,\hop^2_1(1)}&\cdots&e_{k,1}\otimes e_{N,\hop^2_1(1)}&\cdots&e_{k,1}\otimes e_{N,\hop^2_k(N)}\\
0_{(L-2)kN} &\cdots&0_{(L-2)kN}&\cdots&0_{(L-1)kN}&\cdots&0_{(L-2)kN}
\end{pNiceArray}
\end{align*}
By \Cref{lemma: perturbation for second stage}, $\|X^{(2)}-\hat{X}^{(2)}\|_\infty\le 2\epsilon$. By \Cref{lemma: readout for second layer}, we prove that with another large step of GD with $M\ge \tilde{\Omega}(k^4N^6)$, there is a probability at least $1-0.01/L$ the transformer learns the 2-hop with $\epsilon$ error. Thus we finish the proof for stage 2.

\textbf{Stage $\ell$.} By \Cref{lemma: empirical gradient of stage ell}, after training the key-query matrix for layer $\ell$, we have 
$$\S^{(\ell)}_{(i+2^{\ell-2},\hop^{2^{\ell-2}}_i(j))}:=\S((\hat{X}^{(\ell-1)})^\top W_{KQ}^{(\ell)}\hat{X}^{(\ell-1)}_{(i,j)})_{(i+2^{\ell-2},\hop^{2^{\ell-2}}_i(j))} \ge 1-\frac{1}{2}\epsilon.$$
The ideal input for each column $(i,j)$ should be
\begin{align*}
    X^{(\ell)}_{(i,j)} = \begin{bmatrix}
        e_{k,i}\otimes e_{N,j}\\
        e_{k,i}\otimes e_{N,\sigma_i(j)}\\
        e_{k,i}\otimes e_{N,\hop^1_i(j)}\\
        e_{k,i+1}\otimes e_{N,\hop^2_i(j)}\\
        \vdots\\
        e_{k,i+2^{\ell-1}-1}\otimes e_{N,\hop^{2^{\ell-1}}_i(j)}\\
        0_{(L-\ell)kN}
    \end{bmatrix}   
\end{align*}
By \Cref{lemma: perturbation for stage ell}, $\|X^{(\ell)}-\hat{X}^{(\ell)}\|_\infty\le \ell\epsilon$. By \Cref{lemma: readout for ellth layer}, we prove that with another large step of GD with $M\ge \tilde\Omega(k^4N^6)$, there is a probability at least $1-0.01/L$ the transformer learns the $2^{\ell-1}$-hop with $\epsilon$ error. When $\ell = L$, the result implies that the transformer learns the $k$-fold composition task. The failure probability is upper bounded by 0.99 using union bound. Thus, we complete the proof.

\end{proof}

In the following sections, we provide detailed proofs for the supplementary lemmas for each stage.

\subsection{Stage 1: Learning the Hidden Permutations $\pi_i$}
We first consider the gradient of the first layer. We show that during the first stage of training, $W_{KQ}^{(1)}$ learns all the hidden permutations $\pi_i$. As a result, the first layer attention predicts the correct one-hop $\hop^1_i(j)=\sigma_i\pi_i(j)$ for all $(i,j)$. 

\begin{lemma}[Empirical gradient of $W^{(1)}_{KQ}$ learns all $\pi_i(\cdot)$ (Stage 1)]
\label{lemma: empirical gradient of the first stage} Let $W^{(\ell)}_{KQ}(0)=0_{d\times d}$ for all layers $l\in[L]$. After one-step of gradient descent on the first stage finite sample loss $\mathcal{L}^{(1)}$ with $M$ training sequences and learning rate $\eta$, satisfying  $\beta_0\le 1,$
    $M\gtrsim k^2N^4d^2\log^2\frac{d}{\delta},\eta \gtrsim \frac{k^2N^3\log \frac{kN}{\epsilon}}{\beta_0}$, then with probability $1-\delta$ we have for any $(i,j)\in[k]\times [N]$, $$\S^{(1)}_{(i,\pi_i(j))}:=\S(X^{(0)}W_{KQ}^{(1)}X^{(0)}_{(i,j)})_{(i,\pi_i(j))} \ge 1-\frac{1}{2}\epsilon.$$
    Furthermore, if we pick $\eta \gtrsim \frac{Ck^2N^3\log k\log \frac{kN}{\epsilon}}{\beta_0}$, we have
    $\S^{(1)}_{(i,\pi_i(j))}\ge 1-\frac{\epsilon}{2(kNL)^{CL}}$ for any absolute constant $C$.
\end{lemma}
\begin{proof}
We first compute the population gradient, and then do the finite sample analysis to bound the noise of the empirical gradient. For simplicity, we ignore the superscipt $^{(0)}$ of the input sequence and simply denote the input sequence as $X$ in the following calculation when it is clear from context.

Recall the gradient on the population loss is
\begin{align*}
    \nabla_{W_{OV}^{(1)}} \mathcal{L}_{\mathcal{D}}^{(1)} = -\E\qty[\sum_{s'\in[N]}\qty(\mathbf{1}\{s'=\sigma_i\pi_i(j)\}-\mathcal{P}^{(1)}_{s'}) XJ^{(1)} X^\top {\qty(\Psi_1^\top W_{OV}^{(1)})_{s'}}^\top {X_{(i,j)}}^\top]
\end{align*}
Since the input for the first layer is exactly the ideal input, we have the output probability
\begin{align*}
\mathcal{P}^{(1)}_{s'}:={}&\S(\Psi_1^\top f^{(1)}(X^{(0)})_{(i,j)})_{s'}\\
={}&\S(\Psi_1^\top (X^{(0)}+ W^{(1)}_{OV}X^{(0)} \S({X^{(0)}}^\top W^{(1)}_{KQ}X_{(i,j)}^{(0)})))_{s'}\\
={}&\S\qty(0+\frac{\beta_0}{kN} \cdot (e_{L+2,2}\otimes \mathbf{1}_k \otimes I_{N\times N})^\top X^{(0)}\mathbf{1}_{kN})_{s'}\\
={}&\S\qty(\frac{\beta_0}{N} \1_N)_{s'}=\frac{1}N.
\end{align*}
So the gradient becomes
\begin{align*}
    \nabla_{W^{(1)}_{KQ}} \mathcal{L}^{(1)}_{\mathcal{D}}(\theta) = -\E\qty[\sum_{s'\in[N]}\qty(\mathbf{1}\{s'=\sigma_i\pi_i(j)\}-\frac{1}{N}) XJ^{(1)} X^\top {\qty(\Psi_1^\top W_{OV}^{(1)})_{s'}}^\top X_{(i,j)}^\top]
\end{align*}

We first notice that the normalization term vanishes due to the Jacobian: 
\begin{align*}
    \E\qty[\sum_{s'\in[N]}\frac1N XJ X^\top {\qty(\Psi_1^\top W_{OV}^{(1)})_{s'}}^\top{X_{(i,j)}}^\top] = \E\qty[\frac1N XJ \mathbf{1}_{kN} {X_{(i,j)}}^\top]=0.
\end{align*}

Therefore, the idealized gradient equals to the signal term:
\begin{align*}
    \nabla_{W^{(1)}_{KQ}} \mathcal{L}^{(1)}_{\mathcal{D}}(\theta) &= -\E\qty[\sum_{s'\in[N]}\mathbf{1}\{s'=\sigma_i\pi_i(v)\} XJ X^\top {\qty(\Psi_1^\top W_{OV}^{(1)})_{s'}}^\top X_{(i,j)}^\top]\\
    &= -\frac{1}{kN}\E\qty[X\qty(I_{kN\times kN}-\frac{1}{kN}\mathbf{1}\mathbf{1}^\top) X^\top {\qty(\Psi_1^\top W_{OV}^{(1)})_{\sigma_{i}\pi_i(v)}}^\top X_{(i,j)}^\top].
\end{align*}

Now the input sequence in this stage is
\begin{align*}
    X^{(0)} = \begin{pNiceArray}{ccc|c|ccc}
e_{k,1}\otimes e_{N,1}&\cdots&e_{k,1}\otimes e_{N,N}&\cdots&e_{k,k}\otimes e_{N,1}&\cdots&e_{k,1}\otimes e_{N,N}\\
e_{k,1}\otimes e_{N,\sigma_1(1)}&\cdots& e_{k,1}\otimes e_{N,\sigma_1(N)}&\cdots&e_{k,k}\otimes e_{N,\sigma_k(1)}&\cdots&e_{k,k}\otimes e_{N,\sigma_k(N)}\\
0 &\cdots&0&\cdots&0&\cdots&0\\
\end{pNiceArray}
\end{align*}
The following term indicates which coordinate in each block corresponds to the hop label $\sigma_i\pi_i(j)$:
\begin{align*}
     &X^\top {\qty(\Psi_1^\top W_{OV}^{(1)})_{\sigma_{i}\pi_i(j)}}^\top\\ ={}& \begin{pNiceArray}{cc|c|ccc}
e_{k,1}\otimes e_{N,1}&\cdots&\cdots&e_{k,k}\otimes e_{N,1}&\cdots&e_{k,1}\otimes e_{N,N}\\
e_{k,1}\otimes e_{N,\sigma_1(1)}&\cdots&\cdots&e_{k,k}\otimes e_{N,\sigma_k(1)}&\cdots&e_{k,k}\otimes e_{N,\sigma_k(N)}\\
0 &\cdots&\cdots&0&\cdots&0\\
\end{pNiceArray}^\top \begin{bmatrix}
    0_{kN}\\
    \beta_0\1_{k}\otimes e_{N,\sigma_i(\pi_i(j))}\\
    0_{kNL}
\end{bmatrix}\\
     ={}&\beta_0\begin{pNiceArray}{c|c|c|c}
e_{N,\sigma_1^{-1}\sigma_i\pi_i(j)}^\top&e_{N,\sigma_2^{-1}\sigma_i\pi_i(j)}^\top&\cdots&e_{N,\sigma_k^{-1}\sigma_i\pi_i(j)}^\top
\end{pNiceArray}^\top
\end{align*}
Note that $\sigma_i$ is a permutation for each block $j\in [k]$, so there is only one position mapping to $\sigma_i\pi_i(j)$. Since permutation is invertible, the coordinate is $\sigma_j^{-1}\sigma_i\pi_i(j)$. 

With this expression, we can further simplify the gradient:
\begin{align*}
    \nabla \mathcal{L}_{\mathcal{D}}^{(1)}(\theta)
    &= -\frac{1}{kN}\E\qty[X\qty(I_{kN\times kN}-\frac{1}{kN}\mathbf{1}\mathbf{1}^\top) X^\top {\qty(\Psi_1^\top W_{OV}^{(1)})_{\sigma_{i}\pi_i(j)}}^\top {X_{(i,j)}}^\top]\\
    &= -\frac{\beta_0}{kN}\E\qty[X\qty(I_{kN\times kN}-\frac{1}{kN}\mathbf{1}\mathbf{1}^\top) \begin{pNiceArray}{c|c|c}
     e_{N,\sigma_1^{-1}\sigma_i\pi_i(j)}^\top&\cdots&e_{N,\sigma_k^{-1}\sigma_i\pi_i(j)}^\top\end{pNiceArray}^\top {X_{(i,j)}}^\top]\\
     &=-\frac{\beta_0}{kN}\E\qty[X\qty(\begin{pNiceArray}{c|c|c}
     e_{N,\sigma_1^{-1}\sigma_i\pi_i(j)}^\top&\cdots&e_{N,\sigma_k^{-1}\sigma_i\pi_i(j)}^\top\end{pNiceArray}^\top - \frac
    1N \mathbf{1}) {X_{(i,j)}}^\top]\\
    &=-\frac{\beta_0}{kN}\E\qty[\qty(\begin{bmatrix}
        \sum_{p=1}^k e_{k,p}\otimes e_{N,\sigma_p^{-1}\sigma_i\pi_i(j)}\\
        \sum_{p=1}^k e_{k,p}\otimes e_{N,\sigma_i\pi_i(j)}\\
        0_{kNL}
    \end{bmatrix}-\frac1N\begin{bmatrix}
        \mathbf{1}_{kN}\\
        \mathbf{1}_{kN}\\
        0_{kNL}
    \end{bmatrix}) \begin{bmatrix}
        e_{k,i}\otimes e_{N,j} \\
        e_{k,i}\otimes e_{N,\sigma_i(j)}\\
        0_{kNL}
    \end{bmatrix}^\top]\\
    &:=-\frac{\beta_0}{kN} \E\begin{bmatrix}
        A_1&A_2&0_{kN\times kNL}\\
        A_3&A_4&0_{kN\times kNL}\\
        0_{kNL\times kN}&0_{kNL\times kN}&0_{kNL\times kNL}
    \end{bmatrix}\tag{$A_1,\cdots, A_4\in \R^{kN\times kN}$.}
\end{align*}
Suppose $i$ is given. Since $\sigma_p$ are independent of $(i,j)$ and $\sigma_i$ for $p\not=i$, the expectation
$$\E\qty[(e_{k,p}\otimes e_{N,\sigma_p^{-1}\sigma_i\pi_i(j)} )\begin{bmatrix}
        e_{k,i}\otimes e_{N,v}\\
        e_{k,i}\otimes e_{N,\sigma_i(v)}
    \end{bmatrix}^\top\bigg|i,p\not=i] = \frac{1}{N^2}(e_{k,p}\otimes \mathbf{1}_{N})( e_{k,i}\otimes \mathbf
{1}_{N})^\top$$
and cancels with the normalization term introduced by the jacobian.

Meanwhile, the $i$-th block recovers the adajency matrix of the permutation $\pi_i$: 
$$\E\qty[(e_{k,i}\otimes e_{N,\sigma_i^{-1}\sigma_i\pi_i(j)})\begin{bmatrix}
        e_{k,i}\otimes e_{N,j}\\
        e_{k,i}\otimes e_{N,\sigma_i(j)}
    \end{bmatrix}^\top\bigg|i] = \frac1N\sum_{j=1}^N ( e_{k,i}\otimes e_{N,\pi_i(j)}) ( e_{k,i}\otimes e_{N,j})^\top$$
Therefore, we have the first row of the block matrices
$$\E\qty[A_1|i]=\frac1N\sum_{j=1}^N (e_{k,i}\otimes e_{N,\pi_i(j)}) (e_{k,i}\otimes e_{N,j})^\top-\frac{1}{N^2}(e_{k,i}\otimes \mathbf{1}_{N})(e_{k,i}\otimes \mathbf
{1}_{N})^\top, \ \E[A_2|i]=0.$$
The overall expectation of $A_1$ is
$$\E[A_1] = \frac1{kN}\sum_{i=1}^k\sum_{j=1}^N (e_{k,i}\otimes e_{N,\pi_i(j)}) (e_{k,i}\otimes e_{N,j})^\top-\frac{1}{kN^2}\sum_{i=1}^k(e_{k,i}\otimes \mathbf{1}_{N})(e_{k,i}\otimes \mathbf
{1}_{N})^\top.$$
Similarly, the first expectation of the second row is $\E[A_3]=0$. For $A_4$,
$$\E[A_4|i]=\sum_{p=1}^k\E[(e_{k,p}\otimes e_{N,\sigma_i\pi_i(j)})(e_{k,i}\otimes e_{N,\sigma_i(j)})^\top|i] -\frac{1}{N^2}(\1_k\otimes \mathbf{1}_{N})(e_{k,i}\otimes \mathbf{1}_{N})^\top$$
Note that given a fixed index $v\in[N]$, the expectation is $\frac{1}{N}I$ if $\pi_i(v)=v$, and $\frac{1}{N(N-1)}\mathbf{1}_N\mathbf{1}_N^\top-\frac 1{N(N-1)}I_N$ when $\pi_i(v)\not=v$. 
Suppose $f(\pi_i)$ is the number of fixed point of $\pi_i$, we have
$$\E\qty[A_4|i]=(\1_ke_{k,i}^\top)\otimes \qty(\frac{f(\pi_i)-1}{N(N-1)}I_N+\frac{N-f(\pi_i)}{N^2(N-1)}\1_N\1_N^\top)-\frac{1}{N^2}(\1_k\otimes \mathbf{1}_{N})(e_{k,i}\otimes \mathbf
{1}_{N})^\top.$$
The expectation for $A_4$ is 
$$\E\qty[A_4]=\frac{1}{k}\sum_{i=1}^k(\1_ke_{k,i}^\top)\otimes \qty(\frac{f(\pi_i)-1}{N(N-1)}I_N+\frac{N-f(\pi_i)}{N^2(N-1)}\1_N\1_N^\top)-\frac{1}{kN^2}(\1_k\otimes \mathbf{1}_{N})(\1_k\otimes \mathbf
{1}_{N})^\top.$$

Now we consider the empirical estimate of the gradient. We define the matrix
$$g_{1} = \E\begin{bmatrix}
        A_1&A_2&0_{kN\times kNL}\\
        A_3&A_4&0_{kN\times kNL}\\
        0_{kNL\times kN}&0_{kNL\times kN}&0_{kNL\times kNL}
    \end{bmatrix}$$
The empirical gradient can be written as
$\nabla_{W_{KQ}^{(1)}} \hat{\mathcal{L}}^{(1)} = -\frac{\beta_0}{kN}\hat{g}_1,$
where $\hat{g}_1$ is the empirical estimate of $g_1$:
$$\hat{g}_1 = \frac{kN}{\beta_0M}\sum_{m=1}^M \sum_{s'\in[N]}\mathbf{1}\{s'=\hop^{1}_{i_m}(j_m)\} X_mJ X_m^\top {\qty(\Psi_1^\top W_{OV}^{(1)})_{s'}}^\top X_{(i_m,j_m)}^\top.$$
It suffices to show that $\|\hat{g}_1-g_1\|_\infty$ is small, which is shown by the following lemma:
\begin{lemma}\label{lemma: concentration for stage 1}
    For any $\delta>0$, we have that with probability $1-\delta$,
    $$\|\hat{g}_1-g_1\|_\infty\lesssim \frac{\log(kN/\delta)}{\sqrt{M}}$$
\end{lemma}
\begin{proof}
First, the empirical gradient of the single sample $X_m$ and $(i_m,j_m)$ is
\begin{align*}
    \hat{g}_{1,m} &= \frac{kN}{\beta_0} \sum_{s'\in[N]}\mathbf{1}\{s'=\hop^{1}_{i_m}(j_m)\} X_mJ X_m^\top {\qty(\Psi_1^\top W_{OV}^{(1)})_{s'}}^\top X_{(i_m,j_m)}^\top\\
    &=X_m (I-\frac{1}{kN}\1\1^\top)X_m^\top (e_{L+2,2}\otimes \1_k\otimes e_{N,\hop^1_{i_m}(j_m)}) X_{(i_m,j_m)}^\top
\end{align*}
The upper bound of the single entry of the random variable is
\begin{align*}
    \|\hat{g}_{1,m}\|_\infty 
    &\le \left\|X_m (I-\frac{1}{kN}\1\1^\top)X_m^\top (e_{L+2,2}\otimes \1_k\otimes e_{N,\hop^1(j_m)})\right\|_\infty \|X_{(i_m,j_m)}\|_\infty\\
    &=\left\|X_m (I-\frac{1}{kN}\1\1^\top)X_m^\top (e_{L+2,2}\otimes \1_k\otimes e_{N,\hop^1(j_m)})\right\|_\infty\tag{$\|X_{(i_m,j_m)}\|_\infty=1.$}\\
    &=\left\|\qty[\qty(\begin{bmatrix}
        \sum_{j=1}^k e_{k,j}\otimes e_{N,\sigma_j^{-1}\hop^1(v_m)}\\
        \sum_{j=1}^k e_{k,j}\otimes e_{N,\hop^1(j_m)}\\
        0_{kNL}
    \end{bmatrix}-\frac1N\begin{bmatrix}
        \mathbf{1}_{kN}\\
        \mathbf{1}_{kN}\\
        0_{kNL}
    \end{bmatrix})]\right\|_\infty\le 1.
\end{align*}
Then we can concentrate $\hat{g}_1 = \frac{1}{M}\sum_{m=1}^M \hat{g}_{1,m}$ as:
\begin{align*}
    \norm{\hat{g}_1-g_1}_\infty &= \norm{\frac{1}{M}\sum_{m=1}^M \hat{g}_{1,m}-\E\hat{g}_{1}}\lesssim \frac{\log(d/\delta)}{\sqrt{M}}
\end{align*}
with probability $1-\frac{\delta}{d^2}$. Union bounding over all entries of $\hat{g}_1$ and we have the desired result since $d = O(kN\log k)$.
\end{proof}

After the first step gradient, we get $W_{KQ}^{(1)}=\eta \hat{g}_1$. Now we compute the probability output of the softmax attention.
Given a sample sequence $X_m$ and index $(i_m,j_m)$, we have the attention score
\begin{align*}
    \eta X_m^\top \hat{g}_1 X_{(i_m,j_m)} &= \frac{\beta_0\eta}{kN}\qty[X_m^\top {g}_1 X_{(i_m,j_m)} - X_m^\top \underbrace{(\hat{g}_1-g_1)}_{\Delta g_1} X_{(i_m,j_m)}]
\end{align*}
Note that the population attention score is
\begin{align*}
    X_m^\top {g}_1 X_{(i_m,j_m)}&=X_m^\top\E\begin{bmatrix}
        A_1&A_2&0_{kN\times kNL}\\
        A_3&A_4&0_{kN\times kNL}\\
        0_{kNL\times kN}&0_{kNL\times kN}&0_{kNL\times kNL}
    \end{bmatrix}\begin{bmatrix}
        e_{k,i_m}\otimes e_{N,j_m}\\
        e_{k,i_m}\otimes e_{N,\sigma_{i_m}(j_m)}\\
        0_{kNL}
    \end{bmatrix} \\
    &=X_m^\top \begin{bmatrix}
        \frac{1}{kN}\qty(e_{k,i_m}\otimes e_{N,\pi_{i_m}(j_m)}-\frac{1}{N}e_{k,i_m}\otimes \1_N)\\
        \frac{f(\pi_i)-1}{kN(N-1)}\qty(\1_k\otimes e_{N,\sigma_{i_m}(j_m)}-\frac{1}{N}\1_k\otimes \1_N)\\
        0_{kNL}
    \end{bmatrix}
\end{align*}
Define the population/empirical separation between the correct position and the others as follows:
$$\Delta^{(1)}_{i_m,j_m} = \qty(X_m^\top {g}_1 X_{(i_m,j_m)})_{(i_m,\pi_{i_m}(j_m))}-\max_{p\not=(i_m,\pi_{i_m}(j_m))}\qty(X_m^\top {g}_1 X_{(i_m,j_m)})_p.$$
$$\hat{\Delta}^{(1)}_{i_m,j_m} = \qty(X_m^\top \hat{g}_1 X_{(i_m,j_m)})_{(i_m,\pi_{i_m}(j_m))}-\max_{p\not=(i_m,\pi_{i_m}(j_m))}\qty(X_m^\top \hat{g}_1 X_{(i_m,j_m)})_p.$$
If $\pi_{i_m}(v)=v$ for all $v\in[N]$, i.e. $f(\pi_{i_m})=N$, we have the $(i,j)$-th entry of the pre-softmax attention score
\begin{align*}
    \qty(X_m^\top {g}_1 X_{(i_m,j_m)})_{(i,j)} = \begin{cases}
        \frac{2N-2}{kN^2},&i=i_m, j=j_m\\
        -\frac{2}{kN^2},&i=i_m,\ j\not=j_m\\
        -\frac{1}{kN^2},&i\not=i_m, \sigma_i(j)\not=\sigma_{i_m}(j_m)\\
        \frac{N-1}{kN^2},&i\not=i_m, \sigma_i(j)=\sigma_{i_m}(j_m)\\
    \end{cases}
\end{align*}
So $\Delta^{(1)}_{i_m,j_m}\ge \frac{1}{kN^2}$ by $N\ge 2$.

If not all $v$ satisfies $\pi_{i_m}(v)=v$, $f(\pi_i)\le N-2.$ We have the $(i,j)$-th entry of the pre-softmax attention score
\begin{align*}
    \qty(X_m^\top {g}_1 X_{(i_m,j_m)})_{(i,j)} = \begin{cases}
        \frac{N-1}{kN^2},&i=i_m,j=\pi_{i_m}(j_m)\\
        \frac{f(\pi_{i_m})-1}{kN^2},&\sigma_i(j)=\sigma_{i_m}(j_m)\\
        -\frac{1}{kN^2}-\frac{f(\pi_{i_m})-1}{kN^2(N-1)},&i=i_m,\ j\not=j_m,\pi_{i_m}(j_m)\\
        -\frac{1}{kN^2},&i\not=i_m, \sigma_i(j)\not=\sigma_{i_m}(j_m)
    \end{cases}
\end{align*}
So $\Delta^{(1)}_{i_m,j_m}\ge \frac{1}{kN^2}$ by $N\ge 2$ for both cases. Therefore, we have the expected signal at least $\frac{1}{kN^2}$.

Now we bound the noise introduced by $\Delta g_1$. Since $\|\Delta g_1\|_\infty \lesssim \frac{\log \frac{d}{\delta}}{\sqrt{M}}$, we have
$$\left\|X_m^\top {\Delta g_1} X_{(i_m,j_m)}\right\|_\infty \lesssim \frac{d\log \frac{d}{\delta}}{\sqrt{M}}\le \frac{1}{2kN^2}$$
by $M\gtrsim k^2N^4d^2\log^2\frac{d}{\delta}$. Therefore, the noise term can be bounded and the empirical separation $\hat{\Delta}^{(1)}_{i_m,j_m}\ge \frac{1}{2kN^2}$. After one step gradient with learning rate $\eta \gtrsim \frac{k^2N^3}{\beta_0}\log\frac{kN}{\epsilon}$, the softmax output of the correct position can be lower bounded by
\begin{align*}
    \S(X_m^\top W_{KQ}^{(1)} X_{(i_m,j_m)})_{i_m,\pi_{i_m}(j_m)} \ge \frac{\exp\qty(\frac{\eta\beta_0}{kN}\cdot \hat{\Delta}^{(1)}_{i_m,j_m})}{\exp\qty(\frac{\eta\beta_0}{kN}\cdot \hat{\Delta}^{(1)}_{i_m,j_m})+kN-1}\ge 1-\frac{1}{2}\epsilon.
\end{align*}
If we increase the learning rate with $2C\log k$ times, we can accordingly get
$\S^{(1)}_{(i,\pi_i(j))}\ge 1-\frac{\epsilon}{2(kNL)^{CL}}$ for some absolute constant $C$.
\end{proof}

\paragraph{Perturbation analysis of the output} After learning the first layer, we analyze the perturbation that non-saturation of the softmax prediction introduced to the ideal output.
\begin{lemma}[Perturbation analysis of stage 1]\label{lemma: perturbation for first stage}
    Under the conditions of \Cref{lemma: empirical gradient of the first stage}, we have
    $$\|X^{(1)}-\hat{X}^{(1)}\|_\infty\le \epsilon,$$
    where $X^{(1)}$ is the ideal output with saturated softmax, and $\hat{X}^{(1)}$ is the true transformer output.
\end{lemma}
\begin{proof}
After the first stage, the ideal input sequence for the second layer should be
\begin{align*}
    X^{(1)} = \begin{pNiceArray}{ccc|c|ccc}
 e_{k,1}\otimes e_{N,1}&\cdots&e_{k,1}\otimes e_{N,N}&\cdots&e_{k,k}\otimes e_{N,1}&\cdots&e_{k,1}\otimes e_{N,N}\\
e_{k,1}\otimes e_{N,\sigma_1(1)}&\cdots& e_{k,1}\otimes e_{N,\sigma_1(N)}&\cdots&e_{k,k}\otimes e_{N,\sigma_k(1)}&\cdots&e_{k,k}\otimes e_{N,\sigma_k(N)}\\
e_{k,1}\otimes e_{N,\sigma_1\pi_1(1)}&\cdots& e_{k,1}\otimes e_{N,\sigma_1\pi_1(N)}&\cdots&e_{k,k}\otimes e_{N,\sigma_k\pi_k(1)}&\cdots&e_{k,k}\otimes e_{N,\sigma_k\pi_k(N)}\\
0_{(L-1)kN} &\cdots&0_{(L-1)kN}&\cdots&0_{(L-1)kN}&\cdots&0_{(L-1)kN}
\end{pNiceArray}
\end{align*}
However, we should analyze the perturbation of the empirical sequence output $\hat{X}^{(1)}$ for the next layer analysis. 
Note that $X^{(1)} = X^{(0)} + W_{OV}^{(1)}X^{(0)}\S^{(1)}_{\mathrm{ideal}}$, where $\S^{(1)}_{\mathrm{ideal}}$ is the ideal one-hot softmax attention pattern. The empirical output $\hat{X}^{(1)}$ has some error introduced by the non-saturation of the softmax
\begin{align*}
    \hat{X}^{(1)} = X^{(0)} + W_{OV}^{(1)}X^{(0)}\S^{(1)} = X^{(1)} + W_{OV}^{(1)}X^{(0)}\underbrace{(\S^{(1)}-\S^{(1)}_{\mathrm{ideal}})}_{\Delta \S^{(1)}}.
\end{align*}
By \Cref{lemma: empirical gradient of the first stage}, the correct entry of the softmax probability vector $\S^{(1)}$ is greater than $1-\frac{1}{2}\epsilon$ for all indices $(i,j)$. And other probabilities has $\epsilon$ error in total, and they are all positive. Note that $\|X^{(0)}\|_\infty\le 1$. As a result, the error of the input can be bounded as
$$\|X^{(1)}-\hat{X}^{(1)}\|_\infty  = \max_{s,(i,j)}\qty|(W_{OV}^{(1)}X^{(0)})_s^\top \Delta\S^{(1)}_{(i,j)}|\le \|X^{(0)}\|_\infty\cdot 2\cdot\frac{\epsilon}{2} \le \epsilon.$$
So we conclude the proof.
\end{proof}

At the end of Stage 1, we further train the readout layer with one gradient step to output the correct 1-hop for each position. 
\begin{lemma}\label{lemma: readout for first layer}
    Under the conditions of \Cref{lemma: empirical gradient of the first stage}, after one gradient step on $\Psi_1$ we have
    \begin{align*}
    \sup_{\sigma,(i,j)}\norm{\S(\Psi_1^\top f^{(1)}(X)_{(i,j)})-e_{\hop^1_i(j)}}_\infty\le \epsilon.
\end{align*}
\end{lemma}
\begin{proof}
We calculate the population gradient for $\Psi_1$, and then do the finite sample analysis.

Recall the population gradient of the $\ell$th readout layer $\Psi_\ell$: 
\begin{align*}
    \nabla_{\Psi_{\ell}} \mathcal{L}^{(\ell)}_{\mathcal{D}}(\theta)=-\E\qty[\qty(e_{N,\hop^{2^{\ell-1}}_i(j)}-\S(\Psi_\ell^\top f^{(\ell)}(X^{(\ell-1)})_{(i,j)})) (f^{(\ell)}(X^{(\ell-1)})_{(i,j)})^\top]
\end{align*}
For $1$-hop, $\ell = 1$ and we have the population gradient
\begin{align*}
    \nabla_{\Psi_{1}} \mathcal{L}^{(1)}_{\mathcal{D}}(\theta) =-\E\qty[\qty(e_{N,\sigma_i\pi_i(j)}-\S(\Psi_1^\top f^{(1)}(X)_{(i,j)})) (f^{(1)}(X)_{(i,j)})^\top]
\end{align*}
By \Cref{lemma: empirical gradient of the first stage}, the output of the first layer would be
\begin{align*}
    f^{(1)}(X)_{(i,j)} &= W_{OV}^{(1)}X^{(0)}\S^{(1)} = W_{OV}^{(1)}X^{(0)}\S^{(1)}_{\mathrm{ideal}}+W_{OV}^{(1)}X^{(0)}\Delta\S^{(1)}\\
    &= e_{L+2,3}\otimes e_{k,i}\otimes e_{N,\sigma_i\pi_i(j)} + W_{OV}^{(1)}X^{(0)}\Delta\S^{(1)}.
\end{align*}
where $\|W_{OV}^{(1)}X^{(0)}\Delta\S^{(1)}\|_\infty\le {\epsilon}.$ Since $\Psi_\ell(0)=\beta_0 e_{L+2,\ell+2}\otimes \mathbf{1}_{k}\otimes I_{N\times N}$, we have the expansion
\begin{align*}
    \S(\Psi_1^\top f^{(1)}(X)_{(i,j)}) &= \S(\beta_0e_{N,\sigma_i\pi_i(j)}+\Psi_1^\top W_{OV}^{(1)}X^{(0)}\Delta\S^{(1)})\\
    &= \S(\beta_0e_{N,\sigma_i\pi_i(j)}) + \underbrace{\tilde{J}\Psi_1^\top W_{OV}^{(1)}X^{(0)}\Delta\S^{(1)}}_{\Delta_1}.
\end{align*}
Since $\|W_{OV}^{(1)}X^{(0)}\Delta\S^{(1)}\|_\infty\le {\epsilon}$, we have $\norm{\Delta_1}_\infty\le \beta_0\epsilon.$ The signal term
$$\S(\beta_0e_{N,\sigma_i\pi_i(j)}) = \frac{\exp(\beta_0)-1}{\exp(\beta_0)+N-1}e_{N,\sigma_i\pi_i(j)}+\frac{1}{\exp(\beta_0)+N-1}\1_N.$$
The population gradient is thus
\begin{align*}
    \nabla_{\Psi_{1}} \mathcal{L}^{(1)}_{\mathcal{D}}(\theta) &=-\E\qty[\qty(e_{N,\sigma_i\pi_i(j)}-\S(\Psi_1^\top f^{(1)}(X)_{(i,j)})) (f^{(1)}(X)_{(i,j)})^\top]\\
    &=-\E\qty[\qty(\frac{N}{\exp(\beta_0)+N-1}e_{N,\sigma_i\pi_i(j)}-\frac{1}{\exp(\beta_0)+N-1}\1_N-\Delta_1)(f^{(1)}(X)_{(i,j)})^\top]\\
    &=-\frac{1}{(\exp(\beta_0)+N-1)k}\qty(e_{L+2,3}\otimes \1_k \otimes (I_N-\frac{1}{N}\1_N\1_N^\top)) + O(\epsilon).
\end{align*}
where $O(\epsilon)$ denotes the terms with infinity norm smaller than $\epsilon$ with $\epsilon\lesssim \frac{1}{k^2N^3}$. 

Then we analyze the finite sample error. For any sample $X_m$, the upper bound of the empirical gradient for each sample
$$\nabla_{\Psi_{1}} \mathcal{L}^{(1)}(X_m) =-\qty[\qty(e_{N,\sigma_{i_m}\pi_{i_m}(j_m)}-\S(\Psi_1^\top f^{(1)}(X_m)_{(i_m,j_m)})) (f^{(1)}(X_m)_{(i_m,j_m)})^\top]$$
has infinity norm upper bounded by 1. Apply Hoeffding inequalities with $M\gtrsim k^2N^4d^2\log^2\frac{d}{\delta}$, the empirical gradient has noise upper bounded by 
$$\norm{\frac{1}{M}\sum_m\nabla_{\Psi_{1}} \hat{\mathcal{L}}^{(1)}(X_m)-\nabla_{\Psi_{1}} {\mathcal{L}}^{(1)}(X_m)}_\infty\le \frac{\log(\frac{d}{\delta})}{\sqrt{M}}\lesssim \frac{1}{k^2N^3}.$$
Therefore, the error altogether is upper bounded by $O(\frac{1}{k^2N^3})$ in infinity norm.

After one step of gradient, we have the softmax score ($\Delta$ is the error term with $\norm{\Delta}_\infty\le \epsilon$.)
\begin{align*}
    \S(\Psi_1(1)^\top f^{(1)}(X)_{(i,j)})&=\S\qty(\frac{\eta}{(\exp(\beta_0)+N-1)k}(I_N-\frac{1}{N}\1_N\1_N^\top)e_{N,\sigma_i\pi_i(j)}+\beta_0e_{N,\sigma_i\pi_i(j)}+\eta \Delta)
\end{align*}
The separation between the $\sigma_i\pi_i(j)$-th entry and the others are lower bounded by:
$$\frac{\eta}{(\exp(\beta_0)+N-1)k} \frac{N-1}{N}-\eta \|\Delta\|_\infty\gtrsim \frac{\eta}{kN}.$$
By $\eta \gtrsim {k^2N^3\log \frac{kN}{\epsilon}}$, we have $\S(\Psi_1^\top f^{(1)}(X)_{(i,j)})_{\sigma_i\pi_i(j)}\ge 1-\epsilon$ and thus
\begin{align*}
    \sup_{\sigma,(i,j)}\norm{\S(\Psi_1^\top f^{(1)}(X)_{(i,j)})-e_{\hop^1_i(j)}}_\infty\le \epsilon.
\end{align*}
\end{proof}
\subsection{Stage 2: Learning the 2-hop}
Now we start to analyze the second stage GD with the input $\hat{X}^{(1)}$. 
\begin{lemma}[Empirical gradient of $W^{(2)}_{KQ}$ (Stage 2)]
\label{lemma: empirical gradient of the second stage}
    Assume that $W^{(\ell)}_{KQ}(1)=0_{d\times d}$ for all layers $\ell\ge 1$, and that $\|X^{(1)}-\hat{X}^{(1)}\|_\infty \le \epsilon$ before the second stage. After running one-step of gradient descent on the second stage finite sample loss $\mathcal{L}^{(2)}$ with $M$ training sequences and learning rate $\eta$, satisfying $\beta_0\le 1,$  
    $M\gtrsim k^2N^4d^2\log^2\frac{d}{\delta},\eta \gtrsim \frac{k^2N^3\log \frac{kN}{\epsilon}}{\beta_0}$ for any $\epsilon \in(0, \frac{1}{k^2N^3\log^2k})$, then with probability $1-\delta$, for any $(i,j)\in[k]\times [N]$, we have that after the second stage $$\S^{(2)}_{(i+1,\hop^1_i(j))}:=\S\qty((\hat{X}^{(1)})^\top W_{KQ}^{(2)}\hat{X}^{(1)}_{(i,j)})_{(i+1,\hop_i^1(j))} \ge 1-\frac{1}{2}\epsilon.$$
    Furthermore, if we pick $\eta \gtrsim \frac{Ck^2N^3\log k\log \frac{kN}{\epsilon}}{\beta_0}$, we have that after the second stage
    $\S^{(2)}_{(i+1,\hop_i^1(j))}\ge 1-\frac{\epsilon}{2(kNL)^{CL}}$ for any absolute constant $C$.
\end{lemma}
\begin{proof} We follow the strategy in stage 1, first computing the population gradient with the ideal input, then doing the finite sample analysis and finally controlling the perturbation of the input and the sample noise. We also ignore the superscript of $X^{(1)}/\hat{X}^{(1)}$ when it is clear from context in this subsection.

The population gradient of $W^{(2)}_{KQ}$ with ideal input sequence $X^{(1)}$ is
\begin{align*}
    \nabla \mathcal{L}_{\mathcal{D}}^{(2)}=-\E\qty[\sum_{s'\in[N]}\qty(\mathbf{1}\{s'=\hop^{2}_i(j)\}-(\mathcal{P}^{(2)}_{(i,j)})_{s'}) {X}J^{(2)} {X}^{\top} {(\Psi_2^\top W_{OV}^{(2)})_{s'}}^\top (X_{(i,j)})^{\top}]
\end{align*}

Since the normalization terms cancel out as in stage 1, the ideal gradient is the signal term:
\begin{align*}
    \nabla \mathcal{L}_{\mathcal{D}}^{(2)} &= -\E\qty[\sum_{s'\in[N]}\mathbf{1}\{s'=\sigma_{i+1}\pi_{i+1}\sigma_i\pi_i(v)\} XJX^\top {(\Psi_i^\top W_{OV}^{(i)})_{s'}}^\top (X^{(1)}_{(i,j)})^{\top}]\\
    &= -\frac{1}{kN}\E\qty[X\qty(I_{kN\times kN}-\frac{1}{kN}\mathbf{1}\mathbf{1}^\top) X^\top{(\Psi_2^\top W_{OV}^{(2)})_{\sigma_{i+1}\pi_{i+1}\sigma_i\pi_i(j)}}^\top (X^{(1)}_{(i,j)})^{\top}].
\end{align*}
We can also calculate the vector $({X}^{(1)})^{\top} {(\Psi_2^\top W_{OV}^{(2)})_{\sigma_{i+1}\pi_{i+1}\sigma_i\pi_i(j)}}^\top$:
\begin{align*}
\begin{pNiceArray}{c|c|c|c}
e_{N,(\sigma_1\pi_1)^{-1}\sigma_{i+1}\pi_{i+1}\sigma_i\pi_i(j)}^\top&e_{N,(\sigma_2\pi_2)^{-1}\sigma_{i+1}\pi_{i+1}\sigma_i\pi_i(j)}^\top&\cdots&e_{N,(\sigma_k\pi_k)^{-1}\sigma_{i+1}\pi_{i+1}\sigma_i\pi_i(j)}^\top
\end{pNiceArray}^\top
\end{align*}
We further compute the gradient:
\begin{align*}
    \nabla \mathcal{L}_{\mathcal{D}}^{(2)}
     &=-\frac{{1}}{kN}\E\qty[({X}^{(1)})^{\top} {(\Psi_2^\top W_{OV}^{(2)})_{\sigma_{i+1}\pi_{i+1}\sigma_i\pi_i(v)}}^\top - \frac
    1N \mathbf{1}) (X_{(i,j)})^{\top}]\\
    &=-\frac{{\beta_0}}{kN}\E\qty[\qty(\begin{bmatrix}
        \sum_{j=1}^k e_{k,j}\otimes e_{N,(\sigma_j\pi_j)^{-1}\sigma_{i+1}\pi_{i+1}\sigma_i\pi_i(j)}\\
        \sum_{j=1}^k e_{k,j}\otimes e_{N,\sigma_j(\sigma_j\pi_j)^{-1}\sigma_{i+1}\pi_{i+1}\sigma_i\pi_i(j)}\\
        \sum_{j=1}^k e_{k,j}\otimes \sigma_{i+1}\pi_{i+1}\sigma_{i}\pi_i(j)\\
        \1_{kN(L-1)}
    \end{bmatrix}-\frac1N\begin{bmatrix}
        \mathbf{1}_{kN}\\
        \mathbf{1}_{kN}\\
        \mathbf{1}_{kN}\\
        \1_{kN(L-1)}
    \end{bmatrix}) \begin{bmatrix}
        e_{k,i}\otimes e_{N,j}\\
        e_{k,i}\otimes e_{N,\sigma_i(j)}\\
        e_{k,i}\otimes e_{{N,\sigma_i\pi_i(j)}}\\
        0_{kN(L-1)}
    \end{bmatrix}^\top]\\
    &:=-\frac{\beta_0}{kN} \begin{bmatrix}
        A_1&A_2&A_3&0\\
        A_4&A_5&A_6&0\\
        A_7&A_8&A_9&0\\
        0&0&0&0
\end{bmatrix}\tag{$A_i\in\R^{kN\times kN},i=1,2,...,9$.}
\end{align*}
Similar to stage 1, since we have independent $\sigma_{i+1}$ in the second and third row of block matrices, the expectation becomes 0 for $A_1$ and $A_4$ to $A_9$. 

We focus on the calculation on $A_2, A_3$ and have the following expectation given $(i,j)$ is in $[k]\times[N]$:
$$\E\qty[A_3|i]=\frac1N\sum_{j=1}^N (e_{k,i+1}\otimes e_{N,j}) (e_{k,i}\otimes e_{N,j})^\top-\frac{1}{N^2}(e_{k,i+1}\otimes \mathbf{1}_{N})(e_{k,i}\otimes \mathbf
{1}_{N})^\top.$$
$$\E\qty[A_2|i]=\frac{f(\pi_i)-1}{N(N-1)}\sum_{j=1}^N (e_{k,i+1}\otimes e_{N,j}) (e_{k,i}\otimes e_{N,j})^\top-\frac{f(\pi_i)-1}{N^2(N-1)}(e_{k,i+1}\otimes \mathbf{1}_{N})(e_{k,i}\otimes \mathbf
{1}_{N})^\top.$$
This population gradient can correctly learn the functionality of the second layer, and the signal can dominate the sum of the gradient noise and accumulated error. We will show this point after the noise/error analysis.

Now consider the empirical estimate of the gradient. Define the matrix
$$g_{2} = \E\begin{bmatrix}
        A_1&A_2&A_3&0\\
        A_4&A_5&A_6&0\\
        A_7&A_8&A_9&0\\
        0&0&0&0
\end{bmatrix}$$
The empirical gradient can be written as
$\nabla \hat{\mathcal{L}}^{(2)} = -\frac{\beta_0}{kN}\hat{g}_2,$
where $\hat{g}_2$ is the empirical estimate of $g_2$ with perturbed inputs:
$$\hat{g}_2 = \frac{kN}{\beta_0M}\sum_{m=1}^M \sum_{s'\in[N]}(\mathbf{1}\{s'=\hop^{2}_{i_m}(j_m)\} -\mathcal{P}^{(2)}_{s'})\hat{X}^{(1)}_mJ\nabla_{W_{KQ}^{(\ell)}}\mathcal{L}^{(2)} (\hat{X}^{(1)}_m)^\top {\qty(\Psi_2^\top W_{OV}^{(2)})_{s'}}^\top (\hat{X}^{(1)}_{(i,j)})^{\top}$$
It suffices to show that $\|\hat{g}_2-g_2\|_\infty$ is small. It is similar to the first stage analysis for the sample noise part, with the perturbed input error introduced:
\begin{lemma}
\label{lemma: concentration for stage 2}
    Suppose $\|X_m^{(1)}-\hat{X}_m^{(1)}\|_\infty\le \epsilon$ for all $m$. For any $\delta>0$, we have with probability $1-\delta$ s.t.
    $$\|\hat{g}_2-g_2\|_\infty\lesssim \frac{\log(kN/\delta)}{\sqrt{M}} +d\epsilon$$
\end{lemma}
\begin{proof}
First, the empirical gradient of the single ideal sample $X_m$ is
\begin{align*}
    \hat{g}^{\mathrm{ideal}}_{2,m} &= \frac{kN}{\beta_0} \sum_{s'\in[N]}\mathbf{1}\{s'=\hop^{2}_{i_m}(j_m)\} X_mJ X_m^\top {\qty(\Psi_2^\top W_{OV}^{(2)})_{s'}}^\top (X^{(1)}_{(i_m,j_m)})^{\top}\\
    &=X_m (I-\frac{1}{kN}\1\1^\top)X_m^\top (e_{L+2,3}\otimes \1_k\otimes e_{N,\hop^2_{i_m}(j_m)}) (X^{(1)}_{(i_m,j_m)})^{\top}
\end{align*}
Similar to the proof in \Cref{lemma: concentration for stage 1}, the upper bound of each entry of the random variable is
\begin{equation}\label{eq: stage 2, noise analysis}
    \|\hat{g}^{\mathrm{ideal}}_{2,m}\|_\infty 
    \le \left\|X_m (I-\frac{1}{kN}\1\1^\top)X_m^\top (e_{L+2,3}\otimes \1_k\otimes e_{N,\hop^2(v_m)})\right\|_\infty \|(X^{(1)}_{(i_m,j_m)})^{\top}\|_\infty\le 1.
\end{equation}
Then we can concentrate $\hat{g}_2^{\mathrm{ideal}} := \frac{1}{M}\sum_{m=1}^M \hat{g}^{\mathrm{ideal}}_{2,m}$ as:
\begin{align*}
    \norm{\hat{g}_2^{\mathrm{ideal}}-g_2}_\infty &= \norm{\frac{1}{M}\sum_{m=1}^M \hat{g}_{2,m}^{\mathrm{ideal}}-\E\hat{g}^{\mathrm{ideal}}_{2}}\lesssim \frac{\log(d/\delta)}{\sqrt{M}}
\end{align*}
with probability $1-\frac{\delta}{d^2}$. By $d = O(kN\log k)$, we can union bound over all entries of $\hat{g}_2$ and we have the desired first term error.

After considering the ideal empirical gradient, we also need to bound the distance between the empirical gradients with ideal inputs $X_m$ and those with perturbed input sequences $\hat{X}_m$.

Recall that the empirical gradient 
\begin{align*}
    \hat{g}_{2,m} &= \frac{kN}{\beta_0} \sum_{s'\in[N]}\qty(\mathbf{1}\{s'=\hop^{2}_{i_m}(j_m)\}-\mathcal{P}^{(2)}_{s'}) \hat{X}_mJ \hat{X}_m^\top {\qty(\Psi_2^\top W_{OV}^{(2)})_{s'}}^\top (\hat{X}^{(1)}_{(i_m,j_m)})^{\top}\\
    &=\hat{X}_m (I-\frac{1}{kN}\1\1^\top)\hat{X}_m^\top (e_{L+2,3}\otimes \1_k\otimes e_{N,\hop^2_{i_m}(j_m)}) \hat{X}_{(i_{m},j_m)}^\top\\&-\sum_{s'\in[N]}\S_{s'}\hat{X}_m (I-\frac{1}{kN}\1\1^\top)\hat{X}_m^\top (e_{L+2,3}\otimes \1_k\otimes e_{N,s'}) \hat{X}_{(i_{m},j_m)}^\top.
\end{align*}
Denote the following term as
$$\hat{\gamma}_{s',m}=\hat{X}_m (I-\frac{1}{kN}\1\1^\top)\hat{X}_m^\top (e_{L+2,3}\otimes \1_k\otimes e_{N,s'}) \hat{X}_{(i_{m},j_m)}^\top,$$
$${\gamma}_{s',m}={X}_m (I-\frac{1}{kN}\1\1^\top){X}_m^\top (e_{L+2,3}\otimes \1_k\otimes e_{N,s'}) {X}_{(i_{m},j_m)}^\top.$$

Then we can rewrite the empirical gradient into
\begin{align*}
    \hat{g}_{2,m}-\hat{g}^{\mathrm{ideal}}_{2,m} 
    &=\Delta\gamma_{\hop^2_{i_m}(j_m),m}-\sum_{s'\in[N]}\S_{s'}\Delta\gamma_{s',m}.
\end{align*}
The error of the following difference $\|\gamma_{s',m}-\hat{\gamma}_{s',m}\|_\infty\le Cd\epsilon$ with some absolute constant for all possible $s'\in[N]$, since $\norm{\hat{X}_m-X_m}_\infty\le \epsilon$ and $\norm{X_m}_\infty\le 1$. Then, we have
\begin{equation}\label{eq: stage 2, perturbation analysis}
    \norm{\hat{g}_{2,m}-\hat{g}^{\mathrm{ideal}}_{2,m}}_\infty \le \|\Delta\gamma_{\hop^2_{i_m}(j_m),m}\|_\infty+\sum_{s'\in[N]}\S_{s'}\norm{\Delta\gamma_{s',m}}_\infty
    \le2Cd\epsilon.
\end{equation}
Combine both \eqref{eq: stage 2, noise analysis} and \eqref{eq: stage 2, perturbation analysis}, we finished the proof.
\end{proof}

After the one-step gradient, we get $W_{KQ}^{(2)}=\eta \hat{g}_2$.
Given a sample sequence $\hat{X}^{(1)}_m$ and index $(i_m,j_m)$, we have the attention score
\begin{align*}
    \eta (\hat{X}^{(1)}_m)^\top \hat{g}_2 \hat{X}^{(1)}_{(i_m,j_m)} ={}& \frac{\beta_0\eta}{kN}\qty[(X_m^{(1)})^\top {g}_2 X^{(1)}_{(i_m,j_m)} + (X_m^{(1)})^\top \underbrace{(\hat{g}_2-g_2)}_{\Delta g_2} \hat{X}^{(1)}_{(i_m,j_m)}]\\
    &+\frac{\beta_0\eta}{kN}\underbrace{\qty[(\hat{X}^{(1)}_m)^\top \hat{g}_2 \hat{X}^{(1)}_{(i_m,j_m)}-X_m^\top \hat{g}_2 X_{(i_m,j_m)}]}_{\text{Perturbation error}}
\end{align*}
Note that the population attention score is
\begin{align*}
    X_m^\top {g}_2 X_{(i_m,j_m)}&=X_m^\top\E\begin{bmatrix}
        A_1&A_2&A_3&0\\
        A_4&A_5&A_6&0\\
        A_7&A_8&A_9&0\\
        0&0&0&0
\end{bmatrix}\begin{bmatrix}
        e_{k,i_m}\otimes e_{N,j_m}\\
        e_{k,i_m}\otimes e_{N,\sigma_{i_m}(j_m)}\\
        e_{k,i_m}\otimes e_{{N,\sigma_{i_m}\pi_{i_m}(j_m)}}\\
        0_{kN(L-1)}
    \end{bmatrix}=X_m^\top \begin{bmatrix}
        (*)\\
        0_{kN(L+1)}\\
    \end{bmatrix}
\end{align*}
where the term $(*)$ is
\begin{align*}
    \frac{1}{kN}\qty(e_{k,i_m+1}\otimes e_{N,\sigma_{i_m}\pi_{i_m}(j_m)}-\frac{1}{N}e_{k,i_m+1}\otimes \1_N) + \frac{f(\pi_i)-1}{kN(N-1)}\qty(e_{k,i_m+1}\otimes e_{N,\sigma_{i_m}(j_m)}-\frac{1}{N}e_{k,i_m+1}\otimes \1_N)
\end{align*}

For simplicity, we ignore the superscript for the following proof in this subsection. Define the population/empirical separation between the correct position $p_2:=Ni_m+\sigma_{i_m}\pi_{i_m}(j_m)$, i.e. $(i_m+1,\sigma_{i_m}\pi_{i_m}(j_m))$ position, and the others as follows:
$$\Delta^{(2)}_{i_m,j_m} = \qty(X_m^\top {g}_2 X_{(i_m,j_m)})_{p_2}-\max_{j\not=p_2}\qty(X_m^\top {g}_2 X_{(i_m,j_m)})_j.$$
$$\hat{\Delta}^{(2),\mathrm{ideal}}_{i_m,j_m} = \qty(X_m^\top \hat{g}_2 X_{(i_m,j_m)})_{p_2}-\max_{j\not=p_2}\qty(X_m^\top \hat{g}_2 X_{(i_m,j_m)})_j.$$
$$\hat{\Delta}^{(2)}_{i_m,j_m} = \qty(\hat{X}_m^\top \hat{g}_2 \hat{X}_{(i_m,j_m)})_{p_2}-\max_{j\not=p_2}\qty(\hat{X}_m^\top \hat{g}_2 \hat{X}_{(i_m,j_m)})_j.$$
If $\pi_{i_m}(v)=v$ for all $v\in [N]$, i.e. $f(\pi_{i_m})=N$, we have the $N(i-1)+j$-th entry of the pre-softmax attention score
\begin{align*}
    \qty(X_m^\top {g}_2 X_{(i_m,j_m)})_{(i,j)} = \begin{cases}
        \frac{2N-2}{kN^2},&i=i_m+1, j=\sigma_{i_m}\pi_{i_m}(j_m)\\
        -\frac{2}{kN^2},&i=i_m+1,\ j\not=\sigma_{i_m}\pi_{i_m}(j_m)\\
        0,&i\not=i_m
    \end{cases}
\end{align*}
So $\Delta_{i_m,j_m}^{(2)}\ge \frac{1}{kN^2}$ by $N\ge 2$.

If not all $v$ satisfies $\pi_i(v)=v$, $f(\pi_i)\le N-2.$ We have the $N(i-1)+j$-th entry of the pre-softmax attention score
\begin{align*}
    \qty(X_m^\top {g}_2 X_{(i_m,j_m)})_{(i,j)} = \begin{cases}
        \frac{N-1}{kN^2},&i=i_m+1, j=\sigma_{i_m}\pi_{i_m}(j_m)\\
        \frac{f(\pi_{i_m})-1}{kN^2},&i=i_m+1, j=\sigma_{i_m}(j_m)\\
        -\frac{1}{kN^2}-\frac{f(\pi_{i_m})-1}{kN^2(N-1)},&i=i_m+1,\ j\not=\sigma_{i_m}\pi_{i_m}(j_m),\sigma_{i_m}(j_m)\\
        0,&i\not=i_m+1
    \end{cases}
\end{align*}
So $\Delta_{i_m,j_m}^{(2)}\ge \frac{1}{kN^2}$ by $N\ge 2$ for both cases. Therefore, we have the expected signal at least $\frac{1}{kN^2}$.

Now we bound the noise introduced by $\Delta g_2$. Since $\|\Delta g_2\|_\infty \lesssim \frac{\log \frac{d}{\delta}}{\sqrt{M}}$, we have
$$\left\|X_m^\top {\Delta g_2} X_{(i_m,j_m)}\right\|_\infty \lesssim \frac{d\log \frac{d}{\delta}}{\sqrt{M}}\le \frac{1}{2kN^2}$$
by $M\gtrsim k^2N^4d^2\log^2\frac{d}{\delta}$. Therefore, the noise term can be bounded and the empirical separation with the ideal input sequence $\hat{\Delta}_{i_m,j_m}^{(2),\mathrm{ideal}}\ge \frac{1}{2kN^2}$. This also gives the upper bound for $\|\hat{g}_2\|_\infty\le O(1/kN^2)$.

Finally, we add up the perturbation error for the pre-softmax attention score. However, the first layer parameter is also updated by a very small amount in the second stage. So we need to upper bound the perturbation again for $\hat{X}^{(1)}$ for the second stage conclusion.
The following lemma shows that after the second gradient step, the first-layer parameter is almost unperturbed and the softmax score is still close to one-hot.
\begin{lemma}[$W_{KQ}^{(1)}$ is unchanged] Under the condition of \Cref{lemma: empirical gradient of the second stage}, after the second step we still have
$$\S^{(1)}_{(i,\pi_i(j))}:=\S(X^{(0)}W_{KQ}^{(1)}X^{(0)}_{(i,j)})_{(i,\pi_i(j))} \ge 1-\frac{1}{2}\epsilon.$$
\end{lemma}
\begin{proof}
    By \Cref{lemma: gradient upper bound for previous layers with curriculum}, the gradient norm for the first layer is upper bounded by $O(\beta_0k^{5/2}N^{3/2}L^{3/2}\epsilon)$. With the gradient update in the second stage, $W_{KQ}^{(1)}$ will only be perturbed by
    $$\norm{\eta \nabla_{W_{KQ}^{(1)}} \mathcal{L}^{(2)}}\lesssim (kN)^{4.5}\log\frac{kN}{\epsilon}\log^{2.5} k \cdot \epsilon \ll N\log\frac{kN}{\epsilon}.$$
    where $N\log\frac{kN}{\epsilon}$ is the scale of each entry of $W_{KQ}^{(1)}$. Therefore, the previous bound on the attention score in \Cref{lemma: empirical gradient of the first stage} still holds.
\end{proof}

This implies that $\norm{\hat{X}_m - X_m}_\infty \le \epsilon$ still holds using the same perturbation argument in stage 1. Therefore, we have the perturbation error to the pre-softmax attention score
\begin{align*}
    &\norm{(\hat{X}^{(1)}_m)^\top \hat{g}_2 \hat{X}^{(1)}_{(i_m,j_m)}-X_m^\top \hat{g}_2 X_{(i_m,j_m)}}_\infty\\
    \le{}&\norm{(\hat{X}^{(1)}_m)^\top \hat{g}_2 (\hat{X}^{(1)}_{(i_m,j_m)}-X_{(i_m,j_m)})}_\infty+\norm{(\hat{X}^{(1)}_m-X_m)^\top \hat{g}_2 X_{(i_m,j_m)}}_\infty\\
    \le{}&\norm{\hat{g}_2}_2\norm{(\hat{X}^{(1)}_m)^\top (\hat{X}^{(1)}_{(i_m,j_m)}-X_{(i_m,j_m)})}_\infty+\norm{\hat{g}_2}_2\norm{(\hat{X}^{(1)}_m-X_m)^\top X_{(i_m,j_m)}}_\infty\\
    \le{}&\norm{\hat{g}_2}_F\norm{(\hat{X}^{(1)}_m)^\top (\hat{X}^{(1)}_{(i_m,j_m)}-X_{(i_m,j_m)})}_\infty+\norm{\hat{g}_2}_F\norm{(\hat{X}^{(1)}_m-X_m)^\top X_{(i_m,j_m)}}_\infty\\
    \lesssim{}& \sqrt{\frac{1}{(kN^2)^2}\cdot d^2}\cdot d\epsilon.
\end{align*}
Since $\epsilon\le O(\frac{1}{d^2})=O(\frac{1}{k^2N^3L^2})$, the perturbation error is upper bounded by $O(\frac{1}{kN^2})$. Therefore, the empirical separation $\hat{\Delta}^{(2)}_{i_m,j_m}\ge \Omega(\frac{1}{kN^2})$.
After one step gradient with learning rate $\eta \gtrsim \frac{k^2N^3}{\beta_0}\log\frac{kN}{\epsilon}$, the softmax output of the correct position can be lower bounded by
\begin{align*}
    \S(\hat{X}_m^\top W_{KQ}^{(2)} \hat{X}_{(i_m,j_m)})_{i_m,\hop^{1}(X_{(i_m,j_m)})} \ge \frac{\exp\qty(\frac{\eta\beta_0}{kN}\cdot \frac{1}{kN^2})}{\exp\qty(\frac{\eta\beta_0}{kN}\cdot \frac{1}{kN^2})+kN-1}\ge 1-\frac{1}{2}\epsilon.
\end{align*}
Therefore, we finish the proof.
\end{proof}



\paragraph{Perturbation analysis of Stage 2} The lemma presents the perturbation analysis for stage 2.
\begin{lemma}[Perturbation analysis of stage 2]\label{lemma: perturbation for second stage}
    Under the condition of \Cref{lemma: empirical gradient of the second stage}, we have
    $$\|X^{(2)}-\hat{X}^{(2)}\|\le 2\epsilon,$$
    where $X^{(2)}$ is the ideal output with saturate softmax, and $\hat{X}^{(2)}$ is the transformer output.
\end{lemma}
\begin{proof}
Similar to stage 1, the ideal output for the second layer is $$X^{(2)} = X^{(1)} + W_{KQ}^{(2)}X^{(1)}\S^{(2)}_{\mathrm{ideal}},$$ where $\S^{(2)}_{\mathrm{ideal}}$ is the ideal one-hot softmax attention pattern. The empirical output $\hat{X}^{(2)}$ has the error introduced by the non-saturation of the softmax, together with the previous error in $\hat{X}^{(1)}$:
\begin{align*}
    \hat{X}^{(2)} &= \hat{X}^{(1)} + W_{OV}^{(2)}\hat{X}^{(1)}\S^{(2)}\\ &= X^{(2)} + W_{OV}^{(2)}X^{(1)}\underbrace{(\S^{(2)}-\S^{(2)}_{\mathrm{ideal}})}_{\Delta \S^{(2)}, \text{ Non-saturation error}}+\underbrace{(\hat{X}^{(1)}-{X}^{(1)})+W_{OV}^{(2)}(\hat{X}^{(1)}-{X}^{(1)})\S^{(2)}}_{\text{Accumulated perturbation error}}.
\end{align*}
We first bound the non-saturation error term. By \Cref{lemma: empirical gradient of the second stage}, the correct entry of the softmax probability vector $\S^{(2)}$ is greater than $1-\frac{1}{2}\epsilon$ for all index $v$. And other probabilities have $\epsilon$ error in total, and they are all positive. Since $\|X^{(1)}\|_\infty\le 1$, the error of the non-saturation error can be bounded as
$$\|W_{OV}^{(2)}X^{(1)}(\S^{(2)}-\S^{(2)}_{\mathrm{ideal}})\|_\infty  = \max_{s,(i,j)}\qty|(W_{OV}^{(2)}X^{(1)})_s^\top \Delta\S^{(2)}_{(i,j)}|\le \|X^{(1)}\|_\infty\cdot 2\cdot\frac{\epsilon}{2} \le \epsilon.$$
Now consider the accumulated perturbation error. By the perturbation analysis in stage 1, we have $\norm{\hat{X}^{(1)}-{X}^{(1)}}_{\infty}\le \epsilon.$ Note that the error in $\hat{X}^{(1)}-{X}^{(1)}$ are in different rows of the matrices from 
$W_{OV}^{(2)}(\hat{X}^{(1)}-{X}^{(1)})\S^{(2)}$. Therefore, $\hat{X}^{(1)}-{X}^{(1)}$ won't introduce extra error in this stage, and we only need to consider $W_{OV}^{(2)}(\hat{X}^{(1)}-{X}^{(1)})\S^{(2)}$:
\begin{align*}
    \norm{W_{OV}^{(2)}(\hat{X}^{(1)}-{X}^{(1)})\S^{(2)}}_\infty &= \max_{s,(i,j)}\qty|(W_{OV}^{(2)}(\hat{X}^{(1)}-{X}^{(1)}))_s^\top \S^{(2)}_{(i,j)}|\\
    &=\max_{s,{(i,j)}}\qty|\sum_{p=1}^{kN}(W_{OV}^{(2)}(\hat{X}^{(1)}-{X}^{(1)}))_{s,p} (\S^{(2)}_{(i,j)})_p|\\
    &\le \|\hat{X}^{(1)}-{X}^{(1)}\|_\infty\tag{Since $\sum_p (\S^{(2)}_{(i,j)})_p=1, (\S^{(2)}_{(i,j)})_p\ge0$.}
\end{align*}
Combine both parts of error, we have $\norm{\hat{X}^{(2)}-{X}^{(2)}}_\infty \le 2\epsilon.$
\end{proof}

In later stages, we use a similar argument and inductively show that the perturbation error grows with the depth $\ell$. At the end of stage 2, we further train the readout layer with one gradient step to output the correct $2$-hop for each position.
\begin{lemma}\label{lemma: readout for second layer}
    Under the conditions of \Cref{lemma: empirical gradient of the second stage}, after one gradient step on $\Psi_2$ we have
    \begin{align*}
    \sup_{\sigma,(i,j)}\norm{\S(\Psi_2^\top f^{(2)}(\hat{X}^{(1)})_{(i,j)})-e_{\hop^2_i(j)}}_\infty\le \epsilon.
\end{align*}
\end{lemma}
\begin{proof}
We calculate the population gradient for $\Psi_2$, and then do the finite sample analysis.

For $2$-hop, the population gradient is
\begin{align*}
    \nabla_{\Psi_{2}} \mathcal{L}_{\mathcal{D}}^{(2)}(\theta) =-\E\qty[\qty(e_{N,\hop^2_i(j)}-\S(\Psi_2^\top f^{(2)}(X)_{(i,j)})) (f^{(2)}(X)_{(i,j)})^\top]
\end{align*}
By \Cref{lemma: empirical gradient of the second stage}, the output of the second layer would be
\begin{align*}
    f^{(2)}(X)_{(i,j)} &= W_{OV}^{(2)}\hat{X}^{(1)}\S^{(2)} = W_{OV}^{(2)}X^{(1)}\S^{(2)}_{\mathrm{ideal}}+W_{OV}^{(2)}(\hat{X}^{(1)}-X^{(1)})\S^{(2)}_{\mathrm{ideal}}+W_{OV}^{(2)}X^{(1)}\Delta\S^{(2)}\\
    &= e_{L+2,4}\otimes e_{k,i}\otimes e_{N,\hop^2_i(j)} + \Delta_2.
\end{align*}
where $\|\Delta_2\|_\infty\le 2{\epsilon}$ by \Cref{lemma: perturbation for second stage}. Since $\Psi_\ell(0)=\beta_0 e_{L+2,\ell+2}\otimes \mathbf{1}_{k}\otimes I_{N\times N}$, we have the expansion
\begin{align*}
    \S(\Psi_2^\top f^{(2)}(X)_{(i,j)}) = \S(\beta_0e_{N,\hop_i^2(j)}+\Psi_2^\top \Delta_2)= \S(\beta_0e_{N,\hop^2_i(j)}) + \tilde{J}\Psi_2^\top \Delta_2.
\end{align*}
Since $\|\Delta_2\|_\infty\le {\epsilon}$, we have $\norm{\tilde{J}\Psi_2^\top\Delta_2}_\infty\le \beta_0\epsilon.$ The signal term
$$\S(\beta_0e_{N,\hop^2_i(j)}) = \frac{\exp(\beta_0)-1}{\exp(\beta_0)+N-1}e_{N,\hop^2_i(j)}+\frac{1}{\exp(\beta_0)+N-1}\1_N.$$
The population gradient is thus
\begin{align*}
    \nabla_{\Psi_{2}} \mathcal{L}_{\mathcal{D}}^{(2)}(\theta) 
    &=-\E\qty[\qty(\frac{N}{\exp(\beta_0)+N-1}e_{N,\hop^2_i(j)}-\frac{1}{\exp(\beta_0)+N-1}\1_N-\tilde{J}\Psi_2^\top\Delta_2)(f^{(2)}(X)_{(i,j)})^\top]\\
    &=-\frac{1}{(\exp(\beta_0)+N-1)k}\qty(e_{L+2,4}\otimes \1_k \otimes (I_N-\frac{1}{N}\1_N\1_N^\top)) + O(\epsilon).
\end{align*}
where $O(\epsilon)$ denotes the terms with infinity norm smaller than $\epsilon$ with $\epsilon\lesssim \frac{1}{k^2N^3L^2}$. 

Then we analyze the finite sample error. For any sample $X_m$, the upper bound of the empirical gradient for each sample
$$\nabla_{\Psi_{2}} \mathcal{L}^{(2)}(X_m) =-\qty[\qty(e_{N,\hop^2_{i_m}(j_m)}-\S(\Psi_2^\top f^{(2)}(X_m)_{(i_m,j_m)})) (f^{(2)}(X_m)_{(i_m,j_m)})^\top]$$
has infinity norm upper bounded by 1. Apply Hoeffding inequalities with $M\gtrsim k^2N^4d^2\log^2\frac{d}{\delta}$, the empirical gradient has noise upper bounded by 
$$\norm{\frac{1}{M}\sum_m\nabla_{\Psi_{2}} \hat{\mathcal{L}}^{(2)}(X_m)-\nabla_{\Psi_{2}} {\mathcal{L}}^{(2)}(X_m)}_\infty\le \frac{\log(\frac{d}{\delta})}{\sqrt{M}}\lesssim \frac{1}{k^2N^3}.$$
Therefore, the error altogether is upper bounded by $O(\frac{1}{k^2N^3})$ in infinity norm.

After one step of gradient, we have the softmax score ($\Delta$ is the error term with $\norm{\Delta}_\infty\le \epsilon$.)
\begin{align*}
    \S(\Psi_2(1)^\top f^{(2)}(X)_{(i,j)}))&=\S\qty(\frac{\eta}{(\exp(\beta_0)+N-1)k}(I_N-\frac{1}{N}\1_N\1_N^\top)e_{N,\hop^2_i(j)}+\beta_0e_{N,\hop^2_i(j)}+\eta \Delta)
\end{align*}
The separation between the $\hop^2_i(j)$-th entry and the others are lower bounded by:
$$\frac{\eta}{(\exp(\beta_0)+N-1)k} \frac{N-1}{N}-\eta \|\Delta\|_\infty\gtrsim \frac{\eta}{kN}.$$
By $\eta \gtrsim {k^2N^3\log \frac{kN}{\epsilon}}$, we have $\S(\Psi_2^\top f^{(2)}(X)_{(i,j)})_{\hop^2_i(j)}\ge 1-\epsilon$ and thus
\begin{align*}
    \sup_{\sigma,(i,j)}\norm{\S(\Psi_2^\top f^{(2)}(X)_{(i,j)})-e_{\hop^2_i(j)}}_\infty\le \epsilon.
\end{align*}
\end{proof}

\subsection{Stage $\ell$ ($3\le \ell \le 1+\log_2k$): Learning the $2^{\ell-1}$-hop}
Similar to the second stage, we can learn the $\ell$th layer by one step of gradient descent, shown in the following lemma.
\begin{lemma}[Empirical gradient of $W_{KQ^{(\ell)}}$ (Stage $\ell$)]
\label{lemma: empirical gradient of stage ell}
    Suppose the input sequence $\hat{X}^{(\ell-1)}_m$ for the $\ell$th layer satisfies $\|\hat{X}_m^{(\ell-1)}-X_m^{(\ell-1)}\|_{\infty} \le (\ell-1)\epsilon$ before stage $\ell$ where $X_m^{(\ell-1)}$ is the ideal sequence. Assume that $W^{(\ell')}_{KQ}(\ell-1)=0_{d\times d}$ for all layers $
    \ell'\ge \ell$. After one step of gradient descent on the $\ell$th stage finite sample loss $\mathcal{L}^{(\ell)}$ with $M$ training sequences and learning rate $\eta$, satisfying  
    $\beta_0\le 1,$ $M\gtrsim k^2N^4d^2\log^2\frac{d}{\delta},\eta \gtrsim \frac{k^2N^3\log \frac{kN}{\epsilon}}{\beta_0}$ for any $\epsilon \in(0, \frac{1}{k^2N^3\log^2k})$, then for any $(i,j)\in[k]\times [N]$, after stage $\ell$ we have $$\S^{(\ell)}_{(i+2^{\ell-2},\hop^{2^{\ell-2}}_i(j))}:=\S((\hat{X}^{(\ell-1)})^\top W_{KQ}^{(\ell)}\hat{X}^{(\ell-1)}_{(i,j)})_{(i+2^{\ell-2},\hop^{2^{\ell-2}}_i(j))} \ge 1-\frac{1}{2}\epsilon.$$
    Furthermore, if we pick $\eta \gtrsim \frac{Ck^2N^3\log k\log \frac{kN}{\epsilon}}{\beta_0}$, we have
    $\S^{(\ell)}_{(i+2^{\ell-2},\hop^{2^{\ell-2}}_i(j))}\ge 1-\frac{\epsilon}{2(kNL)^{CL}}$ for any absolute constant $C$.
\end{lemma}
\begin{proof}
We follow the first and second stage strategy, first computing the population gradient and then doing the finite-sample analysis and controlling the perturbation error. Similarly, we ignore the subscript of $X^{(\ell-1)}$ when it is clear from context in this subsection.

The ideal input for each layer $\ell (\ell\ge3)$ is in the following form:
{\scriptsize
\begin{align*}
   \hspace{-1cm} X^{(\ell-1)} = \begin{pNiceArray}{ccc|c|ccc}
 e_{k,1}\otimes e_{N,1}&\cdots&e_{k,1}\otimes e_{N,N}&\cdots&e_{k,k}\otimes e_{N,1}&\cdots&e_{k,1}\otimes e_{N,N}\\
e_{k,1}\otimes e_{N,\sigma_1(1)}&\cdots& e_{k,1}\otimes e_{N,\sigma_1(N)}&\cdots&e_{k,k}\otimes e_{N,\sigma_k(1)}&\cdots&e_{k,k}\otimes e_{N,\sigma_k(N)}\\
e_{k,1}\otimes e_{N,\sigma_1\pi_1(1)}&\cdots& e_{k,1}\otimes e_{N,\sigma_1\pi_1(N)}&\cdots&e_{k,k}\otimes e_{N,\sigma_k\pi_k(1)}&\cdots&e_{k,k}\otimes e_{N,\sigma_k\pi_k(N)}\\
e_{k,2}\otimes e_{N,\hop^2_1(1)}&\cdots& e_{k,2}\otimes e_{N,\hop^2_1(1)}&\cdots&e_{k,1}\otimes e_{N,\hop^2_1(1)}&\cdots&e_{k,1}\otimes e_{N,\hop^2_k(N)}\\
\vdots &&\vdots&\vdots&\vdots&&\vdots\\
e_{k,2^{\ell-2}}\otimes e_{N,\hop^{2^{\ell-2}}_1(1)}&\cdots& e_{k,2^{\ell-2}}\otimes e_{N,\hop^{2^{\ell-2}}_1(N)}&\cdots&e_{k,k+2^{\ell-2}-1}\otimes e_{N,\hop^{2^{\ell-2}}_k(1)}&\cdots&e_{k,k+2^{\ell-2}-1}\otimes e_{N,\hop^{2^{\ell-2}}_k(N)}\\
0_{(L-\ell+1)kN} &\cdots&0_{(L-\ell+1)kN}&\cdots&0_{(L-\ell+1)kN}&\cdots&0_{(L-\ell+1)kN}
\end{pNiceArray}
\end{align*}
}

We have the population gradient
\begin{align*}
    \nabla \mathcal{L}^{(\ell)} &= -\E\qty[\sum_{s'\in[N]}\mathbf{1}\{s'=\hop_i^{2^{\ell-1}}(j)\} X^{(\ell-1)}J^{(\ell)} (X^{(\ell-1)})^\top {(\Psi_\ell^\top W_{OV}^{(\ell)})_{s'}}^\top (X^{(\ell-1)}_{(i,j)})^\top]\\
    &= -\frac{1}{kN}\E\qty[X^{(\ell-1)}\qty(I_{kN\times kN}-\frac{1}{kN}\mathbf{1}\mathbf{1}^\top) (X^{(\ell-1)})^\top {(\Psi_\ell^\top W_{OV}^{(\ell)})_{\hop_i^{2^{\ell-1}}}}^\top X_{(i,j)}^\top].
\end{align*}
Similar to previous stages, we can directly calculate $(X^{(\ell-1)})^\top {(\Psi_\ell^\top W_{OV}^{(\ell)})_{\hop_i^{2^{\ell-1}}}}^\top.$ The $p$-th block of the vector should be the one-hot vector of the following positions ($p\in[N]$)
$$e_{(\hop_p^{2^{\ell-2}})^{-1}\hop_i^{2^{\ell-1}}(j)}=e_{(\sigma_{p+2^{\ell-2}-1}\pi_{p+2^{\ell-2}-1}\cdots \sigma_p\pi_p)^{-1}\sigma_{i+2^{\ell-1}-1}\pi_{i+2^{\ell-1}-1}\cdots \sigma_i\pi_i(j)}.$$ 

Therefore, the population gradient become
\begin{align*}
    \nabla \mathcal{L}^{(\ell)}
    &= -\frac{1}{kN}\E\qty[X^{(\ell-1)}\qty(I_{kN\times kN}-\frac{1}{kN}\mathbf{1}\mathbf{1}^\top) (X^{(\ell-1)})^\top {(\Psi_\ell^\top W_{OV}^{(\ell)})_{\hop_i^{2^{\ell-1}}}}^\top X_{(i,j)}^\top]\\
    &=  -\frac{\beta_0}{kN}\E\qty[X^{(\ell-1)}\qty((X^{(\ell-1)})^\top {(\Psi_\ell^\top W_{OV}^{(\ell)})_{\hop_i^{2^{\ell-1}}}}^\top-\frac{1}{N}\mathbf{1}) X_{(i,j)}^\top]\\
    &=-\frac{\beta_0}{kN}\E\qty[\qty(\sum_{p=1}^k X^{(\ell-1)}\qty(e_{k,p}\otimes e_{(\hop_p^{2^{\ell-2}})^{-1}\hop_i^{2^{\ell-1}}(j)})-\frac{1}{N}\begin{bmatrix}
        \1_{kN(\ell+1)}\\
        0_{kN(L-\ell+1)}
    \end{bmatrix})X_{(i,j)}^\top]
\end{align*}
Consider each vector $X^{(\ell)}(e_{k,p}\otimes e_{(\hop_p^{2^{\ell-2}})^{-1}\hop_i^{2^{\ell-1}}(j)})$ in the first summation.
Observe that when $p>2^{\ell-2}+i$, there is a permutation $\sigma_{i+2^{\ell-2}}^{-1}$ that is independent of $X^{(\ell-1)}_{(i,j)}$, since it never appear in any existing hop encoded in $X^{(\ell-1)}_{(i,j)}$. Similarly, when $p<2^{\ell-2}+i$, $\sigma_{i+2^{\ell-1}-1}$ is independent of $X^{(\ell-1)}_{(i,j)}$. Then the vector has the following expectation condition on $X^{(\ell-1)}_{(i,j)}$:
$$\frac{1}{N}\begin{bmatrix}
        e_{k,p}\otimes \1_N,
        e_{k,p}\otimes \1_N,
        e_{k,p}\otimes \1_N,
        e_{k,p+1}\otimes \1_N,
        \cdots,
        e_{k,p+2^{\ell-2}-1}\otimes \1_N,
        0_{kN(L-\ell+1)}
    \end{bmatrix}^\top$$ and they cancel with the normalization term.

Therefore, the remaining index is $p^\star=2^{\ell-2}+i$. We denote the normalization vector for $p^*$:
$$\phi_i=\frac{1}{N}\begin{bmatrix}
        e_{k,p^*}\otimes \1_N,
        e_{k,p^*}\otimes \1_N,
        e_{k,p^*}\otimes \1_N,
        e_{k,p^*+1}\otimes \1_N,
        \cdots,
        e_{k,p^*+2^{\ell-2}-1}\otimes \1_N,
        0_{kN(L-\ell+1)}
    \end{bmatrix}^\top$$
And the ideal population gradient becomes
\begin{align*}
    \nabla \mathcal{L}_{\mathcal{D}}^{(\ell)}
    &= -\frac{\beta_0}{kN}\E\qty[\qty(X^{(\ell-1)}(e_{k,2^{\ell-2}+i}\otimes e_{(\hop_{2^{\ell-2}+i}^{2^{\ell-2}})^{-1}\hop_i^{2^{\ell-1}}(j)})-\frac{1}{N}\phi_i)X^{(\ell)}_{(i,j)}]\\
    &= -\frac{\beta_0}{kN}\E\qty[\qty(X^{(\ell-1)}(e_{k,2^{\ell-2}+i}\otimes e_{\hop_i^{2^{\ell-2}}(j)})-\frac{1}{N}\phi_i)X^{(\ell)}_{(i,j)}]
\end{align*}
The last identity is due to the definition of the hop:
$$e_{(\hop_{p^\star}^{2^{\ell-2}})^{-1}\hop_i^{2^{\ell-1}}(j)}=e_{(\sigma_{p^\star+2^{\ell-2}-1}\pi_{p^\star+2^{\ell-2}-1}\cdots \sigma_{p^\star}\pi_{p^\star})^{-1}\sigma_{i+2^{\ell-1}-1}\pi_{i+2^{\ell-1}-1}\cdots \sigma_i\pi_i(j)} = \hop_i^{2^{\ell-2}}(j).$$
We expand the two vectors $X^{(\ell-1)}(e_{k,2^{\ell-2}+i}\otimes e_{\hop_i^{2^{\ell-2}}(j)})$ and $X^{(\ell-1)}_{(i,j)}$:
\begin{align*}
    X^{(\ell-1)}(e_{k,2^{\ell-2}+i}\otimes e_{\hop_i^{2^{\ell-2}}(j)}) =\begin{bmatrix}
        \underline{e_{k,2^{\ell-2}+i}\otimes e_{\hop_i^{2^{\ell-2}}(j)}}\\
        e_{k,2^{\ell-2}+i}\otimes e_{\sigma_{2^{\ell-2}+i}\hop_i^{2^{\ell-2}}(j)}\\
        e_{k,2^{\ell-2}+i}\otimes e_{\hop_{2^{\ell-2}+i}^1\hop_i^{2^{\ell-2}}(j)}\\
        e_{k,2^{\ell-2}+i+1}\otimes e_{\hop_{2^{\ell-2}+i}^2\hop_i^{2^{\ell-2}}(j)}\\
        \vdots\\
        e_{k,2^{\ell-1}+i-1}\otimes e_{\hop_{2^{\ell-2}+i}^{2^{\ell-2}}\hop_i^{2^{\ell-2}}(j)}\\
        0_{(L-\ell+1)kN}
    \end{bmatrix}, x_v = \begin{bmatrix}
        e_{k,i}\otimes e_{N,j}\\
        e_{k,i}\otimes e_{N,\sigma_{i}(j)}\\
        e_{k,i}\otimes e_{N, \hop_{i}^1(j)}\\
        e_{k,i+1}\otimes e_{N,\hop_i^2(j)}\\
        \vdots\\
        {e_{k,2^{\ell-2}+i-1}\otimes e_{N,\hop_{i}^{2^{\ell-2}}(j)}}\\
        0_{(L-\ell+1)kN}
    \end{bmatrix}
\end{align*}
Observe that $X^{(\ell-1)}_{(i,j)}$ can be decomposed into
$$X^{(\ell-1)}_{(i,j)} = \underbrace{\begin{bmatrix}
        e_{k,i}\otimes e_{N,j}\\
        e_{k,i}\otimes e_{N,\sigma_{i}(j)}\\
        e_{k,i}\otimes e_{N, \hop_{i}^1(j)}\\
        e_{k,i+1}\otimes e_{N,\hop_i^2(j)}\\
        \vdots\\
        e_{k,2^{l-3}+i-1}\otimes e_{N,\hop_{i}^{2^{l-3}}(j)}\\
        0_{(L-\ell+2)kN}
    \end{bmatrix}}_{x_{1}}+\underbrace{\begin{bmatrix}
        0_{kN\ell}\\
        e_{k,2^{\ell-2}+i-1}\otimes e_{N,\hop_{i}^{2^{\ell-2}}(j)}\\
        0_{(L-\ell+1)kN}
    \end{bmatrix}}_{x_{2}}$$
Given the first vector and $\ell\ge 3$, the whole vector $X^{(\ell-1)}(e_{k,2^{\ell-2}+i}\otimes e_{\hop_i^{2^{\ell-2}}(j)})$ has expectation 
$$\E\qty[X^{(\ell-1)}(e_{k,2^{\ell-2}+i}\otimes e_{\hop_i^{2^{\ell-2}}(v)})\big|x_{v,1}]=\frac{1}{N}\begin{bmatrix}
        \phi_i\\
        0_{kN(L-\ell+1)}
    \end{bmatrix}$$
because $\sigma_{2^{\ell-2}+i-2}$ is independent of $x_{1}$, and thus cancel with the normalization term. 

Similarly, given the second vector $x_{2}$, the only block vector that does not have uniform expectation conditioned on $x_{2}$ is the first block (underlined). After taking expectation, the rest block vectors cancels with the corresponding blocks of the normalization vector $\phi_i$.
Therefore, the population gradient can be written into the following form (only the $(1,\ell+1)$-th block is non-zero):
\begin{align*}
    \nabla \mathcal{L}_{\mathcal{D}}^{(\ell)}
    &= -\frac{\beta_0}{kN}\E\begin{bmatrix}
        0_{kN\times kN\ell}\quad A\quad 0_{kN\times (L-\ell+1)kN}\\
        0_{kN(L+1)\times kN(L+2)}
    \end{bmatrix}
\end{align*}
and the expectation of $A$ is
$$\frac{1}{kN}\sum_{i=1}^k\sum_{p=1}^N(e_{k,2^{\ell-2}+i}\otimes e_{N,p})(e_{k,2^{\ell-2}+i-1}\otimes e_{N,p})^\top - \frac{1}{kN^2}\sum_{i=1}^k(e_{k,2^{\ell-2}+i}\otimes \1_N)(e_{k,2^{\ell-2}+i-1}\otimes \1_N)^\top$$
We will later show that this gradient can correctly learn the functionality of the $\ell$th layer, combined with the following analysis upper bounding the gradient noise and accumulated error.

Following similar strategy from the second stage, we consider the empirical estimate of the gradient. Define the population gradient matrix
$$g_{\ell} = \E\begin{bmatrix}
        0_{kN\times kN\ell}\quad A\quad 0_{kN\times (L-\ell+1)kN}\\
        0_{kN(L+1)\times kN(L+2)}
    \end{bmatrix}$$
The empirical gradient can be written as
$\nabla \hat{L}^{(\ell)} = -\frac{\beta_0}{kN}\hat{g}_\ell,$
where $\hat{g}_\ell$ is the empirical estimate of $g_\ell$ with perturbed inputs:
$$\hat{g}_\ell = \frac{kN}{\beta_0M}\sum_{m=1}^M \sum_{s'\in[N]}(\mathbf{1}\{s'=\hop_{i_m}^{2^{\ell-1}}(j_m)\}-\S_{s'}) \hat{X}^{(\ell-1)}_mJ^{(\ell)} (\hat{X}^{(\ell-1)}_m)^\top {\qty(\Psi_\ell^\top W_{OV}^{(\ell)})_{s'}}^\top (\hat{X}^{(\ell-1)}_{(i_m,j_m)})^\top$$
It suffices to show that $\|\hat{g}_\ell-g_\ell\|_\infty$ is small. It is similar to the first stage analysis for the sample noise part, with the perturbed input error introduced:
\begin{lemma}
\label{lemma: concentration for stage l}
    Suppose $\|X_m^{(\ell-1)}-\hat{X}_m^{(\ell-1)}\|_\infty\le (\ell-1)\epsilon$ for all $m$. For any $\delta>0$, we have with probability $1-\delta$ s.t.
    $$\|\hat{g}_{\ell}-g_{\ell}\|_\infty\lesssim \frac{\log(kN/\delta)}{\sqrt{M}} +d\ell\epsilon$$
\end{lemma}
\begin{proof}
For simplicity, we ignore the superscript $^{(\ell-1)}$ in the proof of this lemma. The proof in large follows \Cref{lemma: concentration for stage 2}.

First, the empirical gradient of the single ideal sample $X_m$ is
\begin{align*}
    \hat{g}^{\mathrm{ideal}}_{\ell,m} 
    &=X_m (I-\frac{1}{kN}\1\1^\top)X_m^\top (e_{L+2,\ell+1}\otimes \1_k\otimes e_{N,\hop^{2^{\ell-1}}(v_m)}) ({X}_{(i_m,j_m)})^\top
\end{align*}
Similar to the proof in \Cref{lemma: concentration for stage 2}, the upper bound of each entry of the random variable is
\begin{equation}\label{eq: stage l, noise analysis}
    \|\hat{g}^{\mathrm{ideal}}_{\ell,m}\|_\infty 
    \le \left\|X_m (I-\frac{1}{kN}\1\1^\top)X_m^\top (e_{L+2,\ell+1}\otimes \1_k\otimes e_{N,\hop^{2^{\ell-1}}_{i_m}(j_m)})\right\|_\infty \|X^{(\ell-1)}_{(i_m,j_m)}\|_\infty\le 1.
\end{equation}
Then we can concentrate $\hat{g}_\ell^{\mathrm{ideal}} := \frac{1}{M}\sum_{m=1}^M \hat{g}^{\mathrm{ideal}}_{2,m}$ as:
\begin{align*}
    \norm{\hat{g}_\ell^{\mathrm{ideal}}-g_\ell}_\infty &= \norm{\frac{1}{M}\sum_{m=1}^M \hat{g}_{l,m}^{\mathrm{ideal}}-\E\hat{g}^{\mathrm{ideal}}_{l}}\lesssim \frac{\log(d/\delta)}{\sqrt{M}}
\end{align*}
with probability $1-\frac{\delta}{d^2}$. By $d = O(kN\log k)$, we can union bound over all entries of $\hat{g}_2$ and we have the desired first term error. 

Then we bound the perturbation error between the empirical gradients with $X_m$ and those with perturbed $\hat{X}_m$.
Recall that the empirical gradient 
\begin{align*}
    \hat{g}_{\ell,m} 
    &=\hat{X}_m (I-\frac{1}{kN}\1\1^\top)\hat{X}_m^\top (e_{L+2,\ell+1}\otimes \1_k\otimes e_{N,\hop^{2^{\ell-1}}_{i_m}(j_m)}) (\hat{X}^{(\ell-1)}_{(i_m,j_m)})^\top\\&-\sum_{s'\in[N]}\S_{s'}\hat{X}_m (I-\frac{1}{kN}\1\1^\top)\hat{X}_m^\top (e_{L+2,\ell+1}\otimes \1_k\otimes e_{N,s'}) (\hat{X}^{(\ell-1)}_{(i_m,j_m)})^\top.
\end{align*}
Denote the following term as
$$\hat{\gamma}_{s',m;\ell}=\hat{X}_m (I-\frac{1}{kN}\1\1^\top)\hat{X}_m^\top (e_{L+2,\ell+1}\otimes \1_k\otimes e_{N,s'}) (\hat{X}^{(\ell-1)}_{(i_m,j_m)})^\top,$$
$${\gamma}_{s',m;\ell}={X}_m (I-\frac{1}{kN}\1\1^\top){X}_m^\top (e_{L+2,\ell+1}\otimes \1_k\otimes e_{N,s'}) ({X}^{(\ell-1)}_{(i_m,j_m)})^\top.$$
and we define $\Delta \gamma_{s',m;\ell}=\hat{\gamma}_{s',m;\ell}-\gamma_{s',m;\ell}$. Then we can rewrite the empirical gradient into
\begin{align*}
    \hat{g}_{\ell,m}-\hat{g}^{\mathrm{ideal}}_{\ell,m} 
    &=\Delta\gamma_{\hop^{2^{\ell-1}}(v_m),m}-\sum_{s'\in[N]}\S_{s'}\Delta\gamma_{s',m}.
\end{align*}
The error of the following difference $\|\Delta \gamma_{s',m;\ell}\|_\infty=\|\gamma_{s',m;\ell}-\hat{\gamma}_{s',m;\ell}\|_\infty\le Cd\ell\epsilon$ with some absolute constant for all possible $s'\in[N]$, since $\norm{\hat{X}_m-X_m}_\infty\le (\ell-1)\epsilon$ and $\norm{X_m}_\infty\le 1$. We have the perturbation error upper bounded by $O(d\ell\epsilon)$:
\begin{equation}\label{eq: stage l, perturbation analysis}
    \norm{\hat{g}_{\ell,m}-\hat{g}^{\mathrm{ideal}}_{\ell,m}}_\infty \le \|\Delta \gamma_{\hop,m;\ell}\|+ \sum_{s'}\S_{s'}\|\Delta \gamma_{s',m;\ell}\|
    \le2Cd\ell\epsilon.
\end{equation}
Combine both \eqref{eq: stage l, noise analysis} and \eqref{eq: stage l, perturbation analysis}, we finished the proof.
\end{proof}

After the one-step gradient, we get $W_{KQ}^{(\ell)}=\eta \hat{g}_\ell$.
Given a sample sequence $\hat{X}^{(\ell)}_m$ and index $(i_m,j_m)$, we have the attention score
\begin{align*}
    \eta (\hat{X}^{(\ell-1)}_m)^\top \hat{g}_\ell \hat{X}^{(\ell-1)}_{(i_m,j_m)} ={}& \frac{\beta_0\eta}{kN}\qty[(X_m^{(\ell-1)})^\top {g}_l X^{(\ell-1)}_{(i_m,j_m)} + (X_m^{(\ell-1)})^\top \underbrace{(\hat{g}_\ell-g_\ell)}_{\Delta g_\ell} \hat{X}^{(\ell-1)}_{(i_m,j_m)}]\\
    &+\frac{\beta_0\eta}{kN}\underbrace{\qty[(\hat{X}^{(\ell-1)}_m)^\top \hat{g}_\ell \hat{X}^{(\ell-1)}_{(i_m,j_m)}-X_m^\top \hat{g}_\ell X_{(i_m,j_m)}]}_{\text{Perturbation error}}
\end{align*}
For simplicity, we ignore the superscript for the following proof in this subsection. Note that the population attention score is
\begin{align*}
    X_m^\top {g}_\ell X_{(i_m,j_m)}&=X_m^\top \begin{bmatrix}
        \frac{1}{kN}\qty(e_{k,2^{\ell-2}+i_m}\otimes e_{N,\hop_{i_m}^{2^{\ell-2}}(j_m)}-\frac{1}{N}e_{k,2^{\ell-2}+i_m}\otimes \1_N)\\
        0_{kN(L+1)}\\
    \end{bmatrix}
\end{align*}

Define the population/empirical separation between the correct position $p_\ell:=N(i_m+2^{\ell-2})+\hop_{i_m}^{2^{\ell-2}}(j_m)$, i.e. $(i_m+2^{\ell-2},\hop_{i_m}^{2^{\ell-2}}(j_m))$ position, and the others as follows:
$$\Delta^{(\ell)}_{i_m,j_m} = \qty(X_m^\top {g}_l X_{(i_m,j_m)})_{p_\ell}-\max_{j\not=p_\ell}\qty(X_m^\top {g}_l X_{(i_m,j_m)})_j.$$
$$\hat{\Delta}^{(\ell),\mathrm{ideal}}_{i_m,j_m} = \qty(X_m^\top \hat{g}_\ell X_{(i_m,j_m)})_{p_\ell}-\max_{j\not=p_\ell}\qty(X_m^\top \hat{g}_\ell X_{(i_m,j_m)})_j.$$
$$\hat{\Delta}^{(\ell)}_{i_m,j_m} = \qty(\hat{X}_m^\top \hat{g}_\ell \hat{X}_{(i_m,j_m)})_{p_\ell}-\max_{j\not=p_\ell}\qty(\hat{X}_m^\top \hat{g}_\ell\hat{X}_{(i_m,j_m)})_j.$$
We have the $(i,j)$-th entry of the pre-softmax attention score
\begin{align*}
    \qty(X_m^\top {g}_\ell X_{(i_m,j_m)})_{(i,j)} = \begin{cases}
        \frac{N-1}{kN^2},&i=i_m+2^{\ell-2}, j=\hop_{i_m}^{2^{\ell-2}}(j_m)\\
        -\frac{1}{kN^2},&i=i_m+2^{\ell-2},\ j\not=\hop_{i_m}^{2^{\ell-2}}(j_m)\\
        0,&i\not=i_m+2^{\ell-2}
    \end{cases}
\end{align*}
So $\Delta_{i_m,j_m}^{(\ell)}\ge \frac{1}{kN^2}$ by $N\ge 2$ and the expected signal is at least $\frac{1}{kN^2}$.

Now we bound the noise introduced by $\Delta g_\ell$. Since $\|\Delta g_\ell\|_\infty \lesssim \frac{\log \frac{d}{\delta}}{\sqrt{M}}+d\ell\epsilon$, we have
$$\left\|X_m^\top {\Delta g_\ell} X_{(i_m,j_m)}\right\|_\infty \lesssim \frac{d\log \frac{d}{\delta}}{\sqrt{M}}+d\ell\epsilon\le \frac{1}{2kN^2}$$
by $M\gtrsim k^2N^4d^2\log^2\frac{d}{\delta}$ and $\epsilon\lesssim \frac{1}{k^2N^3\log^2k}$. Therefore, the noise term can be bounded and the empirical separation with the ideal input sequence $\hat{\Delta}_{i_m,j_m}^{(\ell),\mathrm{ideal}}\ge \frac{1}{2kN^2}$. This also gives the upper bound for $\|\hat{g}_\ell\|_\infty\le O(1/kN^2)$.

Similar to stage 2, we check the upper bounds for gradients for previous layers, and make sure the perturbation error upper bounds for $\norm{\hat{X}_m^{(\ell)}-X_m^{(\ell)}}\le \ell \epsilon$ still hold. The following lemma shows that after the $\ell$th gradient step, the previous layer key-query matrices are almost unperturbed and the softmax score is still close to one-hot.
\begin{lemma}[$W_{KQ}^{(\ell')}$ is unchanged] Under the condition of \Cref{lemma: empirical gradient of stage ell}, after the $\ell$th step we still have for all $\ell'<\ell$, 
$$\S^{(\ell')}_{(i,\pi_i(j))}:=\S(\hat{X}^{(\ell'-1)}W_{KQ}^{(\ell')}\hat{X}^{(\ell'-1)}_{(i,j)})_{(i,\hop_i^{2^{\ell'-2}}(j))} \ge 1-\frac{1}{2}\epsilon.$$
\end{lemma}

\begin{proof}
    By \Cref{lemma: gradient upper bound for previous layers with curriculum}, the gradient norm for each previous layer is upper bounded by $O(\beta_0k^{5/2}N^{3/2}L^{3/2}\epsilon)$. With the gradient update in the current stage, $W_{KQ}^{(\ell')}$ will only be perturbed by
    $$\norm{\eta \nabla_{W_{KQ}^{(\ell')}} \mathcal{L}^{(\ell)}}\lesssim (kN)^{4.5}\log\frac{kN}{\epsilon}\log^{2.5} k \cdot \epsilon \ll N\log\frac{kN}{\epsilon}.$$
    where $N\log\frac{kN}{\epsilon}$ is the scale of each entry of $W_{KQ}^{(\ell')}$. Therefore, the lemmas on the attention score in previous stages still hold.
\end{proof}
Given the lemma above, we have $\norm{\hat{X}_m^{(\ell)}-X_m^{(\ell)}}\le \ell \epsilon$ using the previous layer perturbation error analysis for $X_m^{(\ell)}$. Finally, we are ready to bound the perturbation error for pre-softmax attention score:
\begin{align*}
    &\norm{(\hat{X}^{(\ell)}_m)^\top \hat{g}_\ell \hat{X}^{(\ell)}_{(i_m,j_m)}-X_m^\top \hat{g}_\ell X_{(i_m,j_m)}}_\infty\\
    \le{}&\norm{\hat{g}_\ell}_2\norm{(\hat{X}^{(\ell)}_m)^\top (\hat{X}^{(\ell)}_{(i_m,j_m)}-X_{(i_m,j_m)})}_\infty+\norm{\hat{g}_\ell}_2\norm{(\hat{X}^{(\ell)}_m-X_m)^\top X_{(i_m,j_m)}}_\infty\\
    \le{}&\norm{\hat{g}_\ell}_F\norm{(\hat{X}^{(\ell)}_m)^\top (\hat{X}^{(\ell)}_{(i_m,j_m)}-X_{(i_m,j_m)})}_\infty+\norm{\hat{g}_\ell}_F\norm{(\hat{X}^{(\ell)}_m-X_m)^\top X_{(i_m,j_m)}}_\infty\\
    \lesssim{}& \sqrt{\frac{1}{(kN^2)^2}\cdot d^2}\cdot d\ell\epsilon.
\end{align*}
Since $\epsilon=O(\frac{1}{k^2N^3L^2})$, the perturbation error is upper bounded by $O(\frac{1}{kN^2})$. Therefore, the empirical separation $\hat{\Delta}^{(\ell)}_{(i_m,j_m)}\ge \Omega(\frac{1}{kN^2})$.
After one step gradient with learning rate $\eta \gtrsim \frac{k^2N^3}{\beta_0}\log\frac{kN}{\epsilon}$, the softmax output of the correct position can be lower bounded by
\begin{align*}
    \S(\hat{X}_m^\top W_{KQ}^{(\ell)} \hat{X}_{(i_m,j_m)})_{i_m,\hop^{\ell-2}_{i_m}(j_m)} \ge \frac{\exp\qty(\frac{\eta\beta_0}{kN}\cdot \frac{1}{kN^2})}{\exp\qty(\frac{\eta\beta_0}{kN}\cdot \frac{1}{kN^2})+kN-1}\ge 1-\frac{1}{2}\epsilon.
\end{align*}
\end{proof}

\paragraph{Perturbation analysis of Stage $\ell$} 
The lemma presents the perturbation analysis for stage $\ell$.
\begin{lemma}[Perturbation analysis of stage $\ell$]\label{lemma: perturbation for stage ell}
    Under the condition of \Cref{lemma: empirical gradient of stage ell}, we have
    $$\|X^{(\ell)}-\hat{X}^{(\ell)}\|\le \ell\epsilon,$$
    where $X^{(\ell)}$ is the ideal output with saturated softmax, and $\hat{X}^{(\ell)}$ is the transformer output after the $\ell$th stage.
\end{lemma}
\begin{proof} The proof is similar to stage 2.
The ideal output for the $\ell$th layer is $$X^{(\ell)} = X^{(\ell-1)} + W_{OV}^{(\ell)}X^{(\ell-1)}\S^{(\ell)}_{\mathrm{ideal}},$$ where $\S^{(\ell)}_{\mathrm{ideal}}$ is the ideal one-hot softmax attention pattern. The empirical output $\hat{X}^{(\ell)}$ has the error introduced by the non-saturation of the softmax and the previous error in $\hat{X}^{(\ell-1)}$:
\begin{align*}
    \hat{X}^{(\ell)} &= \hat{X}^{(\ell-1)} + W_{OV}^{(\ell)}\hat{X}^{(\ell-1)}\S^{(\ell)}\\ &= X^{(\ell)} + W_{OV}^{(\ell)}X^{(\ell-1)}\underbrace{(\S^{(\ell)}-\S^{(\ell)}_{\mathrm{ideal}})}_{\Delta \S^{(\ell)}, \text{Non-saturation error}}+\underbrace{(\hat{X}^{(\ell-1)}-{X}^{(\ell)})+W_{OV}^{(\ell)}(\hat{X}^{(\ell-1)}-{X}^{(\ell-1)})\S^{(\ell)}}_{\text{Accumulated perturbation error}}.
\end{align*}
First, we consider the non-saturation error term. By \Cref{lemma: empirical gradient of the second stage}, the correct entry of the each softmax probability vector is greater than $1-\frac{1}{2}\epsilon$ for all index $v$. And other probabilities has $\epsilon$ error in total, and they are all positive. Note that $\|X^{(\ell-1)}\|_\infty\le 1$. As a result, the error of the non-saturation error can be bounded as
$$\|W_{OV}^{(\ell)}X^{(\ell-1)}(\S^{(\ell)}-\S^{(\ell)}_{\mathrm{ideal}})\|_\infty  = \max_{s,i,j}\qty|(W_{OV}^{(\ell)}X^{(\ell-1)})_s^\top \Delta\S^{(\ell)}_{i,j}|\le \|X^{(\ell-1)}\|_\infty\cdot 2\cdot\frac{\epsilon}{2} \le \epsilon.$$
Now consider the accumulated perturbation error. By the perturbation analysis in stage $\ell-1$, we have $\norm{\hat{X}^{(\ell-1)}-{X}^{(\ell-1)}}_{\infty}\le (\ell-1)\epsilon.$ Note that the error in $\hat{X}^{(\ell-1)}-{X}^{(\ell-1)}$ are in different rows of the matrices from 
$W_{OV}^{(\ell)}(\hat{X}^{(\ell-1)}-{X}^{(\ell-1)})\S^{(\ell)}$. Therefore, $\hat{X}^{(\ell-1)}-{X}^{(\ell-1)}$ won't introduce extra error in this stage, and we only need to consider $W_{OV}^{(\ell)}(\hat{X}^{(\ell-1)}-{X}^{(\ell-1)})\S^{(\ell)}$:
\begin{align*}
    \norm{W_{OV}^{(\ell)}(\hat{X}^{(\ell-1)}-{X}^{(\ell-1)})\S^{(\ell)}}_\infty &= \max_{s,(i,j)}\qty|(W_{OV}^{(\ell)}(\hat{X}^{(\ell-1)}-{X}^{(\ell-1)}))_s^\top \S^{(\ell)}_{(i,j)}|\\
    &=\max_{s,(i,j)}\qty|\sum_{p=1}^{kN}(W_{OV}^{(\ell)}(\hat{X}^{(\ell-1)}-{X}^{(\ell-1)}))_{s,p} (\S^{(\ell)}_{(i,j)})_p|\\
    &\le \|\hat{X}^{(\ell-1)}-{X}^{(\ell-1)}\|_\infty\tag{Since $\sum_p (\S^{(2)}_{(i,j)})_p=1, (\S^{(2)}_{(i,j)})_p\ge0$.}
\end{align*}
Combine both parts of error, we have
$\norm{\hat{X}^{(\ell)}-{X}^{(\ell)}}_\infty \le \ell\epsilon.$
\end{proof}

At the end of each stage $\ell$, we further train the readout layer with one gradient step to output the correct $2^{\ell-1}$-hop for each position.
\begin{lemma}\label{lemma: readout for ellth layer}
    Under the condition of \Cref{lemma: empirical gradient of stage ell}, after one gradient step on $\Psi_\ell$ we have
    \begin{align*}
    \sup_{\sigma,(i,j)}\norm{\S(\Psi_\ell^\top f^{(\ell)}(X^{(\ell-1)})_{(i,j)})-e_{\hop^{2^{\ell-1}}_i(j)}}_\infty\le \epsilon.
\end{align*}
\end{lemma}
\begin{proof}
We calculate the population gradient for $\Psi_\ell$, and then do the finite sample analysis.
The population gradient is
\begin{align*}
    \nabla_{\Psi_{\ell}} \mathcal{L}_{\mathcal{D}}^{(\ell)}(\theta) =-\E\qty[\qty(e_{N,\hop^{2^{\ell-1}}_i(j)}-\S(\Psi_\ell^\top f^{(\ell)}(X)_{(i,j)})) (f^{(\ell)}(X)_{(i,j)})^\top]
\end{align*}
By \Cref{lemma: empirical gradient of stage ell}, the output of the $\ell$th layer would be
\begin{align*}
    f^{(\ell)}(X)_{(i,j)} &= W_{OV}^{(\ell)}\hat{X}^{(\ell-1)}\S^{(\ell)} = W_{OV}^{(\ell)}X^{(\ell-1)}\S^{(\ell)}_{\mathrm{ideal}}+W_{OV}^{(\ell)}(\hat{X}^{(\ell-1)}-X^{(\ell-1)})\S^{(\ell)}_{\mathrm{ideal}}+W_{OV}^{(\ell)}X^{(\ell-1)}\Delta\S^{(\ell)}\\
    &= e_{L+2,\ell+2}\otimes e_{k,i}\otimes e_{N,\hop^{2^{\ell-1}}_i(j)} + \Delta_\ell.
\end{align*}
where $\|\Delta_\ell\|_\infty\le \ell{\epsilon}$ by \Cref{lemma: perturbation for stage ell}. Since $\Psi_\ell(0)=\beta_0 e_{L+2,\ell+2}\otimes \mathbf{1}_{k}\otimes I_{N\times N}$, we have the expansion
\begin{align*}
    \S(\Psi_\ell^\top f^{(\ell)}(X)_{(i,j)}) = \S(\beta_0e_{N,\hop_i^{2^{\ell-1}}(j)}+\Psi_{\ell}^\top \Delta_\ell)= \S(\beta_0e_{N,\hop^{2^{\ell-1}}_i(j)}) + \tilde{J}\Psi_\ell^\top \Delta_\ell.
\end{align*}
Since $\|\Delta_\ell\|_\infty\le {\epsilon}$, we have $\norm{\tilde{J}\Psi_\ell^\top\Delta_\ell}_\infty\le \beta_0\epsilon.$ The signal term
$$\S(\beta_0e_{N,\hop^{2^{\ell-1}}_i(j)}) = \frac{\exp(\beta_0)-1}{\exp(\beta_0)+N-1}e_{N,\hop^{2^{\ell-1}}_i(j)}+\frac{1}{\exp(\beta_0)+N-1}\1_N.$$
The population gradient is thus
\begin{align*}
    \nabla_{\Psi_{\ell}} \mathcal{L}_{\mathcal{D}}^{(\ell)}(\theta) 
    &=-\E\qty[\qty(\frac{N}{\exp(\beta_0)+N-1}e_{N,\hop^{2^{\ell-1}}_i(j)}-\frac{1}{\exp(\beta_0)+N-1}\1_N-\tilde{J}\Psi_\ell^\top\Delta_\ell)(f^{(\ell)}(X)_{(i,j)})^\top]\\
    &=-\frac{1}{(\exp(\beta_0)+N-1)k}\qty(e_{L+2,\ell+2}\otimes \1_k \otimes (I_N-\frac{1}{N}\1_N\1_N^\top)) + O(\epsilon).
\end{align*}
where $O(\epsilon)$ denotes the terms with infinity norm smaller than $\epsilon$ with $\epsilon\lesssim \frac{1}{k^2N^3L^2}$. 
And with similar analysis in stage 2, the error altogether is upper bounded by $O(\frac{1}{k^2N^3})$ in infinity norm.

After one step of gradient, we have the softmax score ($\Delta$ is the error term with $\norm{\Delta}_\infty\le \epsilon$.)
\begin{align*}
    \S(\Psi_\ell(1)^\top f^{(\ell)}(X)_{(i,j)}))&=\S\qty(\frac{\eta}{(\exp(\beta_0)+N-1)k}(I_N-\frac{1}{N}\1_N\1_N^\top)e_{N,\hop^{2^{\ell-1}}_i(j)}+\beta_0e_{N,\hop^{2^{\ell-1}}_i(j)}+\eta \Delta)
\end{align*}
The separation between the correct entry and the others are lower bounded by:
$$\frac{\eta}{(\exp(\beta_0)+N-1)k} \frac{N-1}{N}-\eta \|\Delta\|_\infty\gtrsim \frac{\eta}{kN}.$$
By $\eta \gtrsim {k^2N^3\log \frac{kN}{\epsilon}}$, we have $\S(\Psi_\ell^\top f^{(\ell)}(X)_{(i,j)})_{\hop^{2^{\ell-1}}_i(j)}\ge 1-\epsilon$ and thus
\begin{align*}
    \sup_{\sigma,(i,j)}\norm{\S(\Psi_\ell^\top f^{(\ell)}(X)_{(i,j)})-e_{\hop^{2^{\ell-1}}_i(j)}}_\infty\le \epsilon.
\end{align*}
\end{proof}
\subsection{Gradient Upper Bound for Trained Layers}
\label{appendix: gradient upper bound for trained layers}

In this section, we prove the following technical lemma showing that the gradients $\nabla_{W_{KQ}^{(\ell')}}\mathcal{L}^{(\ell)}$ are very small with $\ell'<\ell$ in the $\ell$th stage due to softmax saturation. To be specific, we prove a more general version that also covers the mixed training algorithm.

\begin{lemma}\label{lemma: gradient upper bound for previous layers with curriculum}
    Given a single sample $(X,i,j)$ and
    suppose the training is in stage $\ell_0$, for all $\ell'< \ell_0$,$$\S((\hat{X}^{(\ell'-1)})^\top W_{KQ}^{(\ell')}\hat{X}^{(\ell'-1)}_{(i,j)})_{(i+2^{\ell'-2},\hop^{2^{\ell'-2}}_i(j))} \ge 1-\frac{1}{2}\epsilon,$$
    and for $\ell''>\ell_0$ the norm of the parameter $\|W_{KQ}^{(\ell'')}\|_2\lesssim \frac{1}{k^2N^2L^2}$\footnote{This condition holds for both curriculum training and mix training. In particular, the norm is exactly zero for curriculum training \Cref{alg:training_alg}.}.
    Then we have for any $\ell'<\ell_0$, the infinity norm of the $\ell_0$th stage gradient is upper bounded by
    \begin{align*}
        \|\nabla_{W_{KQ}^{(\ell')}} \mathcal{L}^{(\ell)}\|_\infty \le 6\beta_0k(kNL)^{3/2}\epsilon, \forall \ell>\ell'.
    \end{align*}
\end{lemma}


\begin{proof}
    Recall the gradient for a single sample $(X,i,j)$
    \begin{align*}
    \nabla_{W_{KQ}^{(\ell')}} \mathcal{L}^{(\ell)} = -\qty[\sum_{s'\in[N]}\qty(\mathbf{1}\{s'=\hop^{2^{\ell-1}}_i(j)\}-\S(\Psi_\ell^\top f^{(\ell)}_{(i,j)})_{s'}) \nabla_{W_{KQ}^{(\ell')}} (\Psi_\ell^\top f^{(\ell)}_{(i,j)})_{s'}]
\end{align*}
So the norm of the gradient is upper bounded by 
\begin{align*}
    \|\nabla_{W_{KQ}^{(\ell')}} \mathcal{L}^{(\ell)}\| &\le\left\|\sum_{s'\in[N]}\qty(\mathbf{1}\{s'=\hop^{2^{\ell-1}}_i(j)\}-\S(\Psi_\ell^\top f^{(\ell)}_{(i,j)})_{s'}) \nabla_{W_{KQ}^{(\ell')}} (\Psi_\ell^\top f^{(\ell)}_{(i,j)})_{s'}\right\|\\
    &\le 2\norm{\nabla_{W_{KQ}^{(\ell')}} (\Psi_\ell^\top f^{(\ell)}_{(i,j)})_{s'}} = 2\norm{\nabla_{W_{KQ}^{(\ell')}} (e_{N,s'}^\top\Psi_\ell^\top f^{(\ell)}_{(i,j)})}.
\end{align*}
Now we calculate the gradient recursively through Taylor expansion: for all $\Delta W$ with $\|\Delta W\|\to 0$, 
\begin{align*}
    f^{(\ell)}_{(i,j)}(W_{KQ}^{(\ell')}+\Delta W)=f^{(\ell)}_{(i,j)}(W_{KQ}^{(\ell')})+\langle\nabla_{W_{KQ}^{(\ell')}} (f^{(\ell)}_{(i,j)}),\Delta W\rangle + O(\|\Delta W\|^2_F).
\end{align*}
While by Taylor expansion, we have (we ignore $(W_{KQ}^{(\ell')})$ when there is no perturbation $\Delta W$.)
\begin{align*}
    f^{(\ell)}_{(i,j)}(W_{KQ}^{(\ell')}+\Delta W)&=f^{(\ell)}_{(i,j)}(W_{KQ}^{(\ell')})+W_{OV}^{(\ell)}\nabla_{W_{KQ}^{(\ell')}} X^{(\ell-1)}(\Delta W)\S((X^{(\ell-1)})^\top W_{KQ}^{(\ell)}X_{(i,j)}^{(\ell-1)})\\
    &\quad +W_{OV}^{(\ell)}X^{(\ell-1)}J^{(\ell)}\nabla((X^{(\ell-1)})^\top W_{KQ}^{(\ell)}X_{(i,j)}^{(\ell-1)})(\Delta W)+O(\norm{\Delta W}^2_F).
\end{align*}
Here the gradient $\nabla_{W_{KQ}^{(\ell')}} X^{(\ell-1)\in \R^{d\times kN\times d\times d}}$ is a 4-th order tensor, and $\nabla_{W_{KQ}^{(\ell')}} X^{(\ell-1)}(\Delta W)\in \R^{d\times kN}$.
The second term can be expanded into
\begin{align*}
    W_{OV}^{(\ell)}X^{(\ell-1)}J^{(\ell)}\nabla((X^{(\ell-1)})(\Delta W)^\top )W_{KQ}^{(\ell)}X_{(i,j)}^{(\ell-1)} + W_{OV}^{(\ell)}X^{(\ell-1)}J^{(\ell)}(X^{(\ell-1)})^\top W_{KQ}^{(\ell)}\nabla(X_{(i,j)}^{(\ell-1)})(\Delta W).
\end{align*}

In this way, we reduce the problem to calculating the norm upper bound for $\nabla_{W_{KQ}^{(\ell')}} X^{(\ell-1)}(\Delta W)$. We upper bound the infinity norm of the matrix by induction from layer $\ell'$ to $\ell$. We prove that $$ \|\nabla_{W_{KQ}^{(\ell')}} X^{(t)}(\Delta W)\|_\infty\le (kNL)^{3/2}\epsilon(1+\frac{1}{k})^{t-\ell'}\|\Delta W\|_F$$ with $t\in [\ell',\ell-1]$.

\textbf{Base case: $t = \ell'$} By Taylor expansion on $X^{(\ell')}$, we have 
\begin{align*}
    X^{(\ell')}(W_{KQ}^{(\ell')}+\Delta W) &= X^{(\ell')} +W_{OV}^{(\ell')}X^{(\ell'-1)}J^{(\ell')}(X^{(\ell'-1)})^\top \Delta WX^{(\ell'-1)})+O(\norm{\Delta W}^2_F).
\end{align*}
since previous layers are independent of $W_{KQ}^{(\ell')}$. Therefore, the first order term is
\begin{align*}
    \|\nabla_{W_{KQ}^{(\ell')}} X^{(\ell')}(\Delta W)\| = \norm{W_{OV}^{(\ell')}X^{(\ell'-1)}J^{(\ell')}(X^{(\ell'-1)})^\top \Delta WX^{(\ell'-1)})}
\end{align*}
Note that the softmax is close to one-hot,
$$\S((\hat{X}^{(\ell'-1)})^\top W_{KQ}^{(\ell')}\hat{X}^{(\ell'-1)}_{(i,j)})_{(i+2^{\ell'-2},\hop^{2^{\ell'-2}}_i(j))} \ge 1-\frac{1}{2}\epsilon,$$
we have the Jacobian $\norm{J^{(\ell)}}\lesssim \epsilon.$ Therefore, we can upper bound the first order term by
\begin{align*}
    \|\nabla_{W_{KQ}^{(\ell')}} X^{(\ell')}(\Delta W)\|_2 &\le \norm{W_{OV}^{(\ell')}}\norm{X^{(\ell'-1)}}_F\norm{J^{(\ell')}}_2\norm{(X^{(\ell'-1)})^\top}_F\norm{ \Delta W}\norm{X^{(\ell'-1)})}_F\\&\lesssim (kNL)^{3/2}\epsilon\|\Delta W\|_F.\tag{$\norm{X^{(\ell'-1)}}_F^2\le O(kNL)$ since each embedding is either $\epsilon$-close to one-hot or all 0.}
\end{align*}
Since $ \|\nabla_{W_{KQ}^{(\ell')}} X^{(\ell')}(\Delta W)\|_\infty\le  \|\nabla_{W_{KQ}^{(\ell')}} X^{(\ell')}(\Delta W)\|_2$, we finish the proof for base case.

\textbf{Induction: $t\in [\ell', \ell-1]$.} Suppose the induction hypothesis holds:
$$\|\nabla_{W_{KQ}^{(\ell')}} X^{(t)}(\Delta W)\|_\infty\le (kNL)^{3/2}\epsilon(1+\frac{1}{k})^{t-\ell'}\|\Delta W\|_F.$$
We consider expanding $X^{(t+1)}(W_{KQ}^{(\ell')}+\Delta W)$ like the base case:
\begin{align*}
    X^{(t+1)}(W_{KQ}^{(\ell')}+\Delta W)&=X^{(t)}(W_{KQ}^{(\ell')}+\Delta W) + f^{(t+1)}(W_{KQ}^{(\ell')}+\Delta W)\\
    &=X^{(t+1)} + \nabla_{W_{KQ}^{(\ell')}} X^{(t)}(\Delta W)\tag{From $X^{(t)}(W_{KQ}^{(\ell')}+\Delta W)$.}\\
    &\ + W_{OV}^{(t+1)} \nabla_{W_{KQ}^{(\ell')}} X^{(t)}(\Delta W) \S\qty(\qty(X^{(t)})^\top W_{KQ}^{(t+1)}X^{(t)})\\
    &\ + W_{OV}^{(t+1)} X^{(t)} J^{(t+1)}\qty(\nabla_{W_{KQ}^{(\ell')}} X^{(t)}(\Delta W))^\top W_{KQ}^{(t+1)}X^{(t)} \\
    &\ + W_{OV}^{(t+1)} X^{(t)} J^{(t+1)}\qty( X^{(t)})^\top W_{KQ}^{(t+1)}\nabla_{W_{KQ}^{(\ell')}}X^{(t)}(\Delta W) + O(\norm{\Delta W}_F^2)\tag{The last 3 terms are from $f^{(t+1)}(W_{KQ}^{(\ell')}+\Delta W)$.}
\end{align*}
Note that $\nabla_{W_{KQ}^{(\ell')}} X^{(t)}(\Delta W)$ and the rest three terms are in different rows since they uses different $W_{OV}^{(t)}$, so the infinity norm of the first order term should be the maximum of those two. 

We first upper bound the last three terms. Since $W_{OV}^{(t+1)}$ is partial identity, and softmax is some weighted average, we have
\begin{align*}
\norm{W_{OV}^{(t+1)} \nabla_{W_{KQ}^{(\ell')}} X^{(t)}(\Delta W) \S\qty(\qty(X^{(t)})^\top W_{KQ}^{(t+1)}X^{(t)})}\le \norm{\nabla_{W_{KQ}^{(\ell')}} X^{(t)}(\Delta W)}_\infty.
\end{align*}
The rest two terms can be directly upper bounded by
$$\norm{W_{OV}^{(t+1)}}_2\norm{X^{(t)}}_2 \norm{J^{(t+1)}}_2\norm{\nabla_{W_{KQ}^{(\ell')}} X^{(t)}(\Delta W))}_2\norm{W_{KQ}^{(t+1)}}_2\norm{X^{(t)}}_2.$$
In previous stages, we either have (1) $t$th layer is trained: $\norm{X}_F^2\le kNL$, $\|J^{(t+1)}\|_2\le \epsilon$ and $\|W_{KQ}^{(t+1)}\|_2\lesssim N\log \frac{kN}{\epsilon}$, or (2) $t$th layer is close to initialization: $\|J^{(t+1)}\|_2\le 1$ and $\|W_{KQ}^{(t+1)}\|_2\lesssim \frac{1}{kNL}$. So these two terms can be both upper bounded by 
$$kN^2L\epsilon\log\frac{kN}{\epsilon}\norm{\nabla_{W_{KQ}^{(\ell')}} X^{(t)}(\Delta W))}_2\le \frac{1}{k}\norm{\nabla_{W_{KQ}^{(\ell')}} X^{(t)}(\Delta W))}_\infty$$
since $\epsilon \log \frac{1}{\epsilon}\le k^6N^6L^6.$ Combining two error terms, we have
$$\norm{\nabla_{W_{KQ}^{(\ell')}} X^{(t+1)}(\Delta W)}_\infty\le(kNL)^{3/2}\epsilon(1+\frac{1}{k})^{t-\ell'+1}\|\Delta W\|_F.$$

By induction, when $t = \ell-1$ we have (since $\ell-\ell'<k$.) 
$$\norm{\nabla_{W_{KQ}^{(\ell')}} X^{(\ell-1)}(\Delta W)}_\infty\le(kNL)^{3/2}\epsilon(1+\frac{1}{k})^{\ell-\ell'}\|\Delta W\|_F\le 3(kNL)^{3/2}\epsilon \|\Delta W\|_F.$$

Finally, we can upper bound $\|\nabla_{W_{KQ}^{(\ell')}} \mathcal{L}^{(\ell)}\|$ by picking $\Delta W = \alpha \nabla_{W_{KQ}^{(\ell')}} \mathcal{L}^{(\ell)}$ with $\alpha\to 0$:
\begin{align*}
    \alpha\norm{\nabla_{W_{KQ}^{(\ell')}} \mathcal{L}^{(\ell)}}^2_F&\le 2\left\langle\nabla_{W_{KQ}^{(\ell')}} (e_{N,s'}^\top\Psi_\ell^\top f^{(\ell)}_{(i,j)}),\Delta W\right\rangle\\
    &\le 2e_{N,s'}^\top\Psi_\ell^\top W_{OV}^{(\ell)}\nabla_{W_{KQ}^{(\ell')}} X^{(\ell-1)}(\Delta W)\S((X^{(\ell-1)})^\top W_{KQ}^{(\ell)}X_{(i,j)}^{(\ell-1)}) + O(\alpha^2)\\
    &\le 2\beta_0 k\norm{\nabla_{W_{KQ}^{(\ell')}} X^{(\ell-1)}(\Delta W)}_\infty \le 6\beta_0k(kNL)^{3/2}\epsilon\|\Delta W\|_F. 
\end{align*}
Plug in $\Delta W = \alpha \nabla_{W_{KQ}^{(\ell')}} \mathcal{L}^{(\ell)}$, we have 
$\norm{\nabla_{W_{KQ}^{(\ell')}} \mathcal{L}^{(\ell)}}_F \le6\beta_0k(kNL)^{3/2}\epsilon. $
\end{proof}

\section{Training on a Mixture of Different Hops}
\label{appendix:mixture}
Recall that the parameter vector is $\theta := (\theta_{KQ},\theta_\Psi)$, where $\theta_{KQ}=(W_{KQ}^{(1)},\cdots,W_{KQ}^{(L)})$ and $\theta_\Psi=(\Psi_1,\cdots,\Psi_{L})$, and that we consider the following mixed training loss, which is a summation of the loss on all $2^\ell$-hops:
\begin{equation*}
\mathcal{L}^{\mathrm{M}}(\theta) := \sum_{\ell=1}^L \mathcal{L}^{(\ell)}(\theta)= -\frac{1}{M}\sum_{\ell=1}^L\sum_{m=1}^M\qty[\sum_{s'\in[N]}\mathbf{1}\{s'=\hop^{2^{\ell-1}}_{i_m}(\sigma^{(m)}, j_m)\}\log(\S(\Psi_\ell^\top \mathrm{TF}_\theta(X_m)_{(i_m,j_m)})_{s'})]
\end{equation*}
The learning algorithm proceeds by taking $L$ gradient descent steps on the key-query matrices $\theta_{KQ}$, followed by a single gradient descent step on the readout layer $\theta_\Psi$. Pseudocode is given in \Cref{alg:training_alg_mix}.



        
        
        
                
        
        
We restate the theorem for mixed training below. Here we require that the error is sufficiently small s.t. $0<\epsilon\le \tilde{O}(\frac{1}{k^6N^6})$, $\epsilon \log^{2L}\frac{1}{\epsilon}\le 1$, which is needed to keep the accumulation error and gradient error small in the proof.
\mixed*

\subsection{Proof Outline}
We provide a proof outline in this section, and defer the technical lemmas to the later subsections.

\begin{proof}[Proof of \Cref{thm:mix-data-main}]
    We prove that the population gradient dynamics for $\theta_{KQ}$ is identical to the curriculum dynamics in \Cref{thm:main}. 
    On the population trajectory, we first show the following key observation: \textbf{when the previous layer is not trained and stays zero, all later layers have zero gradient and also stay zero.} Therefore, the mixed training algorithm induces the same implicit curriculum training as \Cref{alg:training_alg}.

    We first recall from \Cref{subsec:grad_computation}, the main signal term, i.e. the gradient for the $\ell$th layer $W_{KQ}^{(\ell)}$ over loss $\mathcal{L}^{(\ell)}$ is
    \begin{align*}
        \nabla_{W_{KQ}^{(\ell)}} \mathcal{L}^{(\ell)}=-\E\qty[\sum_{s'\in[N]}\qty(\mathbf{1}\{s'=\hop_i^{2^{\ell-1}}(j)\}-\S(\Psi_\ell^\top f^{(\ell)}(X^{(\ell-1)})_{(i,j)})_{s'}) XJ^{(\ell)} X^\top {(\Psi_\ell^\top W_{OV}^{(\ell)})_{s'}}^\top {X}_{(i,j)}^\top].
    \end{align*}
    Suppose that we are in the $t$th step of gradient descent, so that the first untrained key-query matrix is the $t$th layer $W_{KQ}^{(t)}$.
    We will define $X^{(\ell'-1)}_{t-1}$ to be the ``ideal'' input to the $\ell'$th layer after $t-1$ gradient steps $(\ell' > t)$; the $(i, j)$ column is given by
    \begin{align*}
        X^{(\ell'-1)}_{t-1, (i,j)} := \begin{bmatrix}
            e_{k,i}\otimes e_{N,j}\\
            e_{k,i}\otimes e_{N,\sigma_i(j)}\\
            e_{k,i}\otimes e_{N,\hop^1_i(j)}\\
            e_{k,i+1}\otimes e_{N,\hop^2_i(j)}\\
            \vdots\\
            e_{k,i+2^{t-3}-1}\otimes e_{N,\hop^{2^{t-3}}_i(j)}\\
            \frac{1}{kN}\mathbf{1}_{kN}\\
            \vdots\\
            \frac{1}{kN}\mathbf{1}_{kN}\\
            0_{(L-\ell'+1)kN}
        \end{bmatrix}.
    \end{align*}
    Therefore, we have 
    \begin{align*}
        J^{(\ell')}(X^{(\ell'-1)}_{t-1})^\top {(\Psi_{\ell'}^\top W_{OV}^{(\ell')})_{s'}}^\top=\beta_0 \cdot J^{(\ell')}(X^{(\ell'-1)}_{t-1})^\top  e_{L+2,\ell'+1}\otimes \mathbf{1}_k \otimes e_{N,s'}=J^{(\ell')}\cdot k\1_{kN}=0,
    \end{align*}
    since the Jacobian when $W_{KQ}^{(\ell')} = 0$ is $J^{(\ell')} = \frac{1}{kN}(I-\frac{1}{kN}\1_{kN}\1_{kN}^\top)$. Therefore the gradient $\nabla_{W_{KQ}^{(\ell')}}\mathcal{L}^{(\ell')}$ on the ideal input is indeed equal to zero.
    

    In mixed training, the nonsignal terms $\nabla_{W_{KQ}^{(\ell')}}\mathcal{L}^{(\ell)}$ for $\ell'\not=\ell$ can also play a role. For the case $t \le \ell' < \ell$, we can generalize the above argument to compute the gradient $\nabla_{W_{KQ}^{(\ell')}}\mathcal{L}^{(\ell)}$ on the ideal input; the essential blocks extracted by $W_{OV}^{(\ell)}$ in the sequence are still $\frac{1}{kN}\mathbf{1}_{kN}$, which is independent of $W_{KQ}^{(\ell')}$, and thus the gradient on the ideal inputs is still zero. For the case $\ell' < t$, the softmax in layer $\ell'$ is fully saturated, and thus the gradient on ideal inputs is also zero. Finally, when $\ell' > \ell$, the gradient is trivially zero as well. Altogether in the $t$th step of the population dynamics, the only nonzero gradient (with ideal inputs) is $\nabla_{W^{(t)}_{KQ}}\mathcal{L}^{(t)}$, thus mimicking the curriculum dynamics. 
    



    Now we only need to upper bound the error that deviates from the population dynamics, including the sample noise, non-saturation error, and the additional error introduced during the accumulation when proprogating through layers. Compared to the proof of \Cref{thm:main}, there are \textbf{two additional source of error:}
    \begin{enumerate}
        \item During the $t$th gradient step, layers $\ell < t$ will still be updated. Since the softmax in layer $\ell$ is nearly saturated and thus close to one-hot, the update norm is very small and can be bounded as noise terms.
        \item Due to the non-saturation error, the input is not ideal one-hot vectors for each hop. That means in the $t$th gradient step, there will be some small gradient updates for later layers $\ell' \ge t$ introduced by the perturbation. We can also upper bound this amount of noise.
    \end{enumerate}

    The full proof proceeds as follows. In \textbf{stage 1} (\Cref{sec:mix-train-stage-1}), we can use similar strategy as in \Cref{appendix:upper_bound} with curriculum. The proof is the same, since the later layers are not updated in this stage. \textbf{Note that now the sequence $\hat{X}^{(\ell)}$ evolves with time}, which is different from the analysis for \Cref{alg:training_alg}. We denote $\hat{X}^{(\ell)}_t$ as the intermediate sequence at time $t$, and each column as $\hat{X}^{(\ell)}_{t,(i,j)}.$

    However, the gradient of the later layers $(\ell\ge 2)$ after the first step ($t \ge 2$) is nonzero, because the first layer's attention pattern is not exactly one-hot. The error $\epsilon'$ introduced by softmax non-saturation accumulates to later layers, causing unwanted updates. Since the step-size is very large, each gradient step may introduce $\approx kNL\epsilon'\log\frac{1}{\epsilon}$ error, which means the accumulation error grows exponentially in the number of gradient steps ($t\le L=O(\log k)$). Thus, we have to ensure that the error introduced by softmax is very small. In particular, we pick the error $\epsilon'=O(\frac{\epsilon}{(kNL)^{6L}})$; fortunately, this only requires the step size to be $\Omega(\frac{k^2N^3}{\beta_0}\log k\log \frac{kN}{\epsilon})$, which still satisfies the requirement in the theorem.
    
    In later stages $t\ge 2$ (\Cref{sec:mix-train-stage-2} for stage 2, \cref{sec:mix-train-stage-ell} for stage $t\ge 3$), we will inductively prove that $\norm{\hat{X}_t^{(\ell)}-{X}_t^{(\ell)}}\lesssim \frac{\epsilon\log^{t-1} \frac{1}{\epsilon}}{(kNL)^{6L-6t+6}}.$ Altogether, for the last stage $L$, the error can be bounded by $O(\frac{1}{k^6N^6L^6})$, which guarantees that the true transformer output $\hat X_L^{(L)}$ stores the correct $2^{\ell-1}$-hop information, as in the proof of \Cref{thm:main}. 
    
    To conclude, further training the readout layer leads to learning all the $2^{\ell-1}$-hop tasks, as desired (\Cref{lemma: readout for ellth layer mix training}).
\end{proof}

\subsection{Stage 1 for Mixed Training}\label{sec:mix-train-stage-1}
We first prove that in \textbf{stage 1}, the first layer $W_{KQ}^{(1)}$ learns all the hidden permutations. The proof is the same as the proof of \Cref{thm:main}, because there is no additional noise introduced in this stage by zero initialization.

Using \Cref{lemma: empirical gradient of the first stage}, after the first step gradient we have for any $(i,j)$, the softmax probability satisfies 
    $$\S^{(1)}_{(i,\pi_i(j))}:=\S(({\hat{X}^{(0)}})^\top W_{KQ}^{(1)}\hat{X}^{(0)}_{(i,j)})_{(i,\pi_i(j))} \ge 1-\frac{\epsilon}{2(kNL)^{6L}}.$$
    with a $\log(kN)$ larger learning rate. Note that $\log (kNL)^L =O(\log k\log kNL)\lesssim \poly\log(kN),$ so it falls in the required learning rate range. Before the softmax, we have the separation of the $(i,\pi_i(j))$ position and the others $\gtrsim \log k\log\frac{kN}{\epsilon}.$
    
\textbf{Note that now the sequence $\hat{X}^{(\ell)}$ evolves with time}. Recall $\hat{X}^{(\ell)}_t$ as the intermediate sequence at time $t$, and each column as $\hat{X}^{(\ell)}_{t,(i,j)}.$
We can then calculate the intermediate sequence $\hat{X}^{(1)}_1$:
\begin{align*}
    \hat{X}^{(1)}_1 = X^{(0)} + W_{OV}^{(1)}X^{(0)}\S^{(1)} = X^{(1)}_1 + W_{OV}^{(1)}X^{(0)}{(\S^{(1)}-\S^{(1)}_{\mathrm{ideal}})}.
\end{align*}
while the ideal input sequence for the second layer is
\begin{align*}
    X^{(1)}_1 = \begin{pNiceArray}{ccc|c|ccc}
 e_{k,1}\otimes e_{N,1}&\cdots&e_{k,1}\otimes e_{N,N}&\cdots&e_{k,k}\otimes e_{N,1}&\cdots&e_{k,1}\otimes e_{N,N}\\
e_{k,1}\otimes e_{N,\sigma_1(1)}&\cdots& e_{k,1}\otimes e_{N,\sigma_1(N)}&\cdots&e_{k,k}\otimes e_{N,\sigma_k(1)}&\cdots&e_{k,k}\otimes e_{N,\sigma_k(N)}\\
e_{k,1}\otimes e_{N,\sigma_1\pi_1(1)}&\cdots& e_{k,1}\otimes e_{N,\sigma_1\pi_1(N)}&\cdots&e_{k,k}\otimes e_{N,\sigma_k\pi_k(1)}&\cdots&e_{k,k}\otimes e_{N,\sigma_k\pi_k(N)}\\
0_{(L-1)kN} &\cdots&0_{(L-1)kN}&\cdots&0_{(L-1)kN}&\cdots&0_{(L-1)kN}
\end{pNiceArray}
\end{align*}
By \Cref{lemma: perturbation for first stage}, $\|X_1^{(1)}-\hat{X}_1^{(1)}\|_\infty\le \frac{\epsilon}{(kNL)^{6L}}$. Further, we need to bound the distance between the ideal inputs for later layers ${X}_1^{(\ell)}$
\begin{align*}
        X^{(\ell)}_{1, (i,j)} = \begin{bmatrix}
            e_{k,i}\otimes e_{N,j}\\
            e_{k,i}\otimes e_{N,\sigma_i(j)}\\
            e_{k,i}\otimes e_{N,\hop^1_i(j)}\\
            \frac{1}{kN}1_{kN}\\
            \vdots\\
            \frac{1}{kN}1_{kN}\\
            0_{(L-\ell)kN}
        \end{bmatrix}   
    \end{align*}
and the actual input $\hat{X}_1^{(\ell)}$. The following lemma exhibits the upper bound.
\begin{lemma}[Perturbation analysis of stage 1 for $X_1^{(\ell)}$]\label{lemma: perturbation for first stage, later layer}
    Under the condition of Theorem~\ref{thm:mix-data-main}, we have for all $\ell\ge 2$, 
    $$\|X_1^{(\ell)}-\hat{X}_1^{(\ell)}\|\le \frac{\epsilon}{(kNL)^{6L}},$$
    where $X_1^{(\ell)}$ is the ideal output with saturated softmax, and $\hat{X}_1^{(\ell)}$ is the transformer output.
\end{lemma}
\begin{proof}
    The ideal output for the $\ell$ layer is $$X^{(\ell)} = X^{(\ell-1)} + W_{KQ}^{(\ell)}X^{(\ell-1)}\S^{(\ell)}_{\mathrm{ideal}},$$ where $\S^{(\ell)}_{\mathrm{ideal}}$ is the ideal one-hot softmax attention pattern. The empirical output $\hat{X}^{(\ell)}_1$ has the error introduced by the non-saturation of the softmax, together with the previous error in $\hat{X}_1^{(\ell-1)}$:
\begin{align*}
    \hat{X}_1^{(\ell)} &= \hat{X}_1^{(\ell-1)} + W_{OV}^{(\ell)}\hat{X}_1^{(\ell-1)}\S^{(\ell)}\\ &= X_1^{(\ell)} + W_{OV}^{(\ell)}X_1^{(\ell-1)}\underbrace{(\S^{(\ell)}-\S^{(\ell)}_{\mathrm{ideal}})}_{\Delta \S^{(\ell)}, \text{non-uniform error}}+\underbrace{(\hat{X}^{(\ell-1)}-{X}^{(\ell-1)})+W_{OV}^{(\ell)}(\hat{X}^{(\ell-1)}-{X}^{(\ell-1)})\S^{(\ell)}}_{\text{Accumulated perturbation error}}.
\end{align*}
We first consider the non-uniform error term. Since the $W_{KQ}^{(\ell)}$ is not updated for $\ell\ge 2$,
$$\|W_{OV}^{(\ell)}X_1^{(\ell-1)}(\S^{(\ell)}-\S^{(\ell)}_{\mathrm{ideal}})\|_\infty  = 0.$$
Now consider the accumulated perturbation error. By the perturbation analysis inductively, we have $\norm{\hat{X}_1^{(\ell-1)}-{X}_1^{(\ell-1)}}_{\infty}\le \frac{\epsilon}{(kNL)^{6L}}.$ Note that the error in $\hat{X}_1^{(\ell-1)}-{X}_1^{(\ell-1)}$ are in different rows of the matrices from 
$W_{OV}^{(\ell)}(\hat{X}_1^{(\ell-1)}-{X}_1^{(\ell-1)})\S^{(\ell)}$. Therefore, $\hat{X}_1^{(\ell-1)}-{X}_1^{(\ell-1)}$ won't introduce extra error in this stage, and we only need to consider $W_{OV}^{(\ell)}(\hat{X}_1^{(\ell-1)}-{X}_1^{(\ell-1)})\S^{(\ell)}$:
\begin{align*}
    \norm{W_{OV}^{(\ell)}(\hat{X}_1^{(\ell-1)}-{X}_1^{(\ell-1)})\S^{(\ell)}}_\infty &= \max_{s,(i,j)}\qty|(W_{OV}^{(\ell)}(\hat{X}_1^{(\ell-1)}-{X}_1^{(\ell-1)}))_s^\top \S^{(\ell)}_{(i,j)}|\\
    &=\max_{s,{(i,j)}}\qty|\sum_{p=1}^{kN}(W_{OV}^{(\ell)}(\hat{X}_1^{(\ell-1)}-{X}_1^{(\ell-1)}))_{s,p} (\S^{(\ell)}_{(i,j)})_p|\\
    &\le \|\hat{X}_1^{(\ell-1)}-{X}_1^{(\ell-1)}\|_\infty\tag{Since $\sum_p (\S^{(\ell)}_{(i,j)})_p=1, (\S^{(\ell)}_{(i,j)})_p\ge0$.}
\end{align*}
Combine both parts of error, we have $\norm{\hat{X}_1^{(\ell)}-{X}_1^{(\ell)}}_\infty \le \frac{\epsilon}{(kNL)^{6L}}.$
\end{proof}
Thus after stage 1, the intermediate input sequences $\hat{X}_1^{(\ell)}$ are $\frac{\epsilon}{(kNL)^{6L}}$-close to the ideal sequence. 

\subsection{Stage 2 for Mixed Training} \label{sec:mix-train-stage-2}
After the first stage, $W_{KQ}^{(\ell)}$ for all $\ell\ge 2$ are not updated and remain zero. Given that $\norm{\hat{X}_1^{(\ell)}-{X}_1^{(\ell)}}_\infty\le \frac{\epsilon}{(kNL)^{6L}}$, \Cref{lemma: empirical gradient of the second stage} still applies and we have
$$\S^{(2)}_{(i+1,\hop^1_i(j))}:=\S\qty(({\hat{X}^{(1)}})^\top W_{KQ}^{(2)}\hat{X}^{(1)}_{(i,j)})_{(i+1,\hop^1_i(j))} \ge 1-\frac{\epsilon}{2(kNL)^{6L}}.$$
However, the first layer and the later layers with $\ell\ge 3$ are \textbf{all updated in the second stage} due to the perturbed inputs on the mixture of training data. The following lemma bounds the deviation of the gradient steps, making sure the empirical gradients stay close to the population dynamics. 
\begin{lemma}\label{lemma: gradient perturbation in stage 2}
    In stage 2, given that $\norm{\hat{X}_1^{(\ell)}-{X}_1^{(\ell)}}_\infty\le \frac{\epsilon}{(kNL)^{6L}}$ for all gradients $\nabla_{W_{KQ}^{(\ell)}}\mathcal{L}^{(\ell')}$ for $\ell\not= 2$ or $\ell'\not=2$, the gradients can be upper bounded by 
    $$\|\nabla_{W_{KQ}^{(\ell)}}\mathcal{L}^{(\ell')}\|_\infty\lesssim \frac{\beta_0\epsilon}{(kNL)^{6L-2.5}}.$$
\end{lemma}
\begin{proof}
    We first bound the possible signal term gradients $\nabla_{W_{KQ}^{(i)}}\mathcal{L}^{(i)}$, $i\not=2$. We prove that the gradient norm for $W_{KQ}^{(\ell)}$ for $\ell\not= 2$ can be upper bounded by
$$\|\nabla_{W_{KQ}^{(1)}}\mathcal{L}^{(1)}\|_\infty\lesssim \frac{\beta_0\epsilon}{(kNL)^{6L-1}},\quad \|\nabla_{W_{KQ}^{(\ell)}}\mathcal{L}^{(\ell)}\|_\infty\lesssim \frac{\beta_0L\epsilon}{(kNL)^{6L}}.$$

Recall the gradient of $\ell$th layer on the sample $X_m$ is:
    \begin{align*}
        \nabla_{W_{KQ}^{(\ell)}}\mathcal{L}^{(\ell)} = \sum_{s'\in[N]}(\mathbf{1}\{s'=\hop_{i_m}^{2^{\ell-1}}(j_m)\}-\mathcal{P}^{(\ell)}_{s'}) \hat{X}^{(\ell-1)}_m J^{(\ell)} (\hat{X}^{(\ell-1)}_m)^\top {\qty(\Psi_\ell^\top W_{OV}^{(\ell)})_{s'}}^\top (\hat{X}^{(\ell-1)}_{(i_m,j_m)})^\top
    \end{align*}
    For $\ell = 1$, since the softmax is close to one-hot, we have the Jacobian
    $$\|J^{(1)}\|_2 = \|\diag(\S)-\S\S^\top\|_2\lesssim \frac{\epsilon}{(kNL)^{6L}}.$$
    Therefore the first layer gradient has an infinity norm upper bound (since $\|\hat{X}^{(\ell)}_m\|_\infty\le 1$.)
    $$\|\nabla_{W_{KQ}^{(1)}}\mathcal{L}^{(1)}\|_\infty\le 2d\norm{J^{(1)}}_2 \cdot \beta_0\cdot 1\lesssim \frac{\beta_0\epsilon}{(KNL)^{6L-1}}.$$

    For $\ell\ge 3$, the ideal empirical gradient is zero. The perturbation satisfies $\norm{\hat{X}_1^{(\ell)}-{X}_1^{(\ell)}}_\infty \le \frac{\epsilon}{(kNL)^{6L}}.$ With the same proof strategy in \Cref{lemma: concentration for stage l}, we compare the idealized empirical gradient and actual empirical gradient. Here all key-query matrices are not updated, so the Jacobian is still $\frac{1}{kN}(I-\frac{1}{kN}\1\1^\top)$. The actual empirical gradient is (we ignore the layer number $^{(\ell)}$ here) 
\begin{align*}
    \hat{g}_{\ell} 
    &=\frac{\beta_0}{kN}\hat{X}_1 (I-\frac{1}{kN}\1\1^\top)\hat{X}_1^\top (e_{L+2,\ell+1}\otimes \1_k\otimes e_{N,\hop^{2^{\ell-1}}_{i}(j)}) (\hat{X}^{(\ell-1)}_{1,(i,j)})^\top\\&-\frac{\beta_0}{kN}\sum_{s'\in[N]}\S_{s'}\hat{X}_1 (I-\frac{1}{kN}\1\1^\top)\hat{X}_1^\top (e_{L+2,\ell+1}\otimes \1_k\otimes e_{N,s'}) (\hat{X}^{(\ell-1)}_{1,(i,j)})^\top.
\end{align*}
Denote the following term as
$$\hat{\gamma}_{s';\ell}=\hat{X}_1 (I-\frac{1}{kN}\1\1^\top)\hat{X}_1^\top (e_{L+2,\ell+1}\otimes \1_k\otimes e_{N,s'}) (\hat{X}^{(\ell-1)}_{1,(i,j)})^\top,$$
$${\gamma}_{s';\ell}={X}_1 (I-\frac{1}{kN}\1\1^\top){X}_1^\top (e_{L+2,\ell+1}\otimes \1_k\otimes e_{N,s'}) ({X}^{(\ell-1)}_{1,(i,j)})^\top.$$
and we define $\Delta \gamma_{s';\ell}=\hat{\gamma}_{s';\ell}-\gamma_{s';\ell}$. Then we can rewrite the empirical gradient into
\begin{align*}
    \hat{g}_{\ell}-\hat{g}^{\mathrm{ideal}}_{\ell} 
    &=\Delta\gamma_{\hop^{2^{\ell-1}}(v_m),m}-\sum_{s'\in[N]}\S_{s'}\Delta\gamma_{s'}.
\end{align*}
We have the perturbation error upper bounded by:
\begin{equation}
    \norm{\hat{g}_{\ell}-\hat{g}^{\mathrm{ideal}}_{\ell}}_\infty \le \|\Delta \gamma_{\hop;\ell}\|+ \sum_{s'}\S_{s'}\|\Delta \gamma_{s';\ell}\|.
\end{equation}
The error of the following difference $$\|\Delta \gamma_{s';\ell}\|_\infty=\|\gamma_{s';\ell}-\hat{\gamma}_{s';\ell}\|_\infty\le Cd\norm{\hat{X}_1-X_1}_\infty.$$
with some absolute constant.
    Thus the gradient is upper bounded by 
    $$\|\nabla_{W_{KQ}^{(\ell)}}\mathcal{L}^{(\ell)}\|_\infty\lesssim \frac{\beta_0}{kN}\cdot\frac{d\epsilon}{(kNL)^{6L}} \lesssim \frac{\beta_0L\epsilon}{(kNL)^{6L}}.$$

Next, we bound the gradients $\nabla_{W_{KQ}^{(\ell')}}\mathcal{L}^{(\ell)}$ with $\ell'<2$, i.e. $\ell'=1$.
By \Cref{lemma: gradient upper bound for previous layers with curriculum} and $\S^{(2)}_{(i+1,\hop^1_i(j))}\ge 1-\frac{\epsilon}{2(kNL)^{6L}}$, we have 
$$\norm{\nabla_{W_{KQ}^{(\ell')}} \mathcal{L}^{(\ell)}}_\infty \le6\beta_0k(kNL)^{3/2}\cdot\frac{\epsilon}{(kNL)^{6L}} \lesssim \beta_0\cdot\frac{\epsilon}{(kNL)^{6L-2.5}} .$$

Finally, we bound the gradients $\nabla_{W_{KQ}^{(\ell')}}\mathcal{L}^{(\ell)}$ with $\ell>\ell'\ge2$. Similar to \Cref{lemma: gradient upper bound for previous layers with curriculum}, we can derive a general upper bound for those gradients; this is deferred to \Cref{lemma: gradient upper bound for previous layers with mixing} in \Cref{appendix: general upper bound for non-signal terms in mix training}. 
We have the upper bound with $\ell_0=1$:
$$\norm{\nabla_{W_{KQ}^{(\ell')}} \mathcal{L}^{(\ell)}}_\infty \le6\beta_0\cdot\frac{\epsilon\log^{\ell_0-1}\frac{1}{\epsilon}}{(kNL)^{6L-6\ell_0+5}} \lesssim \beta_0\cdot\frac{\epsilon}{(kNL)^{6L-1}} .$$

Combining all three parts where the worst bound is $O\qty(\frac{\beta_0\cdot\epsilon}{(kNL)^{6L-2.5}})$, we conclude the proof.
\end{proof}

With these bounds for the gradients, we can further upper bound the distance between the ideal $X_2^{(\ell)}$ and the actual intermediate sequence $\hat{X}_2^{(\ell)}$. Here, $\ell=1$ suffer from the first kind of additional error (further training after the layer learns the correct pattern), while $\ell\ge 3$ have the second type of additional error (unwanted updates from zero before the effective training stage).

\begin{lemma}[Perturbation analysis of stage 2 for $X_2^{(\ell)}$]\label{lemma: perturbation for second stage all layer}
    Under the condition of Theorem~\ref{thm:mix-data-main}, we have for all $\ell\in[L]$, 
    $$\|X_2^{(\ell)}-\hat{X}_2^{(\ell)}\|\lesssim \frac{\epsilon}{(kNL)^{6L-6}},$$
    where $X_2^{(\ell)}$ is the ideal output with one-hot attention pattern, and $\hat{X}_2^{(\ell)}$ is the transformer output.
\end{lemma}
\begin{proof}
    The ideal output for the $\ell$ layer is $$X_2^{(\ell)} = X_2^{(\ell-1)} + W_{KQ}^{(\ell)}X_2^{(\ell-1)}\S^{(\ell)}_{\mathrm{ideal}},$$ where $\S^{(\ell)}_{\mathrm{ideal}}$ is the ideal one-hot attention pattern. The empirical output $\hat{X}^{(\ell)}_2$ has the error introduced by the non-saturation of the softmax, together with those from the previous error in $\hat{X}_2^{(\ell-1)}$:
\begin{align*}
    \hat{X}_2^{(\ell)} &= \hat{X}_2^{(\ell-1)} + W_{OV}^{(\ell)}\hat{X}_2^{(\ell-1)}\S^{(\ell)}\\ &= X_2^{(\ell)} + W_{OV}^{(\ell)}X_2^{(\ell-1)}\underbrace{(\S_2^{(\ell)}-\S^{(\ell)}_{2,\mathrm{ideal}})}_{\Delta \S_2^{(\ell)}}+\underbrace{(\hat{X}_2^{(\ell-1)}-{X}_2^{(\ell-1)})+W_{OV}^{(\ell)}(\hat{X}_2^{(\ell-1)}-{X}_2^{(\ell-1)})\S_2^{(\ell)}}_{\text{Accumulated perturbation error}}.
\end{align*}
\textbf{Layer 1:} We first prove the case with $\ell=1$, where the accumulated perturbation error is 0 since $X^{(0)}$ is always the original input. Therefore, we just consider the $\Delta\S_2^{(\ell)}$ term.

According to the last step in \Cref{lemma: empirical gradient of the first stage}, the separation between correct and wrong positions before the softmax $\eta\hat{\Delta}$ is greater than $\Omega(\log k\log\frac{kN}{\epsilon})$. While by \Cref{lemma: gradient perturbation in stage 2}, the update in the second stage should be upper bounded by the sum of all gradient norms of $W_{KQ}^{(1)}$ in the second stage:
$$\sum_{\ell=1}^L\norm{\eta \nabla_{W_{KQ}^{(1)}}\mathcal{L}^{(\ell)}}_\infty \le L\cdot k^2N^3\cdot \frac{\epsilon}{(KNL)^{6L-2.5}} \log k\cdot\log\frac{kN}{\epsilon}\ll \log k\log\frac{kN}{\epsilon},$$
which is dominated by the current parameter since $\epsilon\le O(\frac{1}{k^6N^6})$. Therefore, the softmax $\S^{(1)}_2$ is still close to one-hot by
$$\|\S_2^{(1)}-\S^{(1)}_{2,\mathrm{ideal}}\|_\infty\le \frac{\epsilon}{2(KNL)^{6L}}.$$
and thus we have $\|X_2^{(1)}-\hat{X}_2^{(1)}\|_\infty\le \frac{\epsilon}{2(kNL)^{6L}}$.

\textbf{Layer 2:} For $\ell=2$, we already have $\|\Delta\S^{(2)}_{2}\|\le \frac{\epsilon}{2(KNL)^{6L}}$, so the softmax error is upper bounded by $\frac{\epsilon}{2(KNL)^{6L}}$. 
Now consider the accumulated perturbation error. By the perturbation analysis, we have $\norm{\hat{X}_2^{(1)}-{X}_2^{(1)}}_{\infty}\le \frac{\epsilon}{2(kNL)^{6L}}.$ Note that the error in $\hat{X}_2^{(1)}-{X}_2^{(1)}$ are in different rows of the matrices from 
$W_{OV}^{(2)}(\hat{X}_2^{(1)}-{X}_2^{(1)})\S^{(2)}$. Therefore, $\hat{X}_2^{(1)}-{X}_2^{(1)}$ won't introduce extra error in this stage. By similar arguments in stage 1 we have $\|W_{OV}^{(2)}(\hat{X}_2^{(1)}-{X}_2^{(1)})\S^{(2)}\|_\infty\le \|\hat{X}_2^{(1)}-{X}_2^{(1)}\|_\infty$.
Combine both parts of error, we have 
$\|X_2^{(2)}-\hat{X}_2^{(2)}\|_\infty\le \frac{\epsilon}{(kNL)^{6L}}.$

\textbf{Layer $\ell$:} Finally, we inductively prove that for layer $\ell\ge 3$, the distance $\|X_2^{(\ell)}-\hat{X}_2^{(\ell)}\|\le \frac{\epsilon\log\frac{1}{\epsilon}}{(kNL)^{6L-6}}$. For the base case $\ell=3$, we first consider the second accumulation term, which is upper bounded by the second layer: $$\|W_{OV}^{(3)}(\hat{X}_2^{(2)}-{X}_2^{(2)})\S^{(3)}\|_\infty\le \|\hat{X}_2^{(2)}-{X}_2^{(2)}\|_\infty \le \frac{\epsilon}{(kNL)^{6L}}.$$

Now we prove that the $\|\Delta \S_2^{(\ell)}\|_\infty\le\frac{\epsilon\log\frac{1}{\epsilon}}{(kNL)^{6L-6}L}$. The parameter norm after the update for each $W_{KQ}^{(\ell)}$, $\ell>2$ is upper bounded by
$\eta \sum_{\ell'=1}^L\|\nabla_{W_{KQ}^{(\ell)}}\mathcal{L}^{(\ell')}\|_2\le \frac{L\epsilon}{(kNL)^{6L-2.5}}$ using \Cref{lemma: gradient perturbation in stage 2}. Therefore we can expand the softmax:
\begin{align*}
    \left\|\S_2^{(\ell)}-\frac{1}{kN}\1_{kN}\right\|_\infty&\le \eta\|\tilde{J}\cdot(\hat{X}_2^{(\ell-1)})^\top \nabla_{W_{KQ}^{(\ell)}}\mathcal{L}^{\text{M}}\hat{X}_{2,(i,j)}^{(\ell-1)}\|_\infty\\
    &\le C\cdot \eta\cdot \frac{1}{kN}\cdot \sum_{\ell'=1}^L\|\nabla_{W_{KQ}^{(\ell)}}\mathcal{L}^{(\ell')}\|_2\cdot L^2\tag{$\|\tilde{J}\|_2\lesssim \frac{1}{kN}, \|\hat{X}_{2,(i,j)}^{(\ell-1)}\|_2\lesssim L.$}\\
    &\lesssim \frac{k^2N^3}{\beta_0}\log k\log \frac{kN}{\epsilon}\cdot \frac{1}{kN}\cdot  d\frac{\beta_0L\epsilon}{(kNL)^{6L-2.5}} \cdot L^2\\
    &\le \frac{\epsilon\log \frac{1}{\epsilon}}{(kNL)^{6L-5.5}}\tag{$k>\log k, kN>\log kN$}.
\end{align*}
Since for the ideal $W_{OV}^{(3)}{X}_2^{(2)}$ is one-hot/all zero for each row, $\norm{W_{OV}^{(3)}{X}_2^{(2)}\Delta \S_2^{(\ell)}}_\infty \le \frac{\epsilon\log \frac{1}{\epsilon}}{(kNL)^{6L-5.5}}$. Combine both part, we have $\|X_2^{(3)}-\hat{X}_2^{(3)}\|\le \frac{\epsilon\log\frac{1}{\epsilon}}{(kNL)^{6L-5.5}}$. 

For $\ell\ge 4$, we first consider the second accumulation term, which is upper bounded by the induction hypothesis: $$\|W_{OV}^{(\ell)}(\hat{X}_2^{(\ell-1)}-{X}_2^{(\ell-1)})\S^{(\ell)}\|_\infty\le \|\hat{X}_2^{(\ell-1)}-{X}_2^{(\ell-1)}\|_\infty \le \frac{\epsilon\log\frac{1}{\epsilon}}{(kNL)^{6L-5.5}}.$$
And since each row of ideal $W_{OV}^{(\ell)}{X}_2^{(\ell-1)}$ is either all-one/all-zero, $\norm{W_{OV}^{(\ell)}{X}_2^{(\ell-1)}\Delta \S_2^{(\ell)}}_\infty = 0.$
By induction, we finish the proof.
\end{proof}

To conclude, the intermediate sequence $\hat{X}^{(\ell)}_2$ is $O(\frac{\epsilon\log \frac{1}{\epsilon}}{(kNL)^{6L-6}})$ close to the ideal $X^{(\ell)}_2$ after Stage 2 for all layer $\ell$. In the next section, we will continue the induction and prove that the final error is still $1/\poly(kN)$ small.

\subsection{Stage $t\ge 3$ for Mixed Training}\label{sec:mix-train-stage-ell}
Finally, we prove the rest of the stages still satisfy that for all $\ell\in[L],t\in[L],$
$$\norm{\hat{X}_t^{(\ell)}-{X}_t^{(\ell)}}_\infty\le \frac{\epsilon\log^{t-1} \frac{1}{\epsilon}}{(kNL)^{6L-6t+6}},$$
the softmax score of the first layer satisfies
$$\S^{(1)}_{(i,\pi_i(j))}=\S((X^{(0)})^\top W_{KQ}^{(1)}X^{(0)}_{(i,j)}) \ge 1-\frac{\epsilon}{2(kNL)^{6L}},$$
and the softmax score of $\ell$th layer ($1<\ell<t$) satisfies
$$\S^{(\ell)}_{(i+2^{\ell-2},\hop^{2^{\ell-2}}_i(j))}:=\S\qty((\hat{X}^{(\ell-1)})^\top W_{KQ}^{(\ell)}\hat{X}^{(\ell-1)}_{(i,j)})_{(i+2^{\ell-2},\hop^{2^{\ell-2}}_i(j))} \ge 1-\frac{\epsilon}{2(kNL)^{6L}}.$$

We prove this by induction, and we already have $t=2$ as our induction hypothesis.

First, we prove that for each stage $t$, $\nabla_{W_{KQ}^{(t)}}\mathcal{L}$ is close to the population gradient, while the other gradients can be upper bounded using the perturbation $\norm{\hat{X}_{t-1}^{(\ell)}-{X}_{t-1}^{(\ell)}}_\infty$ from the previous timestep.

Similar to stage 2, we first control the deviation of gradients $\nabla_{W_{KQ}^{(\ell')}}\mathcal{L}^{(\ell)}$ from their idealized versions, where $\ell'\neq t$ or $\ell\neq t$. Since when $\ell'>\ell$ the gradient is zero by definition, we only consider $\ell\ge \ell'$. We first consider the cases when $\ell =\ell'$ and then $\ell > \ell'$.

\begin{lemma}
    \label{lemma: gradient perturbation in stage t}
In stage $t\ge 3$, given $\norm{\hat{X}_{t-1}^{(\ell)}-{X}_{t-1}^{(\ell)}}_\infty\le \frac{\epsilon\log^{t-2} \frac{1}{\epsilon}}{(kNL)^{6L-6t+12}}$, the gradient error of $W_{KQ}^{(\ell)}$ from the idealized gradient can be upper bounded by
\begin{itemize}
    \item If $\ell<t$: $\|\nabla_{W_{KQ}^{(\ell)}}\mathcal{L}^{(\ell)}\|_\infty\lesssim \dfrac{\beta_0\epsilon}{(kNL)^{6L-1}}$, where idealized gradient is zero.
    \item If $\ell>t$: $\|\nabla_{W_{KQ}^{(\ell)}}\mathcal{L}^{(\ell)}\|_\infty\lesssim \dfrac{\beta_0\epsilon\log^{t-2}\frac{1}{\epsilon}}{(kNL)^{6L-6t+12}}$, where idealized gradient is zero.
    \item If $\ell = t$: $\norm{\nabla_{W_{KQ}^{(\ell)}}\mathcal{L}^{(\ell)}-\nabla_{W_{KQ}^{(\ell)}}\mathcal{L}^{(\ell)}_{\mathcal{D}}(X_{t-1})}_\infty\lesssim \dfrac{\beta_0\epsilon\log^{t-2}\frac{1}{\epsilon}}{(kNL)^{6L-6t+12}}+\dfrac{\log(d/\delta)}{\sqrt{M}}$,\\ where $\nabla_{W_{KQ}^{(\ell)}}\mathcal{L}^{(\ell)}_{\mathcal{D}}(X_{t-1})$
    is the population gradient with the idealized input $X_{t-1}^{(\ell-1)}$ to the $\ell$th layer and the $\ell$th key query matrix set to zero: $W_{KQ}^{(\ell)}=0$.
\end{itemize}
\end{lemma}
\begin{proof}
    We consider $\ell<t$ first. Since the softmax is close to one-hot by $O\qty(\frac{\epsilon}{(kNL)^{6L}})$, 
    we have the Jacobian's spectral norm upper bounded by
    $$\|J^{(\ell)}\|_2 = \|\diag(\S)-\S\S^\top\|_2\lesssim \frac{\epsilon}{(kNL)^{6L}}.$$
    Therefore the first layer gradient has an infinity norm upper bound (since $\|\hat{X}^{(\ell)}_m\|_\infty\le 1$.)
    $$\|\nabla_{W_{KQ}^{(\ell)}}\mathcal{L}^{(\ell)}\|_\infty\le 2d\norm{J^{(1)}}_2 \cdot \beta_0\cdot 1\lesssim \frac{\beta_0\epsilon}{(KNL)^{6L-1}}.$$

    For $\ell = t$, we apply a similar strategy as in \Cref{lemma: concentration for stage l}. The sample noise with idealized input sequences $X_{t-1}$ can be upper bounded by $O\qty(\frac{\log(d/\delta)}{\sqrt{M}})$, and we need to further upper bound the perturbation error. 

    The actual empirical gradient is 
    \begin{align*}
    \hat{g}_{\ell} 
    &=\beta_0\hat{X}_{t-1} J_{t-1}^{(\ell)}\hat{X}_{t-1}^\top (e_{L+2,\ell+1}\otimes \1_k\otimes e_{N,\hop^{2^{\ell-1}}_{i}(j)}) (\hat{X}^{(\ell-1)}_{t-1,(i,j)})^\top\\&-\beta_0\sum_{s'\in[N]}\S_{s'}\hat{X}_{t-1}J_{t-1}^{(\ell)}\hat{X}_{t-1}^\top (e_{L+2,\ell+1}\otimes \1_k\otimes e_{N,s'}) (\hat{X}^{(\ell-1)}_{t-1,(i,j)})^\top.
    \end{align*}
    Given that $\norm{\hat{X}_{t-1}^{(\ell)}-{X}_{t-1}^{(\ell)}}_\infty\le \frac{\epsilon\log^{t-2} \frac{1}{\epsilon}}{(kNL)^{6L-6t+12}}$, and by induction of the last stage, the softmax score at initialization of stage $t$ is
    $$\left\|\S_{t-1}^{(\ell)}-\frac{1}{kN}\1_{kN}\right\|_\infty\le \frac{\epsilon\log^{t-2} \frac{1}{\epsilon}}{(kNL)^{6L-6t+12}}.$$
    The perturbation error that lies in the Jacobian $J_{t-1}^{(\ell)}$ and $\hat{X}_{t-1}^{(\ell)}$ can be upper bounded by $O\qty(\frac{\beta_0\epsilon\log^{t-2}\frac{1}{\epsilon}}{(kNL)^{6L-6t+12}}).$
    Combine the two error terms and we finished the proof for $\ell=t.$
    
    For $\ell\ge t+1$, the proof is similar to $\ell=t$. Since the idealized empirical gradient is always zero, we just need to bound the perturbation error, which is the same as $\ell=t$. The upper bound is also $O\qty(\frac{\beta_0\epsilon\log^{t-2}\frac{1}{\epsilon}}{(kNL)^{6L-6t+12}}).$ 
\end{proof}

Next, we bound the non-signal gradients $\nabla_{W_{KQ}^{(\ell')}}\mathcal{L}^{(\ell)}$ with $\ell>\ell'$. We first apply the bounds \Cref{lemma: gradient upper bound for previous layers with curriculum} for $\nabla_{W_{KQ}^{(\ell')}}\mathcal{L}^{(\ell)}$ with $\ell'< t$. By the softmax lower bound $\S^{(\ell')}_{(i+2^{\ell'-1},\hop^{2^{\ell'-2}}_i(j))}\ge 1-\frac{\epsilon}{2(kNL)^{6L}}$ by induction before the gradient in stage $t$ applies, we have 
$$\norm{\nabla_{W_{KQ}^{(\ell')}} \mathcal{L}^{(\ell)}}_F \le6\beta_0k(kNL)^{3/2}\cdot\frac{\epsilon}{(kNL)^{6L}} \lesssim \beta_0\cdot\frac{\epsilon}{(kNL)^{6L-2.5}} .$$

Finally, to upper bound the terms with $\ell>\ell'\ge t$, we apply \Cref{lemma: gradient upper bound for previous layers with mixing} (in \Cref{appendix: general upper bound for non-signal terms in mix training}). The upper bound is $O\qty(\frac{\beta_0\epsilon\log^{t-2}\frac{1}{\epsilon}}{(kNL)^{6L-6t+11}}).$ 

Combining all of the above results, for $t\ge 2$ we have that each gradient error with $\ell\neq t$ or $\ell'\neq t$ is in the worst-case upper bounded by
$$\norm{\nabla_{W_{KQ}^{(\ell')}} \mathcal{L}^{(\ell)}}_F \lesssim \beta_0\cdot\frac{\epsilon}{(kNL)^{\min\{6L-6t+11,6L-2.5\}}}\le \beta_0\cdot\frac{\epsilon}{(kNL)^{6L-6t+9.5}}.$$

Given the above bounds on the empirical gradients, we next prove that after one gradient step: (1) the $t$th layer learns the correct signal and (2) the other layers' unwanted errors can be controlled as predicted in the induction hypothesis. Those are shown in the following two lemmas. The proof strategy resembles \Cref{lemma: perturbation for second stage all layer}.

\begin{lemma}\label{lemma: softmax saturation for stage ell mix training}
    For $\ell\le t$, for all query $(i,j)$ and input $X$, we have after the $t$th gradient step
    $$\S^{(\ell)}_{(i+2^{\ell-2},\hop^{2^{\ell-2}}_i(j))}:=\S\qty((\hat{X}^{(\ell-1)})^\top W_{KQ}^{(\ell)}\hat{X}^{(\ell-1)}_{(i,j)})_{(i+2^{\ell-2},\hop^{2^{\ell-2}}_i(j))} \ge 1-\frac{\epsilon}{2(kNL)^{6L}}.$$
    Furthermore, we have $\|X_t^{(\ell)}-\hat{X}_t^{(\ell)}\|_\infty\le \frac{\epsilon}{2(kNL)^{6L}}$.
\end{lemma}

\begin{proof}
    For each layer $\ell<t$, recall that the separation between the correct position and the other positions before the one-step gradient in stage $t$ is $\Omega(\log k\log\frac{kN}\epsilon)$. By \Cref{lemma: gradient perturbation in stage t} and \Cref{lemma: gradient upper bound for previous layers with curriculum}, the update in the second stage should be upper bounded by 
$$\sum_{\ell'=1}^L\norm{\eta \nabla_{W_{KQ}^{(\ell)}}\mathcal{L}^{(\ell')}}_\infty \le L\cdot k^2N^3\cdot \frac{\epsilon}{(KNL)^{6L-2.5}} \log k\cdot\log\frac{kN}{\epsilon}\ll \log k\log\frac{kN}{\epsilon},$$
which is dominated by the current parameter since $\epsilon\le O(\frac{1}{k^6N^6})$. Therefore, the softmax $\S^{(1)}_2$ is still close to one-hot by
$$\|\S_{t}^{(\ell)}-\S^{(\ell)}_{t,\mathrm{ideal}}\|_\infty\le \frac{\epsilon}{2(KNL)^{6L}}.$$
and thus we have $\|X_t^{(\ell)}-\hat{X}_t^{(\ell)}\|_\infty\le \frac{\epsilon}{2(kNL)^{6L}}$.

For $\ell = t$, we upper bound the error of the gradient estimate by \Cref{lemma: gradient perturbation in stage t}. The previous case with $\ell = t-1$ guarantees that the intermediate input $\hat{X}^{(t-1)}_{t}$ is close to the ideal input:
$$\|X_t^{(t-1)}-\hat{X}_t^{(t-1)}\|_\infty\le \frac{\epsilon}{2(kNL)^{6L}}.$$
We define the separation of the pre-softmax attention score between the correct position $(i+2^{t-2},\hop_i^{2^{t-2}}(j))$ and the others as $\hat{\Delta}^{(\ell)}_{i,j}$:
$$\hat{\Delta}^{(\ell)}_{i,j} = \qty(\hat{X}^\top \hat{g}_\ell \hat{X}_{(i,j)})_{(i+2^{t-2},\hop_i^{2^{t-2}}(j))}-\max_{p\not=(i+2^{t-2},\hop_i^{2^{t-2}}(j))}\qty(\hat{X}^\top \hat{g}_\ell\hat{X}_{(i,j)})_p.$$
By \Cref{lemma: empirical gradient of stage ell} and the upper bound for the perturbation, we have $\|\hat{\Delta}^{(\ell)}_{i,j}\|$ at least $\frac{\eta\beta_0}{k^2N^3}$. After one step gradient with learning rate $\eta \gtrsim \frac{k^2N^3}{\beta_0}\log k\log\frac{kN}{\epsilon}$, the softmax output of the correct position can be lower bounded by
\begin{align*}
    \S(\hat{X}_m^\top W_{KQ}^{(\ell)} \hat{X}_{(i_m,j_m)})_{i_m,\hop^{1}_{i_m}(j_m)} \ge \frac{\exp\qty(\frac{\eta\beta_0}{kN}\cdot \frac{1}{kN^2})}{\exp\qty(\frac{\eta\beta_0}{kN}\cdot \frac{1}{kN^2})+kN-1}\ge 1-\frac{\epsilon}{2(kNL)^{6L}}.
\end{align*}
Thus we have $\|X_t^{(t)}-\hat{X}_t^{(t)}\|_\infty\le \frac{\epsilon}{2(kNL)^{6L}}$, which concludes the proof.
\end{proof}

The second lemma tracks the perturbation error for later layers.

\begin{lemma}[Perturbation analysis of stage $t$ for $X_t^{(\ell)}$]\label{lemma: perturbation for tth stage all layer}
    Under the condition of Theorem~\ref{thm:mix-data-main}, we have for all $\ell\ge t+1$, 
    $$\|X_t^{(\ell)}-\hat{X}_t^{(\ell)}\|\lesssim \frac{\epsilon\log^{t-1}\frac{1}{\epsilon}}{(kNL)^{6L-6t+6}},$$
    where $X_2^{(\ell)}$ is the ideal output with one-hot attention pattern, and $\hat{X}_2^{(\ell)}$ is the transformer output.
\end{lemma}
\begin{proof}
Recall the decomposition of the transformer intermediate sequence
\begin{align*}
    \hat{X}_t^{(\ell)} &= \hat{X}_t^{(\ell-1)} + W_{OV}^{(\ell)}\hat{X}_t^{(\ell-1)}\S^{(\ell)}\\ &= X_t^{(\ell)} + W_{OV}^{(\ell)}X_t^{(\ell-1)}\underbrace{(\S_t^{(\ell)}-\S^{(\ell)}_{t,\mathrm{ideal}})}_{\Delta \S_t^{(\ell)}}+\underbrace{(\hat{X}_t^{(\ell-1)}-{X}_t^{(\ell-1)})+W_{OV}^{(\ell)}(\hat{X}_t^{(\ell-1)}-{X}_t^{(\ell-1)})\S_t^{(\ell)}}_{\text{Accumulated perturbation error}}.
\end{align*}

We inductively prove that for layer $\ell\ge t+1$, the distance $\|X_t^{(\ell)}-\hat{X}_t^{(\ell)}\|\le \frac{\epsilon\log^{t-1} \frac{1}{\epsilon}}{(kNL)^{6L-6t+6}}$. For the base case $\ell=t+1$, we first consider the second accumulation term, which is upper bounded by: $$\|W_{OV}^{(t+1)}(\hat{X}_t^{(t)}-{X}_t^{(t)})\S^{(t+1)}\|_\infty\le \|\hat{X}_t^{(t)}-{X}_t^{(t)}\|_\infty \le \frac{\epsilon}{(kNL)^{6L}}.$$

Now we prove that the $\|\Delta \S_t^{(\ell)}\|_\infty\le\frac{\epsilon\log^{t-1}\frac{1}{\epsilon}}{(kNL)^{6L-6t+6}L}$. We expand the softmax:
\begin{align*}
    \left\|\S_t^{(\ell)}-\frac{1}{kN}\1_{kN}\right\|_\infty&\le \|\tilde{J}_t\cdot(\hat{X}_t^{(\ell-1)})^\top \qty(\eta\nabla_{W_{KQ}^{(\ell)}}\mathcal{L}^{\text{M}}+W_{KQ}^{(\ell)(t)})\hat{X}_{t,(i,j)}^{(\ell-1)}\|_\infty\\
    &\le C \cdot \frac{1}{kN}\cdot \qty(\sum_{\ell'=1}^L\eta\|\nabla_{W_{KQ}^{(\ell)}}\mathcal{L}^{(\ell)}\|_2+\norm{W_{KQ}^{(\ell)}(t)}_2)\cdot L^2\tag{$\|\tilde{J}\|_2\lesssim \frac{1}{kN}, \|\hat{X}_{t,(i,j)}^{(\ell-1)}\|_2\lesssim L.$}\\
    &\lesssim \frac{k^2N^3}{\beta_0}\log k\log \frac{kN}{\epsilon}\cdot \frac{1}{kN}\cdot  d\frac{\beta_0\epsilon\log^{t-2}\frac{1}{\epsilon}}{(kNL)^{6L-6t+9.5}} \cdot L^2\tag{$\norm{W_{KQ}^{(\ell)}(t)}_2\le \frac{\epsilon\log^{t-2}\frac{1}{\epsilon}}{(kNL)^{6L-6t+9.5}}$}\\
    &\le \frac{\epsilon\log^{t-1} \frac{1}{\epsilon}}{(kNL)^{6L-6t+6}L}\tag{$k>\log k, kN>\log kN$}.
\end{align*}
Since for the ideal $W_{OV}^{(\ell)}{X}_t^{(\ell-1)}$ is one-hot/all zero for each row, $\norm{W_{OV}^{(\ell)}{X}_t^{(\ell-1)}\Delta \S_t^{(\ell)}}_\infty \le \frac{\epsilon\log^{t-1} \frac{1}{\epsilon}}{(kNL)^{6L-6t+6}L}$. Combine both part, we have $\|X_t^{(\ell)}-\hat{X}_t^{(\ell)}\|\le \frac{\epsilon\log^{t-1} \frac{1}{\epsilon}}{(kNL)^{6L-6t+6}}$. Further, this also indicates that the 2-norm of $\norm{W_{KQ}^{(\ell)}(t)}_2$ is upper bounded by
$$\norm{W_{KQ}^{(\ell)}(t)}_2\le \frac{\epsilon\log^{t-1}\frac{1}{\epsilon}}{(kNL)^{6L-6t+6}}.$$

For $\ell\ge t+2$, we first consider the second accumulation term, which is upper bounded by the induction hypothesis: $$\|W_{OV}^{(\ell)}(\hat{X}_t^{(\ell-1)}-{X}_t^{(\ell-1)})\S^{(\ell)}\|_\infty\le \|\hat{X}_t^{(\ell-1)}-{X}_t^{(\ell-1)}\|_\infty \le \frac{\epsilon\log^{t-2}\frac{1}{\epsilon}}{(kNL)^{6L-6t+12}}.$$
And since each row of ideal $W_{OV}^{(\ell)}{X}_t^{(\ell-1)}$ is either all-one/all-zero, $\norm{W_{OV}^{(\ell)}{X}_t^{(\ell-1)}\Delta \S_t^{(\ell)}}_\infty = 0.$ Therefore, the upper bound should be $O\qty(\frac{\epsilon\log^{t-1}\frac{1}{\epsilon}}{(kNL)^{6L-6t+6}}).$ By induction, we finish the proof.
\end{proof}

When the induction comes to $\ell=L$, we correctly have all the $2^{\ell-1}$-hops encoded in the pre-readout output $X^{(L)}$.
At the end of the mix training, we further train the readout layer with one gradient step to output all the correct $2^{\ell-1}$-hops for each position.
\begin{lemma}\label{lemma: readout for ellth layer mix training}
    After one step gradient on $\theta_\Psi$ we have for all $\ell\in[L]$,
    \begin{align*}
    \sup_{\sigma,(i,j)}\norm{\S(\Psi_\ell^\top f^{(\ell)}(X^{(\ell-1)})_{(i,j)})-e_{\hop^{2^{\ell-1}}_i(j)}}_\infty\le \epsilon.
\end{align*}
\end{lemma}
\begin{proof}
We calculate the population gradient for $\Psi_\ell$, and then do the finite sample analysis. Note that when training on the mixture of data, we can learn all the readout layers at once.

The population gradient for the $\ell$th readout layer is
\begin{align*}
    \nabla_{\Psi_{\ell}} \mathcal{L}_{\mathcal{D}}^{(\ell)}(\theta) =-\E\qty[\qty(e_{N,\hop^{2^{\ell-1}}_i(j)}-\S(\Psi_\ell^\top f^{(\ell)}(X)_{(i,j)})) (f^{(\ell)}(X)_{(i,j)})^\top]
\end{align*}
By \Cref{lemma: softmax saturation for stage ell mix training}, the output of the second layer is
\begin{align*}
    f^{(\ell)}(X)_{(i,j)} &= W_{OV}^{(\ell)}\hat{X}^{(\ell-1)}\S^{(\ell)} = W_{OV}^{(\ell)}X^{(\ell-1)}\S^{(\ell)}_{\mathrm{ideal}}+W_{OV}^{(\ell)}(\hat{X}^{(\ell-1)}-X^{(\ell-1)})\S^{(\ell)}_{\mathrm{ideal}}+W_{OV}^{(\ell)}X^{(\ell-1)}\Delta\S^{(\ell)}\\
    &= e_{L+2,\ell+2}\otimes e_{k,i}\otimes e_{N,\hop^{2^{\ell-1}}_i(j)} + \Delta_\ell.
\end{align*}
where $\|\Delta_\ell\|_\infty\le \frac{\epsilon}{(kNL)^{6L}}$ by \Cref{lemma: softmax saturation for stage ell mix training}. Since $\Psi_\ell(0)=\beta_0 e_{L+2,\ell+2}\otimes \mathbf{1}_{k}\otimes I_{N\times N}$, we have the expansion
\begin{align*}
    \S(\Psi_\ell^\top f^{(\ell)}(X)_{(i,j)}) = \S(\beta_0e_{N,\hop_i^{2^{\ell-1}}(j)}+\Psi_{\ell}^\top \Delta_\ell)= \S(\beta_0e_{N,\hop^{2^{\ell-1}}_i(j)}) + \tilde{J}\Psi_\ell^\top \Delta_\ell.
\end{align*}
Since $\|\Delta_\ell\|_\infty\le {\epsilon}$, we have $\norm{\tilde{J}\Psi_\ell^\top\Delta_\ell}_\infty\le \beta_0\epsilon.$ The signal term
$$\S(\beta_0e_{N,\hop^{2^{\ell-1}}_i(j)}) = \frac{\exp(\beta_0)-1}{\exp(\beta_0)+N-1}e_{N,\hop^{2^{\ell-1}}_i(j)}+\frac{1}{\exp(\beta_0)+N-1}\1_N.$$
The population gradient is thus
\begin{align*}
    \nabla_{\Psi_{\ell}} \mathcal{L}_{\mathcal{D}}^{(\ell)}(\theta) 
    &=-\E\qty[\qty(\frac{N}{\exp(\beta_0)+N-1}e_{N,\hop^{2^{\ell-1}}_i(j)}-\frac{1}{\exp(\beta_0)+N-1}\1_N-\tilde{J}\Psi_\ell^\top\Delta_\ell)(f^{(\ell)}(X)_{(i,j)})^\top]\\
    &=-\frac{1}{(\exp(\beta_0)+N-1)k}\qty(e_{L+2,\ell+2}\otimes \1_k \otimes (I_N-\frac{1}{N}\1_N\1_N^\top)) + O(\epsilon).
\end{align*}
where $O(\epsilon)$ denotes the terms with infinity norm smaller than $\epsilon$ with $\epsilon\lesssim \frac{1}{k^6N^6}$. 
The error altogether is upper bounded by $O(\frac{1}{k^6N^6})$ in infinity norm.

After one step of gradient, we have the softmax score ($\Delta$ is the error term with $\norm{\Delta}_\infty\le \epsilon$.)
\begin{align*}
    \S(\Psi_\ell(1)^\top f^{(\ell)}(X)_{(i,j)}))&=\S\qty(\frac{\eta}{(\exp(\beta_0)+N-1)k}(I_N-\frac{1}{N}\1_N\1_N^\top)e_{N,\hop^{2^{\ell-1}}_i(j)}+\beta_0e_{N,\hop^{2^{\ell-1}}_i(j)}+\eta \Delta)
\end{align*}
The separation between the correct entry and the others are lower bounded by:
$$\frac{\eta}{(\exp(\beta_0)+N-1)k} \frac{N-1}{N}-\eta \|\Delta\|_\infty\gtrsim \frac{\eta}{kN}.$$
By $\eta \gtrsim {k^2N^3\log \frac{kN}{\epsilon}}$, we have $\S(\Psi_\ell^\top f^{(\ell)}(X)_{(i,j)})_{\hop^{2^{\ell-1}}_i(j)}\ge 1-\epsilon$ and thus
\begin{align*}
    \sup_{\sigma,(i,j)}\norm{\S(\Psi_\ell^\top f^{(\ell)}(X)_{(i,j)})-e_{\hop^{2^{\ell-1}}_i(j)}}_\infty\le \epsilon.
\end{align*}
\end{proof}

\subsection{Gradient Upper Bounds for Mixed Training}
\label{appendix: general upper bound for non-signal terms in mix training}

In the analysis of the curriculum learning algorithm, we already proved a gradient upper bound for $\nabla_{W_{KQ}^{(\ell')}} \mathcal{L}^{(\ell)}$ when $\ell'<\ell, \ell\le t$ in the $t$th stage (\Cref{lemma: gradient upper bound for previous layers with curriculum}). However, those are just half of the non-signal terms in the case of mix training. The following lemma addresses the cases with $\ell>\ell'\ge t$. 

\begin{lemma}\label{lemma: gradient upper bound for previous layers with mixing}
    Given a single sample $(X,i,j)$ and
    suppose the training is in stage $\ell_0$ (before the gradient step). If for all $\ell\ge \ell_0$, $$\left\|\S_{\ell_0}^{(\ell)}-\frac{1}{kN}\1_{kN}\right\|_\infty\le \frac{\epsilon \log^{\ell_0-1}\frac{1}{\epsilon}}{(kNL)^{6L-6\ell_0+6}}$$ and for all $\ell'< \ell_0$,$$\S((\hat{X}^{(\ell'-1)})^\top W_{KQ}^{(\ell')}\hat{X}^{(\ell'-1)}_{(i,j)})_{(i+2^{\ell'-2},\hop^{2^{\ell'-2}}_i(j))} \ge 1-\frac{\epsilon}{2(kNL)^{6L}}.$$
    Then we have for any $\ell>\ell_0, \ell_0\le\ell'<\ell$, the infinity norm of the gradient over $\mathcal{L}^{(\ell)}$ is upper bounded by
    \begin{align*}
        \|\nabla_{W_{KQ}^{(\ell')}} \mathcal{L}^{(\ell)}\|_\infty \le 6\beta_0\cdot\frac{\epsilon\log^{\ell_0-1}\frac{1}{\epsilon}}{(kNL)^{6L-6\ell_0+5}}.
    \end{align*}
\end{lemma}

\begin{proof}
    Recall the gradient for a single sample $(X,i,j)$
    \begin{align*}
    \nabla_{W_{KQ}^{(\ell')}} \mathcal{L}^{(\ell)} = -\qty[\sum_{s'\in[N]}\qty(\mathbf{1}\{s'=\hop^{2^{\ell-1}}_i(j)\}-\S(\Psi_\ell^\top f^{(\ell)}_{(i,j)})_{s'}) \nabla_{W_{KQ}^{(\ell')}} (\Psi_\ell^\top f^{(\ell)}_{(i,j)})_{s'}]
\end{align*}
So the norm of the gradient is upper bounded by 
\begin{align*}
    \norm{\nabla_{W_{KQ}^{(\ell')}} \mathcal{L}^{(\ell)}} &\le\left\|\sum_{s'\in[N]}\qty(\mathbf{1}\{s'=\hop^{2^{\ell-1}}_i(j)\}-\S(\Psi_\ell^\top f^{(\ell)}_{(i,j)})_{s'}) \nabla_{W_{KQ}^{(\ell')}} (\Psi_\ell^\top f^{(\ell)}_{(i,j)})_{s'}\right\|\\
    &\le 2\norm{\nabla_{W_{KQ}^{(\ell')}} (\Psi_\ell^\top f^{(\ell)}_{(i,j)})_{s'}} = 2\norm{\nabla_{W_{KQ}^{(\ell')}} (e_{N,s'}^\top\Psi_\ell^\top f^{(\ell)}_{(i,j)})}.
\end{align*}
Now we calculate the gradient recursively through Taylor expansion: for all $\Delta W$ with $\|\Delta W\|\to 0$, 
\begin{align*}
    f^{(\ell)}_{(i,j)}(W_{KQ}^{(\ell')}+\Delta W)=f^{(\ell)}_{(i,j)}(W_{KQ}^{(\ell')})+\langle\nabla_{W_{KQ}^{(\ell')}} (f^{(\ell)}_{(i,j)}),\Delta W\rangle + O(\|\Delta W\|^2_F).
\end{align*}
While by Taylor expansion, we have (we ignore $(W_{KQ}^{(\ell')})$ when there is no perturbation $\Delta W$.)
\begin{align*}
    f^{(\ell)}_{(i,j)}(W_{KQ}^{(\ell')}+\Delta W)&=f^{(\ell)}_{(i,j)}(W_{KQ}^{(\ell')})+W_{OV}^{(\ell)}\nabla_{W_{KQ}^{(\ell')}} X^{(\ell-1)}(\Delta W)\S((X^{(\ell-1)})^\top W_{KQ}^{(\ell)}X_{(i,j)}^{(\ell-1)})\\
    &\quad +W_{OV}^{(\ell)}X^{(\ell-1)}J^{(\ell)}\nabla\S((X^{(\ell-1)})^\top W_{KQ}^{(\ell)}X_{(i,j)}^{(\ell-1)})(\Delta W)+O(\norm{\Delta W}^2_F).
\end{align*}
Here the gradient $\nabla_{W_{KQ}^{(\ell')}} X^{(\ell-1)\in \R^{d\times kN\times d\times d}}$ is a 4-th order tensor, and $\nabla_{W_{KQ}^{(\ell')}} X^{(\ell-1)}(\Delta W)\in \R^{d\times kN}$. 
We upper bound the infinity norm of the matrix by induction from layer $\ell'$ to $\ell$. We prove that 
$$ \|\nabla_{W_{KQ}^{(\ell')}} X^{(t)}(\Delta W)\|_\infty\le \qty(1+\frac{1}{k})^{t-\ell'}\frac{L\cdot\epsilon\log^{\ell_0-1}\frac{1}{\epsilon}}{(kNL)^{6L-6\ell_0+6}}\|\Delta W\|_F$$
with $t\in [\ell',\ell-1]$.

\textbf{Base case: $t = \ell'$} By Taylor expansion on $X^{(\ell')}$, we have 
\begin{align*}
    X^{(\ell')}(W_{KQ}^{(\ell')}+\Delta W) &= X^{(\ell')} +W_{OV}^{(\ell')}X^{(\ell'-1)}J^{(\ell')}(X^{(\ell'-1)})^\top \Delta WX^{(\ell'-1)})+O(\norm{\Delta W}^2_F).
\end{align*}
since previous layers are independent of $W_{KQ}^{(\ell')}$. Therefore, the first order term is
\begin{align*}
    \|\nabla_{W_{KQ}^{(\ell')}} X^{(\ell')}(\Delta W)\| = \norm{W_{OV}^{(\ell')}X^{(\ell'-1)}J^{(\ell')}(X^{(\ell'-1)})^\top \Delta WX^{(\ell'-1)})}
\end{align*}
Note that $\left\|\S_{\ell_0}^{(\ell)}-\frac{1}{kN}\1_{kN}\right\|_\infty\le \frac{\epsilon \log^{\ell_0-1}\frac{1}{\epsilon}}{(kNL)^{6L-6\ell_0+6}}$,
we have the Jacobian $$\norm{J^{(\ell)}-\frac{1}{kN}\qty(I-\frac{1}{kN}\mathbf{1}\mathbf{1}^\top)}\lesssim \frac{\epsilon}{kN}\cdot\frac{\log^{\ell_0-1}\frac{1}{\epsilon}}{(kNL)^{6L-6\ell_0+6}}.$$ 
Note that $W_{OV}^{(\ell')}X_{\text{ideal}}^{(\ell'-1)}=\frac{1}{kN}(e_{L+2,\ell'+2}\otimes \mathbf{1}_{kN})\mathbf{1}_{kN}^\top$, which cancels with the ideal Jacobian. The excess error is also upper bounded by $\frac{\epsilon}{kN}\cdot\frac{\log^{\ell_0-1}\frac{1}{\epsilon}}{(kNL)^{6L-6\ell_0+6}}$.

Therefore, we can upper bound the first order term by 
\begin{align*}
    \|\nabla_{W_{KQ}^{(\ell')}} X^{(\ell')}(\Delta W)\|_2 &\le \norm{W_{OV}^{(\ell')}X^{(\ell'-1)}J^{(\ell')}}_2\norm{(X^{(\ell'-1)})^\top}_F\norm{ \Delta W}\norm{X^{(\ell'-1)})}_F\\&\lesssim kNL\cdot \frac{\epsilon}{kN}\cdot\frac{\log^{\ell_0-1}\frac{1}{\epsilon}}{(kNL)^{6L-6\ell_0+6}}\|\Delta W\|_F.\tag{$\norm{X^{(\ell'-1)}}_F^2\le O(kNL)$ since each embedding is either $\epsilon$-close to one-hot or all 0.}\\
    &\lesssim \frac{L\cdot\epsilon\log^{\ell_0-1}\frac{1}{\epsilon}}{(kNL)^{6L-6\ell_0+6}}\|\Delta W\|_F.
\end{align*}
Since $ \|\nabla_{W_{KQ}^{(\ell')}} X^{(\ell')}(\Delta W)\|_\infty\le  \|\nabla_{W_{KQ}^{(\ell')}} X^{(\ell')}(\Delta W)\|_2$, we finish the proof for base case.

\textbf{Induction: $t\in [\ell', \ell-1]$.} Suppose the induction hypothesis holds:
$$\|\nabla_{W_{KQ}^{(\ell')}} X^{(t)}(\Delta W)\|_\infty\le \qty(1+\frac{1}{k})^{t-\ell'}\frac{L\cdot\epsilon\log^{\ell_0-1}\frac{1}{\epsilon}}{(kNL)^{6L-6\ell_0+6}}\|\Delta W\|_F.$$
We consider expanding $X^{(t+1)}(W_{KQ}^{(\ell')}+\Delta W)$ like the base case:
\begin{align*}
    X^{(t+1)}(W_{KQ}^{(\ell')}+\Delta W)&=X^{(t)}(W_{KQ}^{(\ell')}+\Delta W) + f^{(t+1)}(W_{KQ}^{(\ell')}+\Delta W)\\
    &=X^{(t+1)} + \nabla_{W_{KQ}^{(\ell')}} X^{(t)}(\Delta W)\tag{From $X^{(t)}(W_{KQ}^{(\ell')}+\Delta W)$.}\\
    &\ + W_{OV}^{(t+1)} \nabla_{W_{KQ}^{(\ell')}} X^{(t)}(\Delta W) \S\qty(\qty(X^{(t)})^\top W_{KQ}^{(t+1)}X^{(t)})\\
    &\ + W_{OV}^{(t+1)} X^{(t)} J^{(t+1)}\qty(\nabla_{W_{KQ}^{(\ell')}} X^{(t)}(\Delta W))^\top W_{KQ}^{(t+1)}X^{(t)} \\
    &\ + W_{OV}^{(t+1)} X^{(t)} J^{(t+1)}\qty( X^{(t)})^\top W_{KQ}^{(t+1)}\nabla_{W_{KQ}^{(\ell')}}X^{(t)}(\Delta W) + O(\norm{\Delta W}_F^2)\tag{The last 3 terms are from $f^{(t+1)}(W_{KQ}^{(\ell')}+\Delta W)$.}
\end{align*}
Note that $\nabla_{W_{KQ}^{(\ell')}} X^{(t)}(\Delta W)$ and the rest three terms are in different rows since they uses different $W_{OV}^{(t)}$, so the infinity norm of the first order term should be the maximum of those two. 

We first upper bound the last three terms. Since $W_{OV}^{(t+1)}$ is partial identity, and softmax is some weighted average, we have
\begin{align*}
\norm{W_{OV}^{(t+1)} \nabla_{W_{KQ}^{(\ell')}} X^{(t)}(\Delta W) \S\qty(\qty(X^{(t)})^\top W_{KQ}^{(t+1)}X^{(t)})}\le \norm{\nabla_{W_{KQ}^{(\ell')}} X^{(t)}(\Delta W)}_\infty.
\end{align*}
The rest two terms can be directly upper bounded by
$$\norm{W_{OV}^{(t+1)}X^{(t)}J^{(t+1)}}_2\norm{\nabla_{W_{KQ}^{(\ell')}} X^{(t)}(\Delta W))}_2\norm{W_{KQ}^{(t+1)}}_2\norm{X^{(t)}}_2.$$
In previous stages, we have $\norm{X}_F^2\le O(kNL)$, $\norm{W_{OV}^{(t+1)}X^{(t)}J^{(t+1)}}_2\le \frac{\epsilon}{kN}\cdot\frac{\log^{\ell_0-1}\frac{1}{\epsilon}}{(kNL)^{6L-6\ell_0+6}}$ and $\|W_{KQ}^{(t+1)}\|_2\lesssim O(1)$. So these two terms can be both upper bounded by 
$\frac{L\epsilon\log^{\ell_0-1}\frac{1}{\epsilon}}{(kNL)^{6L-6\ell_0+6}}\norm{\nabla_{W_{KQ}^{(\ell')}} X^{(t)}(\Delta W))}_2$
Combining two error terms, we have
$$\norm{\nabla_{W_{KQ}^{(\ell')}} X^{(t+1)}(\Delta W)}_\infty\le(1+\frac{1}{k})^{t-\ell'+1}\frac{L\epsilon\log^{\ell_0-1}\frac{1}{\epsilon}}{(kNL)^{6L-6\ell_0+6}}\|\Delta W\|_F.$$

By induction, when $t = \ell-1$ we have (since $\ell-\ell'<k$.) 

$$\norm{\nabla_{W_{KQ}^{(\ell')}} X^{(\ell-1)}(\Delta W)}_\infty\le(1+\frac{1}{k})^{\ell-\ell'}\frac{L\epsilon\log^{\ell_0-1}\frac{1}{\epsilon}}{(kNL)^{6L-6\ell_0+6}}\|\Delta W\|_F\le 3\frac{L\epsilon\log^{\ell_0-1}\frac{1}{\epsilon}}{(kNL)^{6L-6\ell_0+6}}\|\Delta W\|_F.$$

Finally, we can upper bound $\|\nabla_{W_{KQ}^{(\ell')}} \mathcal{L}^{(\ell)}\|$ by picking $\Delta W = \alpha \nabla_{W_{KQ}^{(\ell')}} \mathcal{L}^{(\ell)}$ with $\alpha\to 0$:
\begin{align*}
    \alpha\norm{\nabla_{W_{KQ}^{(\ell')}} \mathcal{L}^{(\ell)}}^2_F&\le 2\left\langle\nabla_{W_{KQ}^{(\ell')}} (e_{N,s'}^\top\Psi_\ell^\top f^{(\ell)}_{(i,j)}),\Delta W\right\rangle\\
    &\le 2e_{N,s'}^\top\Psi_\ell^\top W_{OV}^{(\ell)}\nabla_{W_{KQ}^{(\ell')}} X^{(\ell-1)}(\Delta W)\S((X^{(\ell-1)})^\top W_{KQ}^{(\ell)}X_{(i,j)}^{(\ell-1)}) + O(\alpha^2)\\
    &\le 2\beta_0 k\norm{\nabla_{W_{KQ}^{(\ell')}} X^{(\ell-1)}(\Delta W)}_\infty \le 6\beta_0k\cdot \frac{L\epsilon\log^{\ell_0-1}\frac{1}{\epsilon}}{(kNL)^{6L-6\ell_0+6}}\|\Delta W\|_F. 
\end{align*}
Plug in $\Delta W = \alpha \nabla_{W_{KQ}^{(\ell')}} \mathcal{L}^{(\ell)}$, we have 
$\norm{\nabla_{W_{KQ}^{(\ell')}} \mathcal{L}^{(\ell)}}_F \le6\beta_0\cdot\frac{\epsilon\log^{\ell_0-1}\frac{1}{\epsilon}}{(kNL)^{6L-6\ell_0+5}}. $
\end{proof}

\section{Learning the Value Matrix}
\label{appx:value_matrix}

We now show that the population gradient with respect to the value matrix vanishes at zero initialization. For simplicity, we focus on the first hop:
\begin{align*}
    \nabla_{W_{OV}^{(1)}} L_\mathcal{D}^{(1)}(\theta) &= - \E_{\sigma,(i,j)}\qty[\sum_{s'\in [N]}\qty(\1\{s' = \hop_i^1(j)\} - \frac{1}{N}) \Psi_{s'} \qty(\frac{1}{kN} X \mathbf{1}_{kN} )^\top] \\
    &= - \E_{\sigma,(i,j)}\qty[\sum_{s'\in [N]}\qty(\1\{s' = \hop_i^1(j)\} - \frac{1}{N}) \Psi_{s'} \qty(\frac{1}{kN} \mathbf{1}_{kN} )^\top] \\
    &= - \E_{\sigma,(i,j)}\qty[\sum_{s'\in [N]}\1\{s' = \hop_i^1(j)\} \Psi_{s'} \qty(\frac{1}{kN} \mathbf{1}_{kN} )^\top] + \E_{\sigma,(i,j)}\qty[\sum_{s'\in [N]}\frac{1}{N} \Psi_{s'} \qty(\frac{1}{kN} \mathbf{1}_{kN} )^\top]
\end{align*}
The second expectation does not depend on $\sigma$, so the expectation is 
\begin{align*}
    \E_{\sigma,(i,j)}\qty[\sum_{s'\in [N]}\frac{1}{N} \Psi_{s'} \qty(\frac{1}{kN} \mathbf{1}_{kN} )^\top] = \E_{(i,j)}\qty[\sum_{s'\in [N]}\frac{1}{N} \Psi_{s'} \qty(\frac{1}{kN} \mathbf{1}_{kN} )^\top]
\end{align*}
In the first term, only the indicator $\1\{s' = \hop_i^1(j)\}$ depend on $\sigma$. By symmetry, there is $\frac{1}{N}$ probability that $s' = \hop_i^1(j)$. Therefore, we have the expectation of this term
\begin{align*}
    - \E_{\sigma,(i,j)}\qty[\sum_{s'\in [N]}\1\{s' = \hop_i^1(j)\} \Psi_{s'} \qty(\frac{1}{kN} \mathbf{1}_{kN} )^\top] = - \E_{(i,j)}\qty[\sum_{s'\in [N]}\frac{1}{N} \Psi_{s'} \qty(\frac{1}{kN} \mathbf{1}_{kN} )^\top]
\end{align*}
These two terms are identical and cancel out, showing $\nabla_{W_{OV}^{(1)}} L^{(1)}_\mathcal{D}(\theta)=0.$


\end{document}